\documentclass[twoside,11pt]{article}

\usepackage{blindtext}

\usepackage{amssymb, amsmath, amsthm, url, tikz, xcolor, algorithm, algpseudocodex, caption, subcaption, float, mathtools}

\usepackage[nohyperref, abbrvbib, preprint]{jmlr2e}

\usepackage[shortlabels]{enumitem}
\usepackage[colorlinks=true,bookmarks=false,linkcolor=blue,urlcolor=violet,citecolor=blue,breaklinks=true]{hyperref}
\usepackage{cleveref}
\usepackage[margin=1in]{geometry}

\graphicspath{{fig}}
\usepackage{macros}

\ShortHeadings{Representation Learning via Manifold Manipulation}{Psenka, Pai, Raman, Sastry, Ma}
\firstpageno{1}

\usepackage{lastpage}

\begin{document}

\title{Representation Learning via Manifold Flattening and Reconstruction}

\author{\name Michael Psenka \email psenka@eecs.berkeley.edu \\
       \addr Department of Electrical Engineering and Computer Science \\
       University of California, Berkeley\\
       Berkeley, CA 94720-1776, USA
       \AND
       \name Druv Pai \email druvpai@berkeley.edu \\
       \addr Department of Electrical Engineering and Computer Science \\
       University of California, Berkeley\\
       Berkeley, CA 94720-1776, USA
       \AND
       \name Vishal Raman \email vraman@berkeley.edu \\
       \addr Department of Electrical Engineering and Computer Science \\
       University of California, Berkeley\\
       Berkeley, CA 94720-1776, USA
       \AND
       \name Shankar Sastry \email sastry@coe.berkeley.edu \\
       \addr Department of Electrical Engineering and Computer Science \\
       University of California, Berkeley\\
       Berkeley, CA 94720-1776, USA
       \AND
       \name Yi Ma \email yima@eecs.berkeley.edu \\
       \addr Department of Electrical Engineering and Computer Science \\
       University of California, Berkeley\\
       Berkeley, CA 94720-1776, USA}
% \date{\today}

\maketitle

\begin{abstract}
    This work proposes an algorithm for explicitly constructing a pair of neural networks that linearize and reconstruct an embedded submanifold, from finite samples of this manifold. Our such-generated neural networks, called \textit{Flattening Networks} (FlatNet), are theoretically interpretable, computationally feasible at scale, and generalize well to test data, a balance not typically found in manifold-based learning methods. We present empirical results and comparisons to other models on synthetic high-dimensional manifold data and 2D image data. Our code is publicly available.
\end{abstract}


\section{Introduction} \label{sec:intro}

\textit{Autoencoding} \citep{kramer1991nonlinear} remains an important framework in deep learning and representation learning \citep{bengio2013representation}, where one aims to both encode data into a more compact, lower-dimensional representation and then decode it back into the original data format. This framework is often used on top of deep learning applications to ease downstream learning and processing \citep{rombach2022high,turian2010word,he2022masked} as well as a design principle for the network itself \citep{ronneberger2015u,devlin2018bert}.

However, these models are typically so-called ``black-box models'': many important parameters of the model cannot be directly optimized or chosen from mathematical principle, and instead need to be chosen from either extensive trial-and-error or from previous knowledge likewise obtained from trial-and-error. When one is building autoencoder models in a novel domain, this issue can hinder the speed and ease of development.

To resolve this issue, we design a framework for automatically building up an autoencoder model given a \textit{geometric} model of the dataset; namely, we assume that the ``manifold hypothesis'' \citep{fefferman2016testing} holds for the dataset we wish to encode/decode. Given this assumption, along with some needed regularity and sampling conditions, we can greedily construct the layers of an encoder/decoder pair in a forward fashion, similarly in spirit to recent works \citep{chan2022redunet, hinton2022forward}, such that the learned encoding is as low-dimensional as possible, in as few layers as possible. In this manner, the architecture is automatically constructed to be optimal. While there are a few scalar parameters that need pre-specification by the network designer, these are all mathematically interpretable and in practice not sensitive for training results.

Our approach to this problem is through \textit{manifold flattening}, or \textit{manifold linearization}: if we can globally deform the ``data manifold'' into a linear structure, then we can apply more interpretable linear models on top of these constructed features. Our motivation for focusing on manifold linearization comes from common practice: any successfully trained deep learning model whose last layer is linear (e.g. multi-layer perceptrons) has ``linearized'' their problem, since the deep model has implicitly learned a map (all but the last layer) such that the target problem is solvable by a linear map (the last layer).

\subsection{Problem formulation} \label{sec:prob_formulation}

We now give a mathematical formulation for the autoencoding problem that we want to solve, and define notation used for the rest of the paper. We assume foundational knowledge of differential geometry; for an introduction and review, see \Cref{sec:diffgeo}.
\\

\noindent \textbf{Fundamental problem.} Consider a \textit{dataset} \(\X = \{x_{1}, \dots, x_{N}\} \subseteq \R^{D}\). We make the \textit{geometric} assumption that our data \(x_{i}\) are drawn near a low-dimensional embedded sub-manifold \(\M \subset \R^{D}\) which has intrinsic dimension \(d \leq D\); we say that \(\M\) is the \textit{data manifold}. We aim to construct an \textit{encoder} \(\f \colon \R^{D} \to \R^{p}\), and a \textit{decoder} \(\g \colon \R^{p} \to \R^{D}\), which are a \textit{minimal autoencoding pair} in the following sense.

\begin{definition}\label{def:autoencode}
    Given an embedded submanifold \(\M \subset \R^{D}\) and a scalar parameter \(\eps > 0\), we say that a pair of continuous functions \(\f \colon \R^{D} \to \R^{p}\) and \(\g \colon \R^{p} \to \R^{D}\) are an \textit{\(\eps\)-minimal autoencoding pair} for \(\M\) if both of the following conditions hold:
     \begin{enumerate}[(a)]
         \item (Autoencoding.) For all \(x \in \M\), we have \(\norm{\g(\f(x)) - x}_{2} \leq \eps\), and
         \item (Minimality.) \(p\) is the smallest integer such that there exist continuous \(\f, \g\) such that (a) holds.
     \end{enumerate}
\end{definition}

\noindent \textbf{Computational considerations.} We mainly concern ourselves with computational time, and the scalability requirements often needed for modern data requirements. Ideally, our autoencoder algorithm follows the following scaling laws:

\begin{enumerate}
    \item \textit{Linear time in the number of samples \(N\)}, \(O(N)\). Datasets have become incredibly large, easily reaching the millions for many modern tasks, so even quadratic time in \(N\) quickly becomes uncomputable in practice.
    \item \textit{Linear time in the ambient dimension \(D\)}, \(O(D)\). Similarly with the number of samples, the ambient dimension in many cases can easily reach the thousands or millions, so anything above linear time becomes practically uncomputable.
    \item \textit{Polynomial time in the intrinsic dimension \(d\)}, \(O(d^{k})\) for some \(k\in \mathbb{N}\). In contrast to the sample size and ambient dimension, we expect the intrinsic dimension of data manifolds to be relatively small \citep{wright2022high}; for example, MNIST \citep{deng2012mnist} has ambient dimension \(D = 784\), but the intrinsic dimension of a single class of digits is estimated to be around \(12 \leq d \leq 14\) \citep{hein2005intrinsic,costa2006determining,facco2017estimating}. Still, manifold learning in generality requires exponential time and samples in the intrinsic dimension, \(O(c^{d})\), which is infeasible for even relatively low \(d\).
\end{enumerate}

\noindent \textbf{Assumptions on the data manifold \(\M\).} An important area of interest is to determine the correct assumptions for the kinds of manifolds often found in real world data. For this paper, our assumptions are relatively minimal, but it is important to be clear about them. For more information about the motivation of the following assumptions, see \Cref{thm:conv_flat} and its proof in \Cref{sec:proofs}.

\begin{enumerate}
    \item We assume \(\M\) is smooth.  This is a standard regularity assumption.
    
    \item We assume \(\M\) is compact. This assumption is reasonable in that it is fulfilled for closed and bounded sub-manifolds of Euclidean space, which we encounter in many application contexts. For example, the manifold of \(28 \times 28\) greyscale images as embedded in \(\R^{784}\) (where \(784 = 28 \cdot 28\)) is bounded in \([0, 1]^{784}\). Adding closedness simply makes analysis easier.

    \item We assume \(\M\) is connected. In spirit, we focus on a single class of data in this work. While it is an important future direction to extend to multi-class settings, we focus on the single-class setting in this work to emphasize the role of compression.

    \item We assume \(\M\) is flattenable. Due to requiring differential geometry concepts to define, flattenability is more rigorously defined later in the work, in \Cref{def:flattenable}. For now, this is the most restrictive assumption of the four, but there is still reason to believe this assumption is commonly satisfied for commonly encountered data manifolds in neural network-based applications; see the end of \Cref{sec:manifold_flattening} for more discussion about this point.
\end{enumerate}

\noindent \textbf{Assumptions on the dataset \(\X\).} The main assumption we make is a \textit{geometric} property which quantifies how much the data and the manifold are representative of each other.
\begin{definition}\label{def:faithful}
    Given an embedded submanifold \(\M \subset \R^{D}\), a finite dataset \(\X = \{x_{1}, \dots, x_{N}\} \subseteq \R^{D}\), and scalar parameters \(\eps, \delta > 0\), we say that  is \textit{\((\eps, \delta)\)-faithful} to \(\M\) if both of the following hold:
    \begin{enumerate}
        \item (\(\M\) represents \(\X\).) \(\inf_{x \in \M}\norm{x - x_{i}}_{2} \leq \eps\) for all \(i \in [N]\).
        \item (\(\X\) represents \(\M\).) \(\min_{i \in [N]}\norm{x - x_{i}}_{2} \leq \delta\) for all \(x \in \M\).
    \end{enumerate}
    One may re-phrase these conditions as saying that \(\M\) is an \(\eps\)-cover for \(\X\) and \(\X\) is a \(\delta\)-cover for \(\M\).
\end{definition}
Because \(\M\) is compact, for every \(\eps, \delta > 0\), there exists a finite \((\eps, \delta)\)-faithful dataset for \(\M\).

We do not make explicit statistical assumptions of our data, such as the assumption that the data \(x_{i}\) are drawn i.i.d.~from a distribution on \(\R^{D}\). This distinguishes our work from other autoencoder analysis and, more generally, much of statistical learning theory.
\\

\noindent \textbf{Putting it all together.} Finally, we lay down the defining properties of our desired algorithm for efficient computation of a minimal autoencoding pair. Note that we drop parameters for brevity.

\begin{definition}\label{def:main_problem}
    Given an embedded submanifold \(\M \subset \R^{D}\) of intrinsic dimension \(d\) and an faithful dataset \(\X \subset \R^{D}\), an algorithm is said to \textit{efficiently compute a minimal autoencoding pair for \(\M\)} if it computes a minimal autoencoding pair in \(O(DNd^{k})\) flops for some \(k \in \mathbb{N}\), with access only to \(\X\) (and not \(\M\) or \(d\)).
\end{definition}

\section{Related work}

We now provide a brief overview of works in this domain. The first two, VAEs and manifold learning, are the most direct alternative algorithms to solve the problem given in \Cref{def:main_problem}. However, each of these have fundamental limitations, which we address with our work. We also include methods that match our problem and algorithm in spirit, but in reality address fundamentally different problems.

\subsection{Variational autoencoders}

Arguably the most well-known class of autoencoders at the time of writing is the class of \emph{variational autoencoders} (VAEs) \citep{kingma2013auto}, which have seen much empirical success, both traditionally in computer vision problems \citep{higgins2017beta,van2017neural,kim2018disentangling}, and recently as a black-box method for nonlinear dimensionality reduction within the framework of within so-called \textit{diffusion} models \citep{vahdat2021score,rombach2022high}. Despite their empirical success, VAEs have a few endemic shortcomings, including posterior collapse \citep{lucas2019understanding}. However, the most common issue, obeyed by all flavors of VAE thus far to our knowledge, is that the encoder and decoder networks are \textit{black-box} neural networks, about which comparatively little is understood. Important design choices, including hyperparameter selection, thus are either made through extensive trial-and-error or using a history of trial-and-error for the problem (e.g., 2D image autoencoding), rendering application of VAEs in a novel problem a potentially expensive challenge \citep{rezende2018taming}. In this work, we aim to provide an alternative \textit{white-box} model which ameliorates and linearizes the data geometry, essentially performing nonlinear dimensionality reduction via transforming the data structure to a lower-dimensional affine subspace, all with fewer hyperparameters and greater robustness to hyperparameter changes than the alternatives.

\subsection{Manifold learning}

Manifold learning methods are a well-known class of algorithms that seek to find low dimensional representations of originally high dimensional data while preserving geometric structures.  This includes methods that find subspaces by preserving local structure, such as local linear embedding (LLE) \citep{roweis2000nonlinear} and its variants, local tangent space alignment (LTSA) \citep{zhang2003nonlinear}, t-distributed stochastic neighbor embedding (t-SNE) \citep{van2008visualizing}, Laplacian eigenmaps (LE) \citep{belkin2003laplacian} etc. and methods that try to preserve global structures such as isometric mapping (ISOMAP) \citep{tenenbaum2000global} and uniform manifold approximation and projection (UMAP) \citep{mcinnes2018umap}. Although these have offered useful results on a broad class of manifolds, there are limitations when applying these to real-world data \citep{li2019le}, which we discuss now and aim to mitigate through our algorithm.  

The first limitation is that for all the methods presented above, there is no explicit projection map between the original data and the corresponding low dimensional representation. This means that in order to find a projection for a point in the original space outside of the training data, the entire method needs to be repeated on the original dataset with the new point appended.  This additional computational cost is especially problematic in the case of large-scale datasets. One possible extension maps the original data into a reproducing kernel Hilbert space (RKHS) using kernel functions \citep{xu2008reproducing}, avoiding this ``out-of-sample'' problem \citep{li2008klpp}.  Although this is robust, it requires good kernel function selection and parameter selection, which can be challenging.  Another approach which has been explored is parameterizing the features through a neural network architecture \citep{jansen2017scalable, schmidt1992feed, chen1996rapid}.  This presents similar training difficulties to VAEs due to the black-box structure of the corresponding architectures.

Another potential limitation is that for all of the aforementioned methods, it is required to specify the intrinsic dimension of the data as an input.  In practice, this is hardly ever known a priori, though there are a large class of intrinsic dimension estimators \citep{campadelli2015intrinsic, levinaMLE, carter2009local} that can be used at an additional computational cost.
 
\subsection{Distribution learners}

Three forms of models for distribution learning in recent years have been GANs \citep{goodfellow2014generative}, normalizing flows \citep{kobyzev2019normalizing}, and diffusion models \citep{sohl2015deep,ho2020denoising,song2019generative}. While such models are effective at learning the distribution of the data \citep{arora2018gans} and are able to sample from it \citep{chen2022sampling}, they do not learn \textit{representations} or \textit{features} for individual data points. Thus, they solve different problems than the one we choose to solve in this work, as per our discussion in \Cref{sec:prob_formulation}.

\section{Manifold flattening for data manifolds} \label{sec:manifold_flattening}

We now introduce our high-level approach for finding an efficient algorithm which generates a minimal autoencoding pair: \textit{manifold flattening}. 

As a warmup, let us suppose that \(\M\) were actually an affine subspace, with no nonlinearity or curvature. Then principal component analysis (PCA) would give an efficient autoencoding; we could estimate the dimension using the subspace structure in any of a variety of ways, then the encoder \(\f \colon \R^{D} \to \R^{d}\) would be the orthogonal projection onto the first \(d\) principal components of the data, and the decoder \(\g \colon \R^{d} \to \R^{D}\) would be the adjoint map of \(\f\).

Motivated by the ease of the problem when the data lies on an affine subspace, our high-level approach is to first remove the nonlinearities in the data manifold \(\M\), and then use the previously-discussed PCA-type algorithms once we have ``flattened'' \(\M\). 

In this section, we formalize the notion of flattening a manifold and the equivalence between manifold flattening and minimal autoencoding as described in \Cref{def:autoencode}.

\subsection{Manifold flattening creates minimal representations}

We first provide a classical definition of flatness for a manifold \(\M \subset \R^{D}\) using the second fundamental form (see \Cref{def:sff} in \Cref{sec:diffgeo}).

\begin{definition}
    Let \(\M \subset \R^{D}\) be a smooth embedded submanifold of dimension \(d\) and second fundamental form \(\sff\). We say that \(\M\) is \textit{flat} if and only if \(\sff_{\xz}(u, v) = 0\) for all \(\xz \in \M\) and \(u, v \in \T_{\xz}\M\).
\end{definition}

This definition says that a manifold is flat if it has no extrinsic curvature. We are not restricted to only studying the extrinsic curvature of the original data manifold \(\M\), but also of its image through a smooth map: \(\fl(\M) \doteq \{\fl(x) \mid x \in \M\}\), where \(\fl \colon \R^{D} \to \R^{D}\). Ideally, then, the extrinsic curvature gives a measure of when a manipulation \(\fl\) has successfully removed the nonlinearity of \(\M\), and one can thus safely use PCA on top of \(\fl\) to generate a minimal autoencoding. We formalize this idea in the following definition and theorem.

\begin{definition}\label{def:flattenable}
    Let \(\M \subset \R^{D}\) be a compact, connected, smooth embedded submanifold of dimension \(d\). We say that \(\M\) is \textit{flattenable} if there exist smooth maps \(\fl, \re \colon \R^{D} \to \R^{D}\) such that:
    \begin{enumerate}
        \item \(\re(\fl(\M)) = \M\); and
        \item \(\fl(\M)\) is flat.
    \end{enumerate}
    In this case, we also say that \((\fl, \re)\) \textit{flatten and reconstruct} \(\M\).
\end{definition}

\begin{theorem}\label{thm:conv_flat}
    Let \(\M \subset \R^{D}\) be a flattenable submanifold of dimension \(d\) which is flattened and reconstructed by the maps \(\fl, \re \colon \R^{D} \to \R^{D}\). Then we have:
    \begin{enumerate}
       \item \(\fl(\M)\) is a convex set.
       \item \(\fl(\M)\) is contained within a unique affine subspace of dimension \(d\).
       \item Let \(U \in \R^{D \times d}\) be an orthonormal basis for the aforementioned subspace, stacked into a matrix, and \(z_{0} \in \fl(\M)\). The encoder \(\f\) given by \(\f(x) = U^\top(\fl(x) - z_{0})\) and decoder \(\g\) given by \(\g(z) = \re(Uz + z_{0})\) form a minimal autoencoding pair satisfying \Cref{def:autoencode} for \(\eps = 0\).
    \end{enumerate}
\end{theorem}

As the proof relies on some outside differential geometry, we leave the proof for \Cref{sec:proofs}. The main message for this theorem is a formalization of our previous intuition: \textit{finding flattening and reconstruction maps is equivalent to finding a minimal autoencoding}. Thus, in the rest of the work, we focus on constructing flattening and reconstruction maps \(\fl\) and \(\re\).

A first visualization of manifold flattening and reconstruction is depicted in \Cref{fig:idealized_interpolation}, along with the primary benefit outside of \Cref{def:main_problem} for manifold flattening: nonlinear interpolation. 
\\

\begin{figure}
    \centering
    \begin{tikzpicture}
        \draw (-0.1, -0.2) .. controls (1, 0.25) .. (2.1, -0.1);
        \draw (2.1, -0.1) .. controls (1.75, 1) .. (2.1, 2);
        \draw (2.1, 2) .. controls (1, 1.75) .. (-0.1, 2);
        \draw (-0.1, 2) .. controls (0.25, 1) .. (-0.1, -0.2);
        
        \draw (0.35, 0.5) .. controls (0.7, 0.75) .. (0.85, 1.2);
        \draw (0.85, 1.2) .. controls (1.0, 1.4) .. (1.15, 1.2);
        \draw (1.15, 1.2) .. controls (1.3, 0.75) .. (1.65, 0.5);
        
        \node[fill, red, circle, inner sep=0.75pt] at (1, 0.9) {};
        \node[fill, red, circle, inner sep=0.75pt] at (1.35, 0.45) {};
        
        \node at (1, -0.5) {\(x\)};
        
        \draw[->] (2.25, 1) -- (3.75, 1);
        \node at (3, 1.25) {\(\fl\)};
        
        \draw (4.1, 0.1) -- (5.9, 0.1);
        \draw (5.9, 0.1) -- (5.9, 1.9);
        \draw (5.9, 1.9) -- (4.1, 1.9);
        \draw (4.1, 1.9) -- (4.1, 0.1);
        
        \node[fill, red, circle, inner sep=0.75pt] at (5, 1) {};
        \node[fill, red, circle, inner sep=0.75pt] at (5.5, 0.5) {};
        
        \draw[red] (5, 1) -- (5.5, 0.5);
        
        \node at (5, -0.5) {\(z\)};
        
        \draw[->] (6.25, 1) -- (7.75, 1);
        \node at (7, 1.25) {\(\re\)};
        
        \draw (7.9, -0.2) .. controls (9, 0.25) .. (10.1, -0.1);
        \draw (10.1, -0.1) .. controls (9.75, 1) .. (10.1, 2);
        \draw (10.1, 2) .. controls (9, 1.75) .. (7.9, 2);
        \draw (7.9, 2) .. controls (8.25, 1) .. (7.9, -0.2);
        
        \draw (8.35, 0.5) .. controls (8.7, 0.75) .. (8.85, 1.2);
        \draw (8.85, 1.2) .. controls (9.0, 1.4) .. (9.15, 1.2);
        \draw (9.15, 1.2) .. controls (9.3, 0.75) .. (9.65, 0.5);
        
        \node[fill, red, circle, inner sep=0.75pt] at (9, 0.9) {};
        \node[fill, red, circle, inner sep=0.75pt] at (9.35, 0.45) {};
        
        \draw[red] (9, 0.9) .. controls (9.1, 0.55) .. (9.35, 0.45);
        
        \node at (9, -0.5) {\(\hat{x}\)};
    \end{tikzpicture}
    \caption{A depiction of interpolation through manifold flattening on a manifold in \(\R^{3}\) of dimension \(d = 2\). To interpolate two points on the data manifold, map them through the flattening map \(\fl\) to the flattened space, take their convex interpolants, and then map them back to the data manifold through the reconstruction map \(\re\).}
    \label{fig:idealized_interpolation}
\end{figure}
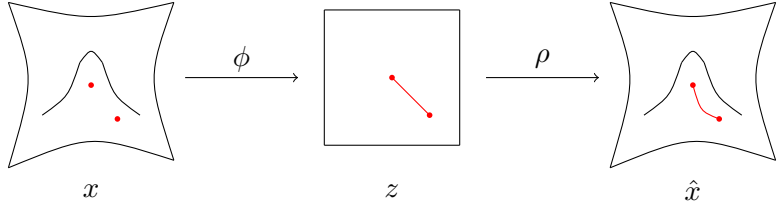

\noindent \textbf{A note on the flattenable assumption for \(\M\)}.
Under the assumption that \(\M\) is flattenable, the converse of \Cref{thm:conv_flat} also holds: any lossless, smooth autoencoding pair generates a flattening. We argue that this is a reasonable assumption for data manifolds, as for any continuous autoencoding pair \(\f, \g\) (e.g. neural network encoder/decoder pair), exactly one of the following must occur:
\begin{enumerate}
    \item \(\f(\M)\) is a convex set, in which case \(\M\) is flattenable; or
    \item \(\f(\M)\) is not convex, in which case samples generated via \(\g(z)\) for any normally distributed \(z\) will fall outside \(\M\) with high probability. 
\end{enumerate}
This sampling method is a common technique within the VAE framework, so manifolds which are amenable to autoencoding via VAEs should generally be flattenable.

\section{Flattening Networks: an algorithmic approach to manifold flattening}

In this section we discuss a computational method for constructing globally-defined flattening and reconstruction maps. In our method, such maps will be a composition of multiple simpler functions, akin to layers in neural networks, as depicted in \Cref{fig:ccnet_layers}. Accordingly, we call our method the \textit{Flattening Networks}, or \textit{FlatNet}. 

However, a large difference between Flattening Networks and traditional neural networks is that our ``layers'' are constructed iteratively in a white-box process with no back-propagation. This design is qualitatively similar to other white-box networks such as described in \cite{chan2022redunet}, but is very different in the details.

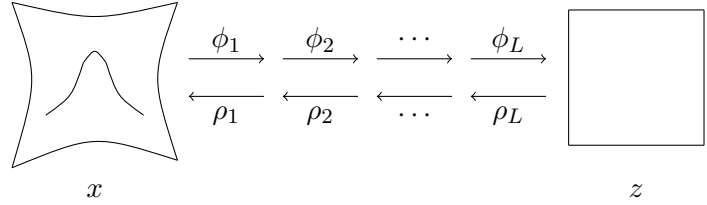
\begin{figure}
    \centering
    \begin{tikzpicture}
        \draw (-0.1, -0.2) .. controls (1, 0.25) .. (2.1, -0.1);
        \draw (2.1, -0.1) .. controls (1.75, 1) .. (2.1, 2);
        \draw (2.1, 2) .. controls (1, 1.75) .. (-0.1, 2);
        \draw (-0.1, 2) .. controls (0.25, 1) .. (-0.1, -0.2);
        
        \draw (0.35, 0.5) .. controls (0.7, 0.75) .. (0.85, 1.2);
        \draw (0.85, 1.2) .. controls (1.0, 1.4) .. (1.15, 1.2);
        \draw (1.15, 1.2) .. controls (1.3, 0.75) .. (1.65, 0.5);
        
        \node at (1, -0.5) {\(x\)};
        
        \draw[->] (2.25, 1.25) -- (3.25, 1.25);
        \node at (2.75, 1.5) {\(\fl_{1}\)};

        \draw[->] (3.5, 1.25) -- (4.5, 1.25);
        \node at (4, 1.5) {\(\fl_{2}\)};

        \draw[->] (4.75, 1.25) -- (5.75, 1.25);
        \node at (5.25, 1.5) {\(\cdots\)};

        \draw[->] (6, 1.25) -- (7, 1.25);
        \node at (6.5, 1.5) {\(\fl_{L}\)};

        \draw[->] (7, 0.75) -- (6, 0.75);
        \node at (6.5, 0.5) {\(\re_{L}\)};
        
        \draw[->] (5.75, 0.75) -- (4.75, 0.75);
        \node at (5.25, 0.5) {\(\cdots\)};
        
        \draw[->] (4.5, 0.75) -- (3.5, 0.75);
        \node at (4, 0.5) {\(\re_{2}\)};
        
        \draw[->] (3.25, 0.75) -- (2.25, 0.75);
        \node at (2.75, 0.5) {\(\re_{1}\)};
        
        \draw (7.3, 0.1) -- (9.1, 0.1);
        \draw (9.1, 0.1) -- (9.1, 1.9);
        \draw (9.1, 1.9) -- (7.3, 1.9);
        \draw (7.3, 1.9) -- (7.3, 0.1);

        \node at (8.2, -0.5) {\(z\)};
        
    \end{tikzpicture}
    \caption{A depiction of the construction process of the flattening and reconstruction pair \((\fl_{\CC}, \re_{\CC})\), where the encoder \(\fl_{\CC} = \fl_{L} \circ \fl_{L - 1} \circ \cdots \circ \fl_{1}\) is constructed from composing flattening layers, and the decoder \(\re_{\CC} = \re_{1} \circ \re_{2} \circ \cdots \circ \re_{L}\) is composed of inversions of each \(\fl_{\ell}\).}
    \label{fig:ccnet_layers}
\end{figure}

We summarize the overall construction of our network now, and then give a more detailed description and motivation in the following sections.
\begin{enumerate}  
    \item Initialize composite flattening and reconstruction maps \(\fl_{\CC} = \re_{\CC} = \id_{\R^{D}} \colon \R^{D} \to \R^{D}\).
    \item For each layer \(\ell \in [L]\):
    \begin{enumerate}
        \item Sample a point \(\xz\) uniformly at random from \(\X\).
        \item Fit a local quadratic model to a neighborhood \(\Nz\) of \(\xz\), and compute a local flattening map \(\fl_{\loc} \colon \R^{D} \to \R^{D}\).
        \item Use a so-called \textit{partition of unity} and the local flattening map \(\fl_{\loc}\) to form a global flattening map \(\fl \colon \R^{D} \to \R^{D}\).
        \item Compute an analytic global approximate inverse, i.e., a global reconstruction map \(\re \colon \R^{D} \to \R^{D}\) such that \(\re \circ \fl \approx \id_{\R^{D}}\).
        \item Compose the constructed layer with existing layers: set \(\fl_{\CC} \gets \fl \circ \fl_{\CC}\) and \(\re_{\CC} \gets \re_{\CC} \circ \re\).
    \end{enumerate}
    \item Return the composite maps \(\fl_{\CC}\) and \(\re_{\CC}\).
\end{enumerate}

In the following sections, we will discuss the motivation and implementation of each of these steps. \\

\begin{figure}[H]
% \vspace{-15mm}
\begin{minipage}{.5\linewidth}
\centering
\subfloat[]{\label{main:a}\includegraphics[width = 0.7\textwidth]{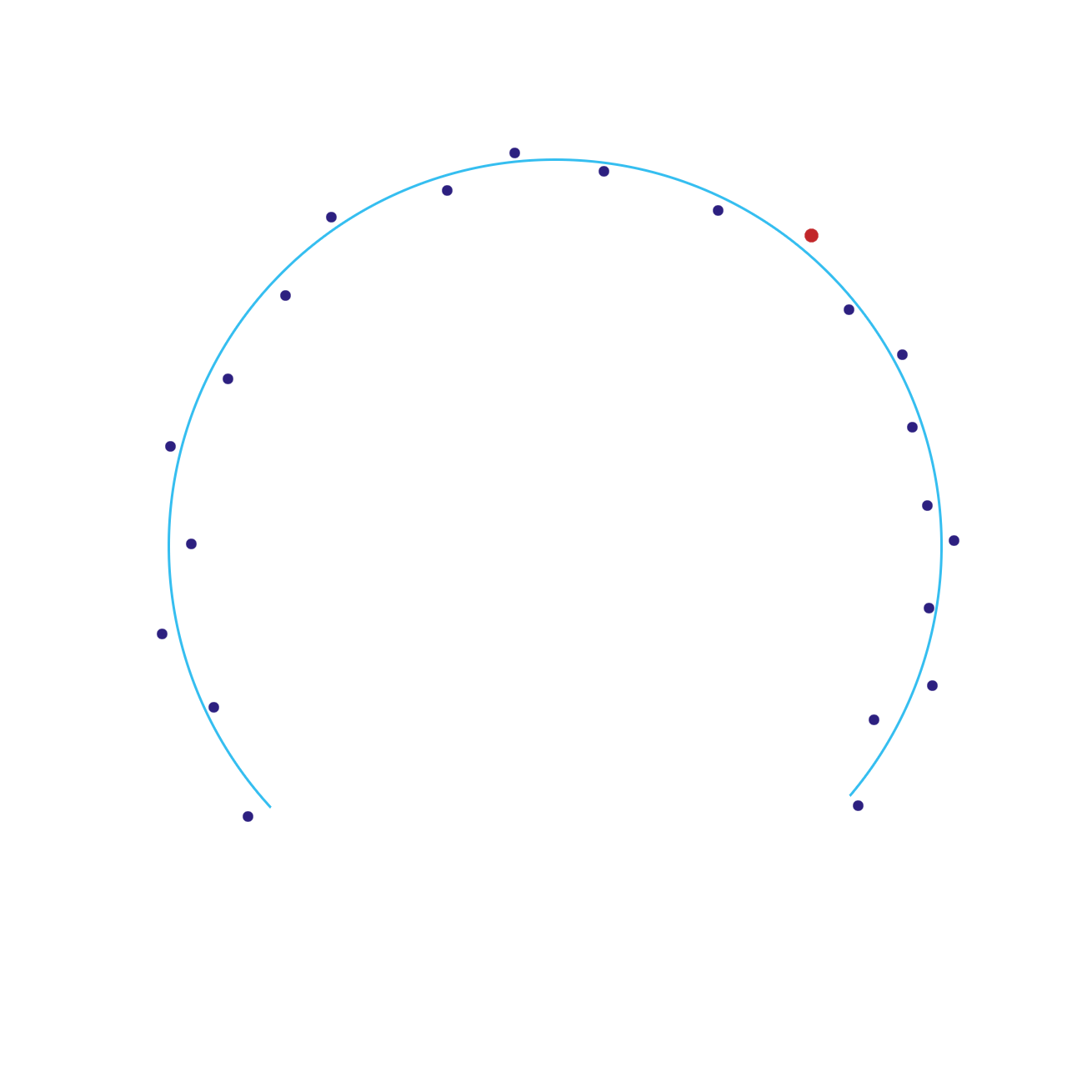}}
\end{minipage}%
\begin{minipage}{.5\linewidth}
\centering
\subfloat[]{\label{main:b}\includegraphics[width = 0.7\textwidth]{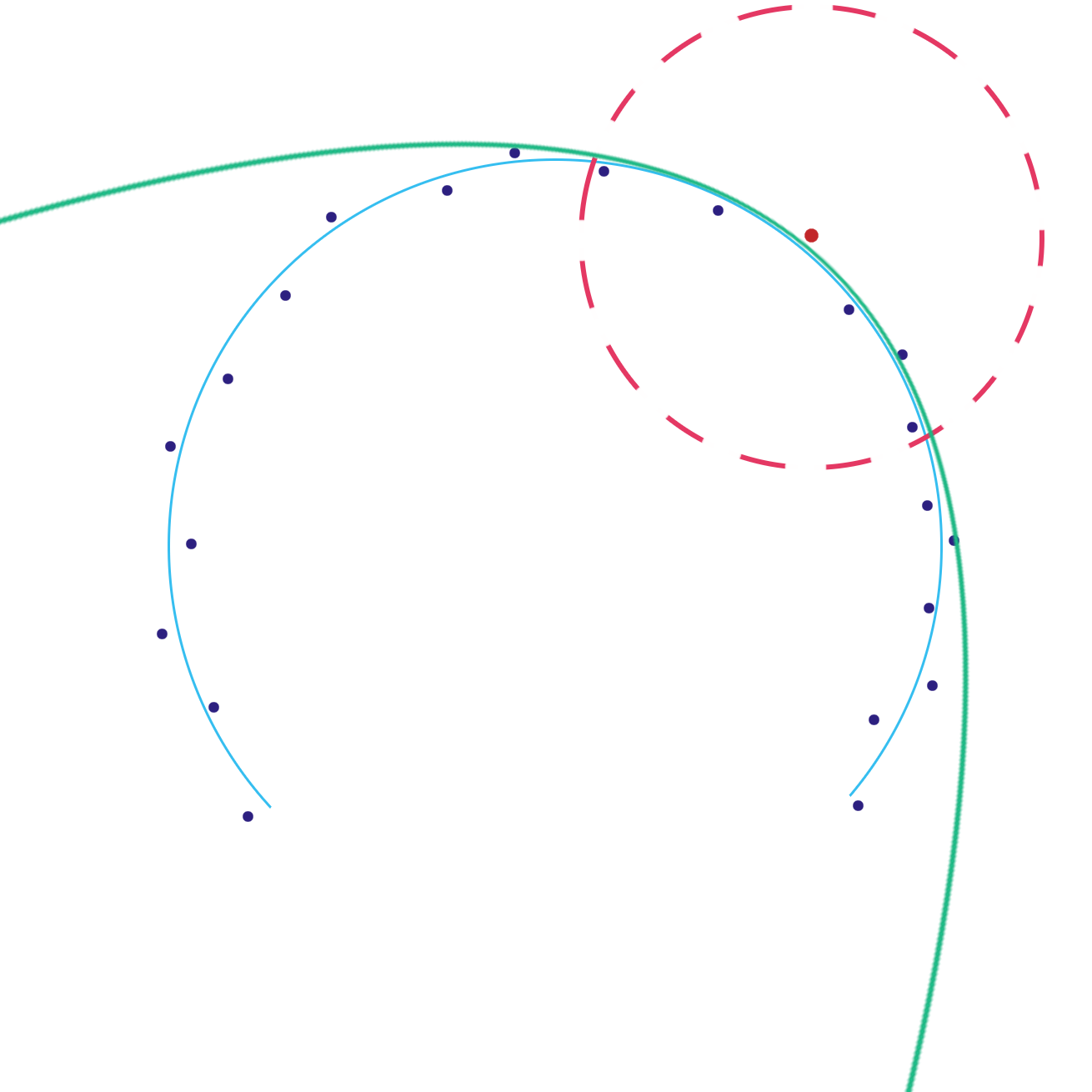}}
\end{minipage}\par\medskip
\centering
\vspace{-18mm}
\subfloat[]{\label{main:c}\includegraphics[width = 0.4\textwidth]{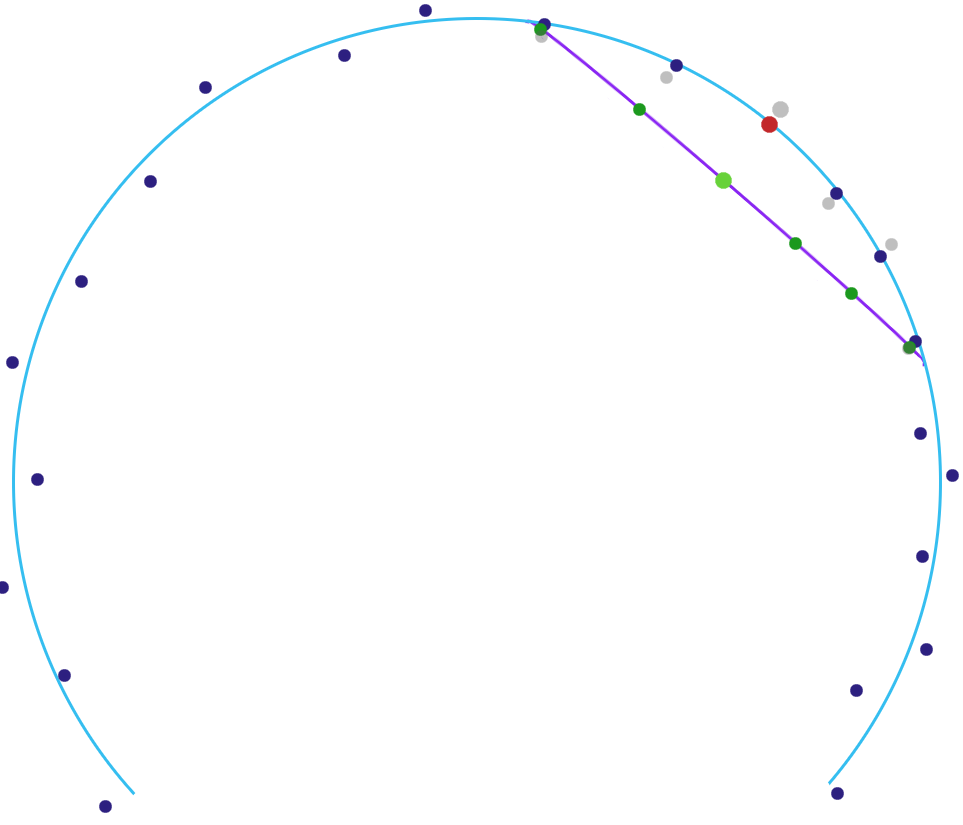}}

\caption{A graphical depiction of one step for our manifold flattening and reconstruction algorithm. \textbf{(a)} From a dataset \(\X\) (dots) drawn near a manifold, randomly choose one data point \(\xz\) (red). \textbf{(b)} We then find the largest radius around the selected point where our simplified model (green) fits the data well, and \textbf{(c)} We use this model to both flatten this patch of the data and reconstruct onto our local model. This process is repeated on the resulting manifold until the manifold is flattened.}
\label{fig:main_algo}
\end{figure}

\subsection{Local quadratic models for data manifolds}

Since we aim to manipulate and reconstruct a dense manifold \(\M\) from finite samples \(\X\), we need some way to model manifolds to fit to the dataset \(\X\). Commonly, the structures of manifolds exploited for manifold learning are local. Even beyond the classical ``locally Euclidean'' definition for manifolds, many manifold processing methods use the fact that, near any fixed point \(\xz \in \M\), the manifold \(\M\) will be approximately linear for a small enough ball around \(\xz\).

While linear models will not be sufficient to solve our problem, as we need to reconstruct the nonlinearities that our method flattens out, we still build our method from the above core principle: \textit{while a data manifold \(\M\) may be complicated and hard to model globally, restricting our attention to one neighborhood at a time allows us to fit a much simpler model accurately to that patch of the manifold.} Such a local model not only tells us how to reconstruct the original manifold after flattening, but also the ``best'' way to flatten the manifold. Concretely, this allows us to construct a map \(\fl_{\loc}\colon \R^{D} \to \R^{D}\) which flattens points in a chosen neighborhood \(\Nz\) around \(\xz\), but whose behavior is necessarily ill-defined far from \(\Nz\). We now discuss the construction of such a map.

The \textit{exponential map} at \(\xz\), denoted \(\exp_{\xz} \colon \TxzM \to \M\), is a local bijection from the \(d\)-dimensional subspace \(\TxzM\) to the full manifold \(\M\). Hence, \(\exp_{\xz}\) perfectly models the local structure of the manifold \(\M\). However, not only is \(\exp_{\xz}\) hard to estimate from finite data, the exponential map is typically only analytically known for very structured manifolds. However, we can estimate approximations to \(\exp_{\xz}\), and restrict our attention to neighborhoods small enough such that this approximation holds. A well-studied approximation scheme are Taylor approximations. A linear approximation is too restrictive, as this removes all local nonlinear information about \(\M\), thus making it impossible to reconstruct after flattening. Thus, we use the next simplest Taylor series model, the second-order Taylor approximation \citep{monera2014taylor}:
\begin{equation}\label{eq:exp_taylor}
    \exp_{\xz}(v) \approx \xz + v + \frac{1}{2}\sff_{\xz}(v, v),
\end{equation}
where \(\sff_{\xz} \colon \TxzM \times \TxzM \to \NxzM\) is the second fundamental form of \(\M\) at \(\xz\) (see \Cref{def:sff}). Based on this expansion we will construct principled local flattening and reconstruction maps \(\fl_{\loc}\) and \(\re_{\loc}\). \\

\noindent \textbf{Proposed local flattening and reconstruction maps.} We propose a local flattening map \(\fl_{\loc} \colon \R^{D} \to \R^{D}\) which is an affine projection operator:
\begin{equation}
    \fl_{\loc}(x) = \P_{S + \xc}\{x\}
\end{equation}
where \(S\) is a linear subspace and \(\xc \in \R^{D}\) is a fixed base point (both to be defined next), and the sum is taken in the usual sense, i.e., \(S + \xc = \{s + \xc \mid s \in S\}\). There are then two variables we need to decide: the subspace \(S\) and the offset \(\xc\). The choices we make here are not claimed to be optimal in any sense, but their design serves two important roles: \(S\) is chosen to maintain local invertibility of \(\fl_{\loc}\), and \(\xc\) is chosen to make the overall flattening process more likely to converge.

First, we choose \(S = \TxzM\). If the approximation in \Cref{eq:exp_taylor} holds, then for any \(v \in \TxzM\) we have
\begin{align}
    \fl_{\loc}(\exp_{\xz}(v))
    &= \P_{\TxzM + \xc}\{\exp_{\xz}(v)\} \\
    &\approx \P_{\TxzM + \xc}\left\{\xz + v + \frac{1}{2}\sff_{\xz}(v, v)\right\} \\
    &= \P_{\TxzM}\left\{\xz + v + \frac{1}{2}\sff_{\xz}(v, v) - \xc\right\} + \xc \\
    &= \underbrace{\P_{\TxzM}\{v\}}_{= v} + \frac{1}{2}\underbrace{\P_{\TxzM}\{\sff_{\xz}(v, v)\}}_{= 0} + \underbrace{\P_{\TxzM}\{\xz - \xc\} + \xc}_{= \P_{\TxzM + \xc}\{\xz\}} \\
    &= v + \P_{\TxzM + \xc}\{\xz\} \\
    &= v + \fl_{\loc}(\xz).
\end{align}
Here \(\P_{\TxzM}\{v\} = v\) because \(v \in \TxzM\) and \(\P_{\TxzM}\{\sff_{\xz}(v, v)\} = 0\) because \(\sff_{\xz}(v, v) \in \NxzM\). Thus the local flattening map
\begin{equation}\label{eq:local_flattening}
    \fl_{\loc}(x) = \P_{\TxzM + \xc}\{x\}
\end{equation}
is (approximately) invertible in the neighborhood \(\Nz\) of \(\xz\) for which the quadratic approximation in \Cref{eq:exp_taylor} holds, say by the local reconstruction map
\begin{equation}\label{eq:local_recon}
    \re_{\loc}(z) = \xz + z - \fl_{\loc}(\xz) + \frac{1}{2}\sff_{\xz}(z - \fl_{\loc}(\xz), z - \fl_{\loc}(\xz)).
\end{equation}
This local invertibility property follows from the below calculation:
\begin{align}
    \re_{\loc}(\fl_{\loc}(\exp_{\xz}(v)))
    &\approx \re_{\loc}(v + \fl_{\loc}(\xz)) \\
    &= \xz + v + \frac{1}{2}\sff_{\xz}(v, v) \\
    &\approx \exp_{\xz}(v).
\end{align}
This choosing of the subspace \(S = \TxzM\) ensures that the local flattening map is (approximately) invertible in the local neighborhood \(\Nz\) of \(\xz\) for which the quadratic approximation in \Cref{eq:exp_taylor} holds. While the choice of \(S\) as the tangent space is clearly optimal in preserving the metric at \(\xz\), the intrinsic volume within the neighborhood can compress at arbitrary ratios depending on the curvature at \(\xz\).
% a more precise characterization of this deformation can be found in \Cref{sec:volume_analysis}
Nonetheless, flattening to the tangent space gives an easy verification that the flattening is invertible.

The offset \(\xc\) is an important design choice, as it directs the global direction of the process induced from repeated local flattenings, and thus controls the convergence of the process. In order to make sure the process converges, we need to design an \(\xc\) that is both (a) pushing the output manifold closer to a converged, globally flattened state, but (b) close enough to our dataset \(\X\) to make the eventual global flattening map \(h\) reasonably smooth. 

Our certificate of convergence is the \textit{convex hull} of the manifold. We use the connection between the flatness and convexity of \(\M\) furnished by \Cref{thm:conv_flat}. Indeed, such an embedded submanifold of Euclidean space is fully flat if and only if it is convex. Thus, any difference between the convex hull of \(\M\), denoted \(\conv(\M)\), and itself, gives an indication of where to flatten the manifold into. This idea is depicted in \Cref{fig:conv_hull_manifold}.

\begin{figure}
    \centering
    \begin{tikzpicture}[domain=0.25:4]
        \draw[draw=white, fill=blue!5] (0.33, 0) -- (1, -1.48) -- (6.05, 1.05) -- (3.5, 1.45) -- (2.8, 1.5) -- (0.6, 1.5);
        \draw[scale=1.5, smooth, variable=\x] plot ({\x}, {sin(180*(1/\x))});
        \node at (1, -1.8) {\(\M\)};
        \node at (3.5, -0.75) {\color{blue!60} \(\conv(\M)\)};
    \end{tikzpicture}
    \caption{A depiction of the convex hull of a manifold \(\M\) versus the manifold itself. If there is no difference between \(\conv(\M)\) and \(\M\) itself, the resulting manifold is convex and thus flattened. We thus design local flattenings to push \(\M\) into its convex hull.}
    \label{fig:conv_hull_manifold}
\end{figure}
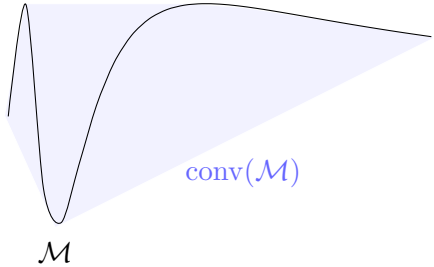

Note that for any probability measure \(P\) supported on \(\M\), we have that the extrinsic average 
\begin{equation}
    \bar{x} \doteq \int_{\M}x\ \mathrm{d}P(x)
\end{equation}
lies within the convex hull. Thus, if we chose our distribution \(P\) to have density w.r.t.~the Haar measure on \(\M\) be proportional to a smooth function \(\pi \colon \R^{D} \to [0, 1]\) such that \(\pi(x) \approx 1\) for all \(x \in \Nz\) and \(\pi\) decays rapidly outside \(\Nz\), to be determined later, then the weighted local average given by the following:
\begin{equation}\label{eq:local_offset_def}
    \xc = \bar{x} \doteq \frac{\int_{\M}x\pi(x)\ \mathrm{d}x}{\int_{\M} \pi(x)\ \mathrm{d}x},
\end{equation}
lies in the convex hull, \(\xc \in \conv(\M)\), while \(\xc\) is still relatively close to \(x\) due to the imposed structure on \(\pi\). Thus, to meet the desired criteria for the local flattening, we choose \(S\) to be the tangent space \(\TxzM\) and the offset \(\xc\) to be the local average \(\xc = \frac{\int_{\M}x\pi(x)\ \mathrm{d}x}{\int_{\M} \pi(x)\ \mathrm{d}x}\).

\subsection{From local to global}

Once we find a local flattening map \(\fl_{\loc} \colon \R^{D} \to \R^{D}\), the next step to make it globally well-defined is to ``glue'' it with a well-behaved globally-defined map outside of the chosen neighborhood \(\Nz\) of \(\xz\). Classically in differential geometry, this is done through a \textit{partition of unity} \citep{dieudonne1937fonctions,lee2012smoothmanifolds}, a smooth scalar function \(\pou \colon \R^{D} \to [0, 1]\) such that \(\pou(x) \approx 1\) for all \(x \in \Nz\) and \(\pou(x) \approx 0\) for all \(x \notin \Nz\). Given the local flattening map \(\fl_{\loc} \colon \R^{D} \to \R^{D}\) computed with respect to the neighborhood \(\Nz\), we can then define a globally-defined flattening map \(\fl \colon \R^{D} \to \R^{D}\) by:
\begin{equation}\label{eq:part_unity}
    \fl(x) \doteq \pou(x)\fl_{\loc}(x) + (1 - \pou(x))x,
\end{equation}
so that \(\fl(x) \approx \fl_{\loc}(x)\) if \(x \in \Nz\) and \(\fl(x) \approx x\) otherwise. 

Of course, there are many choices for \(\pou\). However, as we design a practical algorithm satisfying \Cref{def:main_problem}, our choices become more limited. The following are our two main concerns that will limit the design choice for \(\pou\):
\begin{enumerate}
    \item \(\pou\) should lead to relatively smooth maps constructed as in \Cref{eq:part_unity}. Once we flatten a portion of the manifold, we will need to both further flatten and reconstruct the manifold, and if we create cusps as in \Cref{main:c}, the downstream flattening and reconstruction tasks will be unnecessarily difficult.

    \item \(\pou\) should allow for computable inverses of maps constructed as in \Cref{eq:part_unity}. Even if we are able to invert the local flattening function \(\fl_{\loc}\) locally around \(\xz\) using the local reconstruction function \(\re_{\loc}\), it's not always trivial (or even possible) to invert the globally well-defined function \(\fl\). Nonetheless, this task is crucial for constructing a true autoencoder.
\end{enumerate}

To motivate our choice of \(\pou\), we attempt to compute an approximate inverse of \(\fl\), i.e., a function \(\re\) such that \(\re \circ \fl \approx \id_{\R^{D}}\). To do this we first simplify the expression for \(\fl\):
\begin{align}
    \fl(x)
    &= \pou(x)\fl_{\loc}(x) + (1 - \pou(x))x \\
    &= \pou(x)(\P_{\TxzM + \xc}\{x\}) + (1 - \pou(x))x \\
    &= \pou(x)(\P_{\TxzM}\{x - \xc\} + \xc) + (1 - \pou(x))x \\
    &= \pou(x)\P_{\TxzM}\{x\} - \pou(x)\P_{\TxzM}\{\xc\} + \pou(x)\xc + x - \pou(x)x \\
    &= \pou(x)\P_{\TxzM}\{x\} - \pou(x)\P_{\TxzM}\{\xc\} + \pou(x)\P_{\TxzM}\{\xc\} + \pou(x)\P_{\NxzM}\{\xc\} \\
    &\quad + \P_{\TxzM}\{x\} + \P_{\NxzM}\{x\} - \pou(x)\P_{\TxzM}\{x\} - \pou(x)\P_{\NxzM}\{x\} \nonumber \\
    &= \P_{\TxzM}\{\pou(x)x - \pou(x)\xc + \pou(x)\xc + x - \pou(x)x\} + \P_{\NxzM}\{\pou(x)\xc + x - \pou(x)x\} \\
    &= \P_{\TxzM}\{x\} + \P_{\NxzM}\{\pou(x)\xc + (1 - \pou(x))x\}.
\end{align}
This simplification reveals a geometric interpretation of the flattening map \(\fl\); the component of \(x\) towards the tangent space \(\TxzM\) is preserved, while the component of \(x\) towards the normal space \(\NxzM\) is replaced by a convex combination of \(\xc\) and \(x\) whose weights are given by the partition of unity \(\pou\). The following computation now motivates an approximate inverse of \(\fl\):
\begin{align}
    &\fl(x) + \pou(x)\left(\xz - \P_{\TxzM + \xc}\{\xz\} + \frac{1}{2}\sff_{\xz}(\P_{\TxzM}\{\fl(x) - \xz\}, \P_{\TxzM}\{\fl(x) - \xz\})\right) \\
    &= \fl(x) + \pou(x)\left(\xz - \P_{\TxzM + \xc}\{\xz\} + \frac{1}{2}\sff_{\xz}(\P_{\TxzM}\{x - \xz\}, \P_{\TxzM}\{x - \xz\})\right) \\
    &= \fl(x) + \pou(x)\left(\xz - \P_{\TxzM}\{\xz\} - \P_{\NxzM}\{\xc\} + \frac{1}{2}\sff_{\xz}(\P_{\TxzM}\{x - \xz\}, \P_{\TxzM}\{x - \xz\})\right) \\
    &= \P_{\TxzM}\{x\} + \P_{\NxzM}\{\pou(x)\xc + (1 - \pou(x))x\} \\
    &\quad + \pou(x)\left(\xz - \P_{\TxzM}\{\xz\} - \P_{\NxzM}\{\xc\} + \frac{1}{2}\sff_{\xz}(\P_{\TxzM}\{x - \xz\}, \P_{\TxzM}\{x - \xz\})\right) \nonumber \\
    &= \P_{\TxzM}\{x\} + \pou(x)\P_{\NxzM}\{\xc\} + \P_{\NxzM}\{x\} - \pou(x)\P_{\NxzM}\{x\} \\
    &\quad + \pou(x)\left(\scriptstyle \P_{\TxzM}\{\xz\} + \P_{\NxzM}\{\xz\} - \P_{\TxzM}\{\xz\} - \P_{\NxzM}\{\xc\} + \frac{1}{2}\sff_{\xz}(\P_{\TxzM}\{x - \xz\}, \P_{\TxzM}\{x - \xz\})\right) \nonumber \\
    &= \P_{\TxzM}\{x\} + \P_{\NxzM}\{x\} + \pou(x)\left(-\P_{\NxzM}\{x - \xz\} + \frac{1}{2}\sff_{\xz}(\P_{\TxzM}\{x - \xz\}, \P_{\TxzM}\{x - \xz\})\right) \\
    &= x + \pou(x)\underbrace{\left(-\P_{\NxzM}\{x - \xz\} + \frac{1}{2}\sff_{\xz}(\P_{\TxzM}\{x - \xz\}, \P_{\TxzM}\{x - \xz\})\right)}_{\approx 0} \label{eq:approx_recon} \\
    &\approx x.
\end{align}
Here we make the approximation \(\P_{\NxzM}\{x - \xz\} \approx \frac{1}{2}\sff_{\xz}(\P_{\TxzM}\{x - \xz\}, \P_{\TxzM}\{x - \xz\})\) which follows from the definition of the second fundamental form.

Thus, we are motivated to set \(z = \fl(x)\) and obtain the reconstruction map
\begin{equation}
    \re(z) \doteq z + \pou(x)\left(\xz - \P_{\TxzM + \xc}\{\xz\} + \frac{1}{2}\sff_{\xz}(\P_{\TxzM}\{z - \xz\}, \P_{\TxzM}\{z - \xz\})\right).
\end{equation}
However, this is not a well-defined global map because the coefficient \(\pou(x)\) depends on \(x \approx \re(z)\). To solve this problem we use the following lemma:
\begin{lemma}\label{lem:local_inv}
    Let \(\lambda > 0\), let \(\pou(x) = e^{-\lambda\norm{x - \xz}_{2}^{2}}\), and \(\fl\) be defined as in \Cref{eq:part_unity}. There exists a unique smooth function \(\pouinv \colon \R^{D} \to \R\) such that \(\pouinv(\fl(x)) = \pou(x)\) for all \(x \in \R^{D}\).
\end{lemma}
The proof of this lemma is deferred to \Cref{sec:proofs}. It is not apparent from the theorem statement, but clear from the proof, is that \(\pouinv(z)\) is efficiently computable, requiring only the inversion of a globally invertible smooth scalar function, which we use the secant method to compute.

This lemma motivates our choice of \(\pou(x) = e^{-\lambda \norm{x - \xz}_{2}^{2}}\) for some \(\lambda > 0\), and thus our global reconstruction map is well-defined as 
\begin{equation}\label{eq:global_recon}
    \re(z) \doteq z + \pouinv(z)\left(\xz - \P_{\TxzM + \xc}\{\xz\} + \frac{1}{2}\sff_{\xz}(\P_{\TxzM}\{z - \xz\}, \P_{\TxzM}\{z - \xz\})\right).
\end{equation}

One important interpretation of the \(\lambda\) parameter in the partition of unity is that it controls the ``radius'' of the partition. In particular, the function \(r \mapsto e^{-\lambda r^{2}}\) is monotonically decreasing in \(r \geq 0\). If we want to find the radius \(r\) for which \(e^{-\lambda r^{2}} \leq \eps\), where \(\eps > 0\) is some small constant, we obtain
\begin{equation}
    r \geq \sqrt{\frac{\log(1/\eps)}{\lambda}} = O(\lambda^{-1/2}).
\end{equation}
Thus we can think of the ``radius'' of \(\pou\) as proportional to \(\lambda^{-1/2}\).

Finally, recall that we defined the local average \(x_{c} = \frac{\int_{\M}x\pi(x)\ \mathrm{d}x}{\int_{\M}\pi(x)\ \mathrm{d}x}\) for some smooth function \(\pi \colon \R^{D} \to [0, 1]\) such that \(\pi(x) \approx 1\) for all \(x \in \Nz\) and \(\pi\) decays rapidly outside \(\Nz\). In order to achieve these two criteria while making sure that \(\pi\) is smooth and \(\xc\) is close to \(x\), we simply choose \(\pi = \pou\), so that \(\xc = \frac{\int_{\M}x\pou(x)\ \mathrm{d}x}{\int_{\M}\pou(x)\ \mathrm{d}x}\).

\subsection{Estimation from finite samples}

In this section we discuss estimation and computation of the various quantities discussed in the prior sub-sections. Generic estimates are labeled with a tilde, i.e., \(\txc\) or \(\tTxzM\), while optimal estimates are labeled with a hat, i.e., \(\hxc\) or \(\hTxzM\). \\

\noindent \textbf{Estimating the local average, tangent space and second fundamental form.} Suppose that we already know the intrinsic dimension \(d\) of the manifold, and that we already know the parameters for the partition of unity \(\pou\) (these will be estimated automatically, as we describe momentarily). Then the only things left to estimate are the local average \(\xc\), the tangent space \(\TxzM\) (and technically also the normal space \(\NxzM = (\TxzM)^{\perp}\)), and the second fundamental form \(\sff_{\xz}\). 

The local average \(\xc\) with respect to a particular partition of unity \(\psi\) can be efficiently approximated by the sample average: 
\begin{equation}\label{eq:local_mean_approx}
    \xc \approx \hxc \doteq \frac{\sum_{i = 1}^{N}x_{i}\pou(x_{i})}{\sum_{i = 1}^{N}\pou(x_{i})}.
\end{equation}
However, the local tangent space \(\TxzM\) is harder to approximate from a finite set of samples. A popular method to estimate \(\TxzM\) from finite data is local PCA \citep{hoppe1993mesh,oue1996asymptotics}; however, these methods rely on an assumption that the manifold \(\M\) is close to linear around the base point \(\xz\), which holds less strongly if \(\M\) has non-negligible curvature. While improvements have been made past local PCA \citep{zhang2011improved}, there are still issues when the data has both extrinsic curvature and extrinsic noise.

We may estimate the tangent space \(\TxzM\) from finite samples by minimizing the autoencoding loss, which is a sum of terms of the form 
\begin{align}
    \norm{x_{i} - \re(\fl(x_{i}))}_{2}^{2}
    &= \norm{\pou(x_{i})\left(-\P_{\NxzM}\{x_{i} - \xz\} + \frac{1}{2}\sff_{\xz}(\P_{\TxzM}\{x_{i} - \xz\}, \P_{\TxzM}\{x_{i} - \xz\})\right)}_{2}^{2} \\
    &= \pou(x_{i})^{2}\norm{\frac{1}{2}\sff_{\xz}(\P_{\TxzM}\{x_{i} - \xz\}, \P_{\TxzM}\{x_{i} - \xz\}) - \P_{\NxzM}\{x_{i} - \xz\}}_{2}^{2},
\end{align}
where the first equality is from \Cref{eq:approx_recon}. For the sake of optimization, we parameterize \(\TxzM\) by an orthogonal matrix \(U \in \mathsf{O}(D, d)\) such that \(UU^{\top} = \P_{\TxzM}\). This gives 
\begin{equation}
    \norm{x_{i} - \re(\fl(x_{i}))}_{2}^{2} = \pou(x_{i})^{2}\norm{\frac{1}{2}\sff_{\xz}(UU^{\top}(x_{i} - \xz), UU^{\top}(x_{i} - \xz)) - (I - UU^{\top})(x_{i} - \xz)}_{2}^{2}.
\end{equation}
Finally, we parameterize the second fundamental form (or more specifically the quantity \(\frac{1}{2}\sff_{\xz}\)) by its coordinates with respect to the basis defined by the columns of \(U\), represented as a multi-array \(V \in \R^{D \times d \times d}\) in the following way:
\begin{equation}\label{eq:def_V_bilinear_form}
    \frac{1}{2}\sff_{\xz}(w_{1}, w_{2}) = V(w_{1}, w_{2}) \doteq \sum_{j = 1}^{d}\sum_{k = 1}^{d}v_{jk}\ip{u_{j}}{w_{1}}\ip{u_{k}}{w_{2}}
\end{equation}
where \(v_{jk} \in \R^{D}\) are the slices through the first coordinate of \(V\). Thus, we can estimate \(U\) and \(V\) through a computationally feasible autoencoding loss:
\begin{equation}\label{eq:autoencoding_loss}
    \mathcal{L}_{\pou}(\tU, \tV) \doteq \frac{1}{N}\sum_{i = 1}^{N}\pou(x_{i})^{2}\norm{\tV(\tU\tU^{\top}(x_{i} - \xz), \tU\tU^{\top}(x_{i} - \xz)) - (I - \tU\tU^{\top})(x_{i} - \xz)}_{2}^{2},
\end{equation}
This yields a global reconstruction problem, whose value we label \(\mathsf{Recon}(d, \pou)\):
\begin{equation}\label{eq:autoencoding_problem}
    \mathsf{Recon}(d, \pou) \doteq \inf_{\substack{U \in \mathsf{O}(D, d) \\ V \in \R^{D \times d \times d}}} \mathcal{L}_{\pou}(U, V).  \end{equation}
The solutions to this problem are our estimates for the tangent space basis \(\hU\) and second fundamental form matrix \(\hV\), which translate into estimates for the \textit{actual} tangent space \(\hTxzM\) and second fundamental form \(\hsff_{\xz}\) in a straightforward way.

While the above problem \eqref{eq:autoencoding_problem} may look daunting, note that for a fixed \(\tU \in \mathsf{O}(D, d)\), solving for the entries of \(\hV\) is a least-squares problem and may be efficiently solved in closed-form. The remaining component of the optimization is to solve for \(\hU\), which is an \(O(Dd)\)-dimensional optimization over the Stiefel manifold \(\mathsf{O}(D, d)\). While the Stiefel manifold is a non-convex set, optimization over the Steifel manifold is well studied \citep{fraikin2007optimization,li2020efficient,qu2020geometric,zhai2020complete,zhai2020understanding}. We reiterate that all associated complexities fulfill our complexity requirement of \(O(NDd^{k})\) for some positive integer \(k\).

There is one final complexity that we have swept under the rug: the second fundamental form \(\sff_{\xz}\) has codomain \(\NxzM\), which is a \(D - d\) dimensional subspace, but our bilinear \(V\) map has codomain \(\R^{D}\) in general, and at first glance is not restricted to any \(D - d\) dimensional subspace. The resolution to this dilemma is that if we fix a \(\tU \in \mathsf{O}(D, d)\) and solve for \(\hV\) using least squares, the resulting bilinear \(\hV\) map will always have codomain equal to \(\mathsf{Im}(\tU)^{\perp}\) due to the structure of the problem. In particular, if \(\mathsf{Im}(\tU) = \TxzM\) then \(\hV\) will have codomain \(\NxzM\). We now formally state this result; the proof can be found in Appendix \ref{sec:proofs}.

\begin{proposition}\label{prop:V_normal_space}
    Fix \(\tU \in \mathsf{O}(D, d)\). Further suppose that $\pou(x_i) > 0$ for all $x_i \in \X$. Then any solution \(\hV\) to the problem
    \begin{equation}
        \inf_{\tV \in \R^{D \times d \times d}}\mathcal{L}_{\psi}(\tU, \tV)
    \end{equation}
    has the following properties:
    \begin{enumerate}
        \item $\tU^\top \hV(x_i - x_0, x_i - x_0) = 0$ for all $x_i \in X$
    \end{enumerate}
    For the following parts: define the matrix $A \in \R^{N\times \frac{1}{2}(d^2 + d)}$ in the following way: let $B \in \R^{N\times d\times d}$ be the 3-tensor with entry values $B_{ijk} = \ip{u_j}{x_i - x_0}\ip{u_k}{x_i - x_0}$. Let $B_i \in \R^{d\times d}$ be the $i$\textsuperscript{th} slice of the tensor $B$, fixing the first dimension. Finally, let the $i$\textsuperscript{th} row of $A$ be the vectorized upper-diagonal component of $B_i$.

    \begin{enumerate}
        \item [2.] If $A$ has full column-rank, i.e. $\mathrm{rank}(A) = \frac{1}{2}(d^2 - d)$, then $\hV$ is unique.
        \item [3.] Further, if $A$ has full column-rank, then \(\tU^{\top}\hv_{jk} = 0\) for all \(1 \leq j, k \leq d\), therefore $\tU^\top \hV(w_1, w_2) = 0$ for all $w_1, w_2 \in \R^D$.
    \end{enumerate}
\end{proposition}

There are still a few parameters which the problem in \Cref{eq:autoencoding_problem} relies on: namely, the intrinsic dimension \(d\) of the manifold and the partition of unity \(\pou\). \\

\noindent \textbf{Estimating the local dimension.} One such parameter is the intrinsic dimension \(d\). Since \(d\) is equivalently the number of basis vectors used to locally represent the data: 
\begin{center}
    \textit{The local dimension \(d\) is chosen such that each encoder layer is as sparse and efficient as possible.} 
\end{center}
This directly translates to finding the smallest \(d\) such that we can still locally reconstruct the data up to a desired precision \(\eps_{\dim} > 0\):
\begin{align}
    \hd = \min_{\td \geq 0} \quad
    & \td \label{eq:auto_dim} \\
    \text{s.t.} \quad
    & \textsf{Recon}(\td, \pou_{0}) \leq \eps_{\dim}. \label{eq:auto_dim_constraints}
\end{align}
Here, since we have not configured a good partition of unity \(\pou\) yet, we use another, ``default'' partition of unity \(\pou_{0} \colon \R^{D} \to \R\) given by
\begin{equation}\label{eq:pou_dim}
    \pou_{0}(x) = e^{-\lambda_0\norm{x - \xz}_{2}^{2}}
\end{equation}
where \(\lambda_0 > 0\) is a scale parameter which is set as a hyperparameter to the algorithm. Heuristically it controls any set radius about \(\xz\) for which we expect the dataset \(\X\) to at least have samples in every intrinsic direction; we should think of this radius as proportional to \(\lambda_0^{-1/2}\), as per our previous discussion.

This optimization is in similar spirit to fixed-rank approaches to low-rank matrix completion, such as ALS \citep{takacs2012alternating} and Riemannian methods \citep{vandereycken2013low}, where we can adaptively choose the dimension \(\hd\) at each point \(\xz\) by finding the minimal \(\td\) such that the optimization problem in \Cref{eq:autoencoding_problem} achieves some desired threshold loss \(\eps_{\dim}\).

Indeed, \(\hd\) is designed to model the intrinsic dimension of \(\M\), since any local flattening of lower dimension would collapse a direction in the tangent space and thus not be locally invertible. Since in practice we expect \(d \ll D\), we can afford to solve \Cref{eq:auto_dim} by iteratively solving \Cref{eq:autoencoding_problem} from \(d = 1\) upwards. Further, since theoretically \(d\) should be the same at every \(\xz \in \M\), we can start the search at the next iterate from the previous estimate \(\hd\) to save time. \\ 

\noindent \textbf{Learning a good partition of unity.} Once we estimate the dimension \(\hd\), the only parameter left is the choice of \(\lambda\) in the partition of unity, i.e., 
\begin{equation}\label{eq:parameterized_pou}
    \pou(x) = \pou_{\lambda}(x) = e^{-\lambda \norm{x - \xz}_{2}^{2}}.
\end{equation}
Again, we know that \(\lambda^{-1/2}\) is proportional to the radius of our partition, which is really the radius of the neighborhood in which our quadratic model approximately holds.

In similar spirit to \(d\), we choose \(\lambda\) to make the individual layer as effective as possible, by maximizing the radius of the partition and thus the affected radius of the layer:
\begin{center}
    \textit{The radius parameter \(\lambda^{-1/2}\) is chosen such that each encoder layer is as productive as possible.} 
\end{center}
While \(d\) is optimized to use the fewest number of parameters within an individual layer (analogous to \textit{network width}), \(\lambda\) is chosen to minimize the number of needed layers (analogous to \textit{network depth}) by ensuring each layer makes as much progress towards a flattening as possible. Analogously to \(d\), then, we can formalize the above statement into a constrained optimization problem:
\begin{align}
    \hat{\lambda} = \inf_{\tilde{\lambda} > 0} \quad 
    &\tilde{\lambda} \\
    \text{s.t.} \quad 
    &  \textsf{Recon}(\hd, \pou_{\tilde{\lambda}}) \leq \eps_{\mathrm{POU}}.
\end{align}
Since \(\textsf{Recon}(\hd, \pou_{\tilde{\lambda}})\) is a monotonic function with respect to \(\tilde{\lambda}\) (and the monotonicity is non-strict only if \(\textsf{Recon}(\hd, \pou_{\tilde{\lambda}}) = 0\)), if we define 
\begin{equation}\label{eq:ell_def}
    \ell(\tilde{\lambda}) \doteq \textsf{Recon}(\hd, \pou_{\tilde{\lambda}})
\end{equation}
then the above optimization problem amounts to computing 
\begin{equation}\label{eq:lambda_hat_choice}
    \hat{\lambda} = \ell^{-1}(\eps_{\mathrm{POU}}).
\end{equation}
Since \(\ell\) is a function from \(\R\) to \(\R\), its inversion can be efficiently computed by e.g.~the secant method.

One issue with naively implementing this optimization is that we only have finite samples \(\X\) from \(\M\). Thus there is a smallest radius (largest \(\tilde{\lambda}\)) such that, for any smaller radii, there are not enough points in \(\X\) which are at most this distance to \(\xz\) to make the problem in \Cref{eq:autoencoding_problem} well-conditioned. Thus, we introduce another hyperparameter \(\lambda_{\max}\) which controls the maximum permissible value of \(\tilde{\lambda}\). The optimization problem becomes 
\begin{align}
    \hat{\lambda} = \inf_{\tilde{\lambda} \in (0, \lambda_{\max})} \quad 
    &\tilde{\lambda} \\
    \text{s.t.} \quad 
    & \textsf{Recon}(\hd, \pou_{\tilde{\lambda}}) \leq \eps_{\mathrm{POU}}.
\end{align}
Similarly to before, we can exploit the monotonicity property of \(\pou_{\tilde{\lambda}}\) to obtain
\begin{equation}
    \hat{\lambda} = \min\{\lambda_{\max}, \ell^{-1}(\eps_{\mathrm{POU}})\}.
\end{equation}
The issue is that if \(\ell^{-1}(\eps_{\mathrm{POU}}) > \lambda_{\max} = \hat{\lambda}\) then we do \textit{not} have \(\textsf{Recon}(\hd, \pou_{\hat{\lambda}}) \leq \eps_{\mathrm{POU}}\). To fix this, we amend the partition of unity with a multiplicative constant:
\begin{equation}\label{eq:full_pou_definition}
    \pou(x) \doteq \pou_{\hat{\lambda}}(x) = \min\left\{\alpha_{\max}, \sqrt{\frac{\eps_{\mathrm{POU}}}{\ell(\hat{\lambda})}}\right\} \cdot e^{-\hat{\lambda}\norm{x - \xz}_{2}^{2}}
\end{equation}
where \(\alpha_{\max} \in (0, 1]\) is a user-set hyperparameter which controls the smoothness of the boundary of the partition of unity created by the flattening procedure in each step; in particular, if \(\alpha_{\max} = 1\), then non-smooth cusps appear in the flattening procedure, so we usually take \(\alpha_{\max} < 1\).
\\

\noindent \textbf{A full description of the algorithm.} With all these details, we are ready to compose a full description of the algorithm. At each iteration, we estimate the local dimension, compute a good partition of unity, estimate the local mean, tangent space, and second fundamental form, and compose our local flattening map, then we glue it with a partition of unity to get a global flattening map. Finally, we compute its inverse, the global reconstruction map, and append them to our multi-layer flattening and reconstruction maps, iteratively forming the Flattening Network.

More formally, we propose the following algorithm for training Flattening Networks (FlatNets):
\begin{algorithm}[H]
    \caption{Construction of FlatNet.}
    \label{alg:ccnet_construction}
    \begin{algorithmic}[1]
        \Function{FlatNetConstruction}{$\X, L, \eps_{\dim}, \lambda_0, \eps_{\mathrm{POU}}, \lambda_{\max}, \alpha_{\max}$}
            \State{Initialize \(\fl_{\CC}, \re_{\CC} \gets \id_{\R^{D}}\)}
            \For{\(\ell \in [L]\)}
                \State{Sample a random point \(\xz \in \X\)} 

                \State{{\color{gray} \texttt{\% Estimate the local dimension}}}
                \State{Define \(\pou_{0} \colon x \mapsto e^{-\lambda_0\norm{x - \xz}_{2}^{2}}\)} \Comment{\Cref{eq:pou_dim}}
                \State{Compute \(\hd \gets \min\{\td \in [D] \colon \textsf{Recon}(\td, \pou_{0}) \leq \eps_{\dim}\}\)} \Comment{\Cref{eq:auto_dim,eq:auto_dim_constraints}}

                \State{{\color{gray} \texttt{\% Compute the partition of unity}}}
                \State{For each \(\tilde{\lambda}\), define \(\pou_{\tilde{\lambda}} \colon x \mapsto e^{-\tilde{\lambda}\norm{x - \xz}_{2}^{2}}\)} \Comment{\Cref{eq:parameterized_pou}}
                \State{Define \(\ell \colon \tilde{\lambda} \mapsto \textsf{Recon}(\hd, \pou_{\tilde{\lambda}})\)} \Comment{\Cref{eq:ell_def}}
                \State{Compute \(\hat{\lambda} \gets \min\{\lambda_{\max}, \ell^{-1}(\eps_{\mathrm{POU}})\}\)} \Comment{\Cref{eq:lambda_hat_choice}}
                \State{Define \(\pou \colon x \mapsto \min\Big\{\alpha_{\max}, \sqrt{\eps_{\mathrm{POU}}/\ell(\hat{\lambda})}\Big\} \cdot e^{-\hat{\lambda}\norm{x - \xz}_{2}^{2}}\)} \Comment{\Cref{eq:full_pou_definition}} 

                \State{\color{gray} \texttt{\% Estimate the local average, tangent space, and second fundamental form}}
                \State{Compute \(\hxc \gets \displaystyle \frac{\sum_{i = 1}^{N}x_{i}\pou(x_{i})}{\sum_{i = 1}^{N}\pou(x_{i})}\)} \Comment{\Cref{eq:local_mean_approx}}
                \State{Compute \(\hU, \hV \gets \) solutions of \(\textsf{Recon}(\hd, \pou)\)} \Comment{\Cref{eq:autoencoding_problem}}

                \State{{\color{gray} \texttt{\% Compute local and global flattening maps and and global reconstruction map}}}
                \State{Define \(\fl_{\loc} \colon x \mapsto \hU\hU^{\top}(x - \hxc) + \hxc\)} \Comment{\Cref{eq:local_flattening}}
                \State{Define \(\fl \colon x \mapsto \pou(x)\fl_{\loc}(x) + (1 - \pou(x))x\)} \Comment{\Cref{eq:part_unity}}
                \State{Compute \(\pouinv\) such that \(\pouinv(\fl(x)) = \pou(x)\) for all \(x \in \R^{D}\)} \Comment{\Cref{lem:local_inv}}
                \State{Define \(\re \colon z \mapsto z + \pouinv(z)\left(\xz - \hU\hU^{\top}(\xz - \hxc) + \hxc + \hV(\hU\hU^{\top}(z - \xz), \hU\hU^{\top}(z - \xz))\right)\)} \Comment{\Cref{eq:global_recon}}

                \State{{\color{gray} \texttt{\% Compose with previous layers}}}
                \State{Redefine \(\fl_{\CC} \gets \fl \circ \fl_{\CC}\)}
                \State{Redefine \(\re_{\CC} \gets \re_{\CC} \circ \re\)}
            \EndFor
            \State{\Return{\(\fl_{\CC}, \re_{\CC}\)}}
        \EndFunction
    \end{algorithmic}
\end{algorithm}
In \Cref{sec:time_complexity}, we show that Flattening Networks meet the scalability requirements detailed in \Cref{sec:prob_formulation}.

\section{Algorithm intuition and convergence analysis}
In this section, we provide the reader with some understandings and interpretations of the presented Algorithm \ref{alg:ccnet_construction}.

\subsection{Network diagram}

Recall that our method builds a flattening pair \(\fl_{\CC}, \re_{\CC}\) by repeatedly constructing and composing forward and backwards maps: \(\fl_{\CC} = \fl_{L} \circ \cdots \circ \fl_{1}\) and \(\re_{\CC} = \re_{1} \circ \cdots \circ \re_{L}\), where each layer pair \(\fl_{\ell}, \re_{\ell}\) takes the following form:
\begin{align}
    \fl_{\ell}(x) 
    &= \pou(x)(UU^\top(x - x_c) + x_c) + (1 - \pou(x))x,\\
    &= \pou(x)(U U^\top x + (I - U U^\top)x_c) + (1-\pou(x))x, \\
    \re_\ell(z) &= z + \pouinv(z)\left(\xz - \P_{\TxzM + \xc}\{\xz\} + \frac{1}{2}V(\P_{\TxzM}\{z - \xz\}, \P_{\TxzM}\{z - \xz\})\right),\\
    &= z + \pouinv(z)\left((I - UU^\top)(x_{0} - x_c) + \frac{1}{2}\sum_{j = 1}^{d}\sum_{k = 1}^{d}v_{jk}\ip{u_{j}}{z - x_0}\ip{u_{k}}{z - x_0}\right),
\end{align}
where \(U \in \R^{D\times d}\) with \(U^\top U = I\), \(V \in \R^{D\times d \times d}\), \(\xz \in \X\), \(\pou(x) = \alpha e^{-\lambda\norm{x - \xz}_{2}^{2}}\) for algorithmically chosen \(\alpha, \lambda > 0\), \(\xc = \frac{\sum_{i = 1}^{N}x_{i}\pou(x_{i})}{\sum_{i = 1}^{N}\pou(x_{i})}\), and \(\pouinv(z)\) is the computed scalar function such that \(\pouinv(\psi(x)) = \pou(x)\). These maps are graphically depicted in \Cref{fig:layer_diagram}.

\begin{figure}
    \centering
    \begin{tikzpicture}[scale=7]
        
        % f map

        \node[scale=1.5] at (-0.07, 0.2) {\(x\)};
        \draw[line width = 1, black] (-0.015, 0) -- (-0.015, 0.4);
        
        \draw[draw=white, fill={rgb,255:red,217;green,252;blue,222}] (0, 0) -- (0, 0.4) -- (0.25, 0.25) -- (0.25, 0.15);
        \node[scale=1.5] at (0.125, 0.2) {\(U^\top\)};

        \draw[line width = 1, black] (0.265, 0.25) -- (0.265, 0.15);

        \draw[draw=white, fill={rgb,255:red,217;green,252;blue,222}] (0.28, 0.15) -- (0.53, 0) -- (0.53, 0.4) -- (0.28, 0.25);
        \node[scale=1.5] at (0.405, 0.2) {\(U\)};

        \draw[-stealth] (-0.015, -0.01) .. controls (0.2, -0.23) and (0.55, -0.23) .. (0.7, 0.17);

        \draw[-stealth] (0.565, 0.2) -- (0.675, 0.2);

        \node at (0.7, 0.2) {\(\bigoplus\)};

        \draw[-stealth] (0.73, 0.2) -- (0.835, 0.2);

        \draw[line width = 1, black] (0.865, 0) -- (0.865, 0.4);
        \node[scale=1.5] at (1.05, 0.2) {\(\fl(x) = z\)};

        % g map
        \node[scale=1.5] at (-0.07, -0.4) {\(z\)};
        \draw[line width = 1, black] (-0.015, -0.2) -- (-0.015, -0.6);

        \draw[-stealth] (0, -0.4) -- (0.15, -0.4);
        
        \draw[draw=black] (0.17, -0.3) -- (0.37, -0.3) -- (0.37, -0.5) -- (0.17, -0.5) -- (0.17, -0.3);
        \node[scale=1.5] at (0.27, -0.4) {\(\sff_{U, V}\)};
        
        \draw[-stealth] (0.4, -0.4) -- (0.675, -0.4);
        \draw[-stealth] (-0.015, -0.61) .. controls (0.2, -0.83) and (0.55, -0.83) .. (0.7, -0.43);

        \node at (0.538, -0.37) {\(\pouinv(z)\)};
        \node at (0.7, -0.4) {\(\bigoplus\)};
        
        \draw[-stealth] (0.73, -0.4) -- (0.835, -0.4);

        \draw[line width = 1, black] (0.865, -0.2) -- (0.865, -0.6);
        \node[scale=1.5] at (1.05, -0.4) {\(\re(z) \approx x\)};

    \end{tikzpicture}
    \caption{Network diagrams for each individual layer pair, with the encoder $\fl_\ell : \R^D\to \R^D$ depicted on the top and the decoder $\re_\ell : \R^D \to \R^D$ depicted on the bottom. Note the encoder's structure resembles a multi-layer percepton layer with a residual connection; here though, the depicted sum is not a plain sum, but a weighed sum based on the Boltzmann distribution $\pou(x)$. We have omitted the bias terms that arise from constants in the expressions of $\fl, \re$. Note that in generalizing this framework, $\sff_{U, V}$ can be replaced with any local, interpretable model of the data.}
    \label{fig:layer_diagram}
\end{figure}
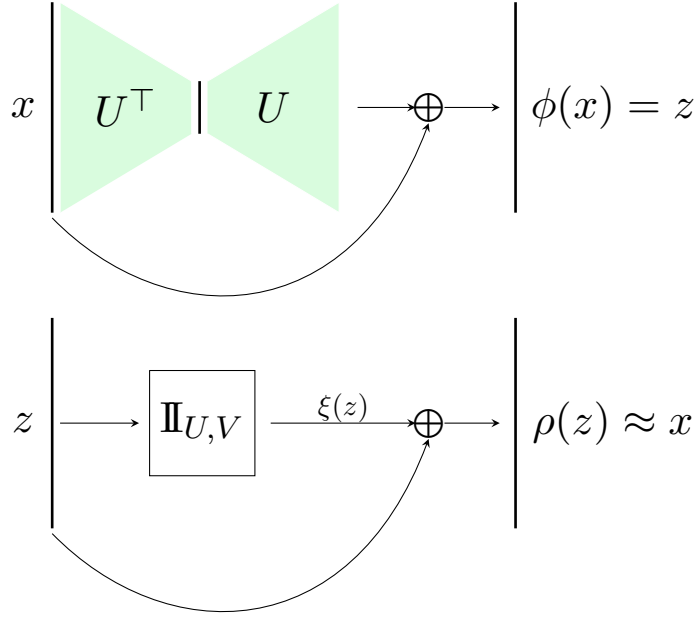

\subsection{What does each iteration look like?}

To build some intuition for what each iteration of the network construction algorithm is doing, in \Cref{fig:flow_prog} we provide an example run of the algorithm on data sampled from a three-quarters-circle curve, and graph the output of the flattening map \(\fl_{\mathrm{CC}}\) after \(0\), \(20\), \(40\),  \(60\), and \(80\) layers have been constructed. More formally, let \(\fl_{1:\ell}\) be the flattening map composed of the first \(\ell\) layers of \(\fl_{\mathrm{CC}}\). Then define
\begin{align}
    \X_{\ell} &= \fl_{1:\ell}(\X) \\
    \M_{\ell} &= \fl_{1:\ell}(\M).
\end{align}
Note that we build layers of a FlatNet in a ``closed-loop fashion'' \citep{ma2022principles}: the directions to flatten the data manifold \(\M\), i.e., the estimator of the tangent space used by each individual layer of \(\fl_{\mathrm{CC}}\), are only determined by repeatedly going back to the data through the corresponding reconstruction map and finding the optimal flattening direction.
\begin{figure}[H]
    \centering
    \subfloat[\centering \(\M\)]{{\includegraphics[width=2.61cm]{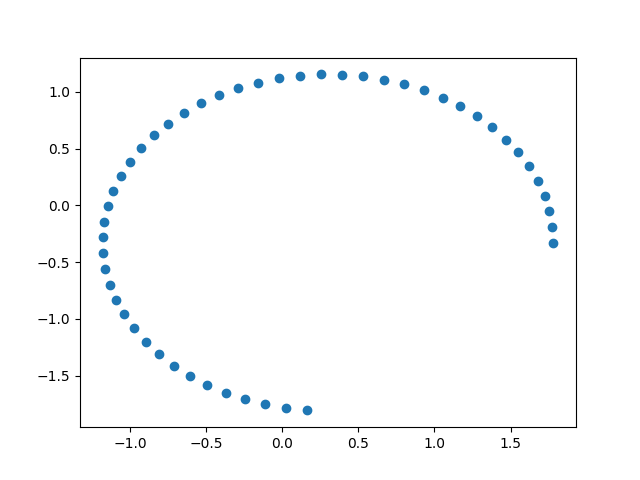} }}%
    \quad
    \subfloat[\centering \(\M_{20}\)]{{\includegraphics[width=2.61cm]{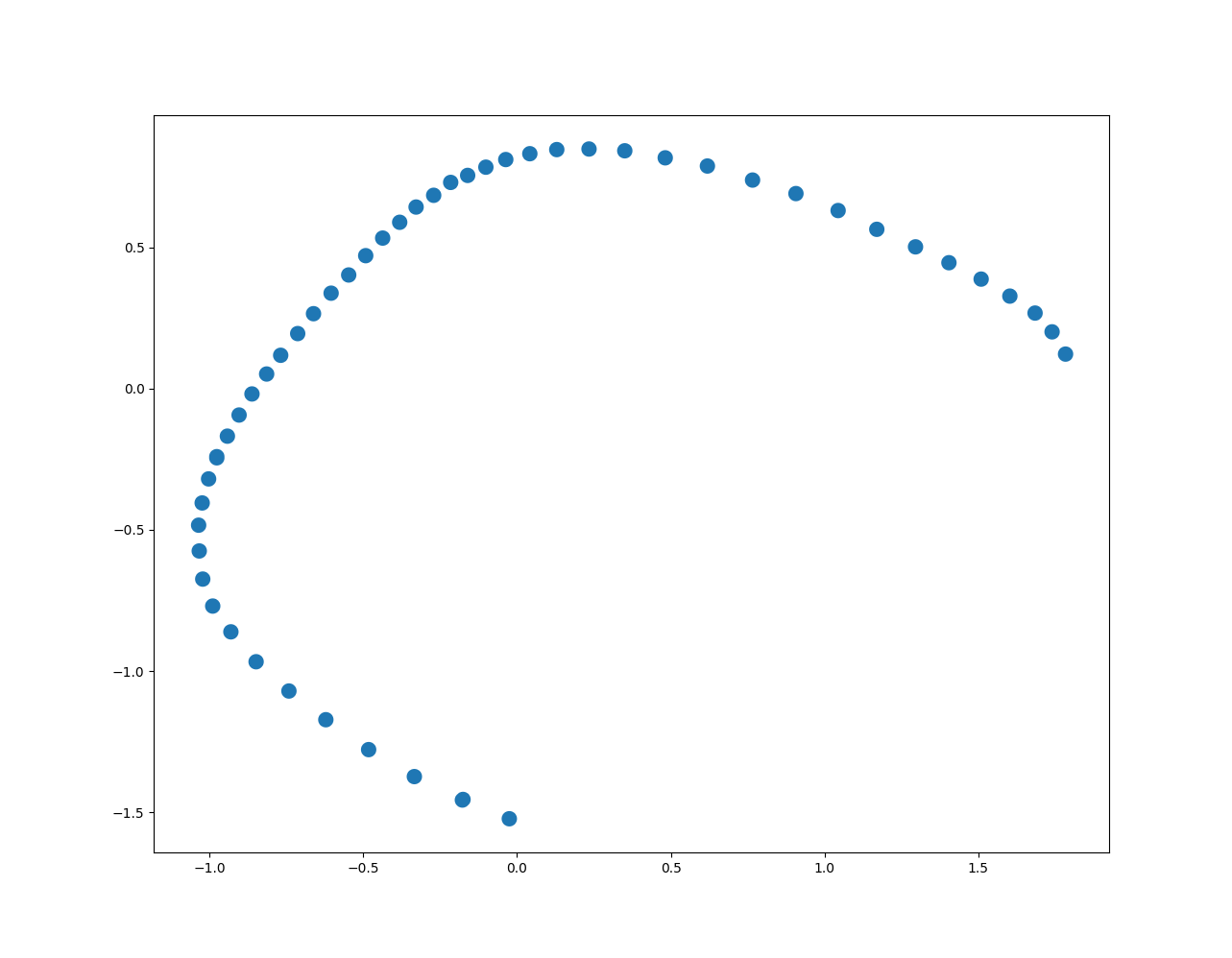} }}%
    \quad
    \subfloat[\centering \(\M_{40}\)]{{\includegraphics[width=2.61cm]{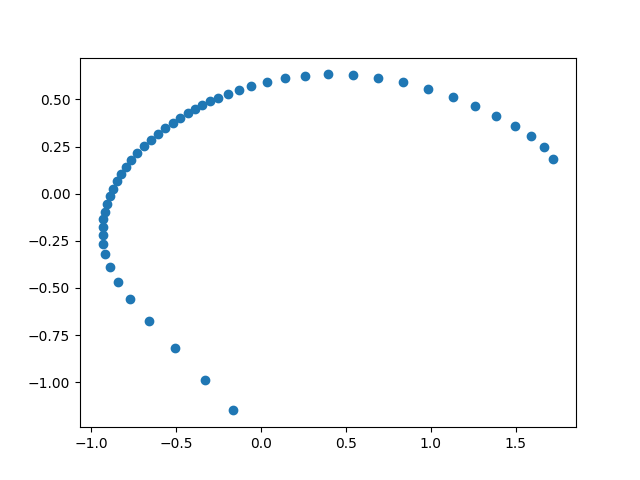} }}%
    \quad
    \subfloat[\centering \(\M_{60}\)]{{\includegraphics[width=2.61cm]{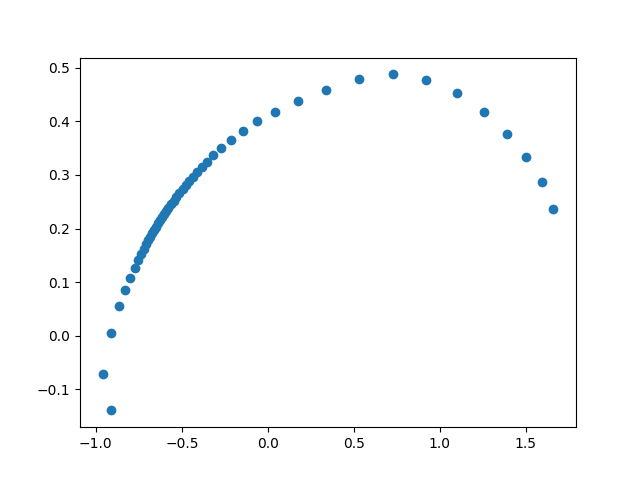} }}%
    \quad
    \subfloat[\centering \(\M_{80}\)]{{\includegraphics[width=2.61cm]{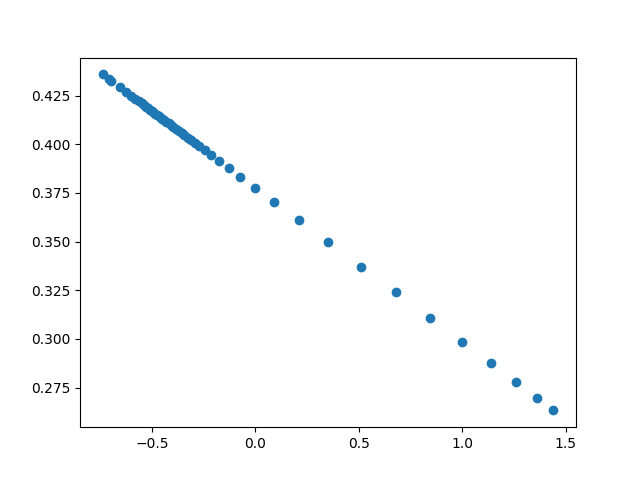} }}%
    
    \caption{Plots of samples from the manifolds \(\M_{k}\) after a given number of layers \(k\) have been constructed. Note that the flattening is not isometric, so some intrinsic distortion is expected. However, this distortion is controlled by the requirement of each flattening to be invertible.}%
    \label{fig:flow_prog}
\end{figure}

\subsection{Convergence analysis}

A complete convergence proof is outside of the scope of this paper, as proving a full theorem for convergence would require utilizing structure from the sampling size and density, noise level, structure of \(\M\) (something akin to low-curvature), and more. We leave a complete theorem to future work.

However, there are still many theoretical ideas for convergence that both illuminate the design choices for our architecture and characterize typical optimization behaviors: namely, what does hitting bad/good local minima look like. To this end, we give theoretical formality to the convergence of our main algorithm. We first detail the proper ``norm'' to use when detailing how close a manifold along the algorithm's process \(\M_{k}\) to converging. We then give an outline for a convergence theorem using the proposed norm, characterizing both why we expect our main algorithm to converge to a good minimum and what happens when a bad local minimum is hit.
\\

\noindent \textbf{A potential function for convergence}. Recall from \Cref{fig:conv_hull_manifold} that we use the difference between the convex hull \(\conv(\M)\) and the manifold \(\M\) itself to heuristically measure how close \(\M\) is to being flattened. Thus, we can first design a scalar-valued \textit{potential function} \(\H\) that measures the difference between \(\conv(\M)\) and \(\M\). For the sake of simplicity, we use the well-studied \textit{Hausdorff metric} \citep{birsan2005one} between \(\M\) and \(\conv(\M)\), which, since \(\M \subseteq \conv(\M)\) and \(\M\) is compact, reduces to the following:
\begin{equation}
    \H(\M) \doteq \max_{x^{\prime} \in \conv(\M)}\min_{x \in \M}\norm{x - x^{\prime}}_{2}.
\end{equation}
Now the following are equivalent:
\begin{itemize}
    \item \(\H(\M) = 0\);
    \item \(\M = \conv(\M)\);
    \item the second fundamental form of \(\M\) is identically \(0\);
\end{itemize}
and the last definition shows that \(\H\) is a metric determining how non-flat a manifold is. However, such a potential function does not reveal anything about the convergence of \Cref{alg:ccnet_construction}, since there exist manifolds such that \(\M \subseteq \T_{x}\M\) at all \(x \in \M\) but still have extrinsic curvature. Such tangent space-contained manifolds \(\M\) are ``fixed points'' of \Cref{alg:ccnet_construction} in the sense that any constructed flattening maps will evaluate to the identity, but \(\M\) is not convex or flat. We then shift our potential function in a manner that more closely resembles \Cref{alg:ccnet_construction}, i.e., measuring convexity of local subsets of a fixed radius \(\eta > 0\):
\begin{align}
    &\Omega_{\eta}(\M) = \int_{\interior[\eta]{\M}} \H(\M \cap B_{\eta}(x))\ \mathrm{d}x, \\
    \text{where} \quad 
    &\interior[\eta]{\M} \doteq \{x \in \M \colon \boundary{\M} \cap B_{\eta}(x) = \emptyset\},
\end{align}
where $\boundary \M$ is the boundary of $\M$. If \(\H(\M) = 0\), then \(\H(\M \cap B_{\eta}(x)) = 0\) at each \(x\) ensures that the tangent space \(\T_{x}\M\) doesn't change locally, and thus \(\T_{x}\M\) cannot change globally, but \(\M\) does not necessarily have to be convex, since the above integral is only taken in the interior of $\M$. We can finally discretize the integral into a finite sum over the dataset \(\X\):
\begin{equation}\label{eq:convergence_loss_main}
    \tilde{\Omega}_{\eta}(\M) = \sum_{i = 1}^{n} \H(\M \cap B_{\eta}(x_{i})).
\end{equation}
The role of \(\eta\) becomes more apparent in the above formulation: if \(\X\) is an \(\eta\)-cover for \(\M\) (see \Cref{def:faithful}), then \(\tilde{\Omega}_{\eta}(\M) = 0\) implies \(\M \subseteq \T_{x}\M\) for all \(x \in \M\), and thus implies a good convergence property for \Cref{alg:ccnet_construction}. As neighborhoods in eq. \eqref{eq:convergence_loss_main} may contain boundary points, some case-work is needed in order to handle behavior at boundary points, where Algorithm \ref{alg:ccnet_construction} may converge while the local loss remains nonzero.
\\

\noindent \textbf{Convergence analysis via modelable radius}. While \(\X_{\ell}\) and \(\M_{\ell}\) are obviously important quantities to track for convergence, we present one more quantity of interest. Let \(r_{Q, \ell, \eps}(x) \geq 0\) be the maximum radius \(r\) around \(x\) such that \(\M_{\ell}\) is modelable by a quadratic around \(x\) of radius \(r\) with at most \(\eps\) error: 
\begin{equation}
    \max_{\substack{y \in \M \\ \norm{y - x}_{2} \leq r}}\max_{\substack{z \in \hat{Q} \\ \norm{z - x}_{2} \leq r}}\norm{y - z}_{2} \leq \eps
\end{equation}
where \(\hat{Q}\) is a local quadratic model given by \Cref{eq:exp_taylor} for some proposed tangent space \(\hat{\T}_{x}\M_{k} \subseteq \R^{D}\) of dimension \(d\) and second fundamental form \(\hat{\sff}_{x} \colon \T_{x}\M \times \T_{x}\M \to \N_{x}\M\). If no such upper bound exists, we say \(r_{Q, \ell, \eps}(x) = \infty\).

In particular, \(r_{Q, \ell, \eps}\) plays a crucial role in studying the convergence of \Cref{alg:ccnet_construction}, as it completely characterizes whether or not we converge to a good or bad fixed point. This idea can be expressed concretely by the following dichotomy:

\begin{enumerate}
    \item Suppose that \(\M_{\ell} \subseteq \T_{x}\M_{\ell}\) for some \(\ell \in \mathbb{N}\), thus halting \Cref{alg:ccnet_construction} at a good fixed point. Then \(r_{Q, \ell, \eps}(\fl_{\ell}(x_{i})) = \infty\) for all \(i \in [N]\) by setting \(\hat{\sff}_{x_{i}} = 0\), and thus the sequence \((r_{Q, \ell, \eps}(\fl_{\ell}(x_{i})))_{\ell \in \mathbb{N}}\) has no upper bound for any \(i \in [N]\). 
    \item Suppose on the contrary there exists some point \(x_{i} \in \X\) such that the sequence \((r_{Q, \ell, \eps}(\fl_{\ell}(x_{i})))_{\ell \in \mathbb{N}}\) has an upper bound. Thus, \(\M_{\ell} \nsubseteq \T_{x_{j}}\M_{\ell}\) for all \(\ell \in [L]\), and \(x_{j} \in \X\). Then if \(\M_{t}\) reaches a fixed point at some timestep \(t\), there must exist some \(x_{j}\) such that \(r_{Q, t, \eps}(\fl_{t}(x_{j})) < \eta\), since otherwise there exists an atlas of \(\M_{t}\) where each chart/neighborhood is completely flat, contradicting the statement that  \(\M_{\ell} \nsubseteq \T_{x}\M_{\ell}\).
\end{enumerate}
From the above dichotomy, we see the role of \(r_{Q, \ell, \eps}\): this quantity will diverge if \Cref{alg:ccnet_construction} converges to a good local minimum, and converge below the ``threshold level'' \(\eta\) if a bad local minima is reached. However, the above analysis still \textit{assumes} that \Cref{alg:ccnet_construction} converges. Extending beyond this assumption is feasible, but requires detailed numerical analysis to account for the specific problem structure. The following is a dichotomy on the trajectory \((\M_{\ell})_{\ell \in \mathbb{N}}\) that does not require assumptions on convergence:
\begin{enumerate}
    \item Suppose that for some \(\ell \in \Z\) and \(x_{i} \in \X\), we have \(r_{Q, t, \eps}(\fl_{t}(x_{i})) < \eta\) for all \(t > \ell\). Then \Cref{alg:ccnet_construction} either does not converge or converges to a local minima where \(\M_{\ell} \nsubseteq \T_{x}\M_{\ell}\).

    \item Suppose that for all \(\ell \in \Z\) and \(x_{i} \in \X\), there exists \(t > \ell\) such that \(r_{Q, t, \eps}(\fl_{t}(x_{i})) \geq \eta\). Then \Cref{alg:ccnet_construction} can always make some progress on every \(x_{i}\) at iteration \(t\).
\end{enumerate}

Formalizing the above ``progress'' into positive progress via \(\lim_{\ell \to \infty}\tilde{\Omega}(\M_{\ell}) = 0\) is the main missing ingredient for a full convergence theorem. While it is clear that an iteration of \Cref{alg:ccnet_construction} at point \(x_{i}\) decreases the component of \(\tilde{\Omega}_{\eta}(\M_{\ell})\) coming from \(x_{i}\) in \Cref{eq:convergence_loss_main}, namely \(\H(\M_{\ell} \cap B_{\eta}(\fl_{\ell}(x_{i})))\), it's not clear that the amount this iteration inadvertently increases \(\H(\M_{\ell} \cap B_{\eta}(\fl_{\ell}(x_{j})))\) for other \(x_{j} \in \X\) would not lead to an overall increase in \(\tilde{\Omega}_{\eta}(\M_{\ell})\). Proving this relation requires careful analytical work on the global structure of \(\M\) and \(\X\), which we leave for future work.

\section{Experiments}

To evaluate our method, we test our Flattening Networks (FlatNet) on synthetic low-dimensional manifold data, randomly generated high-dimensional manifolds generated by Gaussian processes \citep{lahiri2016random,lawrence2005probabilistic}, and popular real-world imagery datasets. The following are experimental observations that set our method apart from what is currently available:

\begin{enumerate}
    \item \textit{There is no need to select the intrinsic dimension $d$}, as this is automatically learned from the data using \eqref{eq:auto_dim}. This is in contrast to neural network-based representation learners, e.g. variational autoencoders \citep{kingma2013auto}, which need a feature dimension selection beforehand, or many manifold learning methods that need an intrinsic dimension specified. This is important for practical data manifolds, since the intrinsic dimension of the underlying data manifold is hardly ever known.
    \item \textit{There is no need to select a stopping time}, as our convexifying geometric flow converges once the learned representation is already flat.
    \item \textit{Learned manifolds from noisy data are smooth}, as depicted in figures in Section \ref{sec:lowdim}. This is likewise important for practical data manifolds, as samples have off-manifold noise. 
\end{enumerate}

Code for both FlatNet itself and the below experiments is publicly available\footnote{\url{https://github.com/michael-psenka/manifold-linearization}}.

\subsection{Low-dimensional manifold data}\label{sec:lowdim}

We test on three types of low-dimensional manifolds: 
\begin{itemize}
    \item Data sampled from a (noisy) sine wave,
    \item Data sampled from a Gaussian process manifold \citep{lahiri2016random,lawrence2005probabilistic} with intrinsic dimension \(d = 1\) in Euclidean space of extrinsic dimension \(D = 2\),
    \item and data sampled from a Gaussian process manifold with \(d = 2\) and \(D = 3\).
\end{itemize}
Since the Gaussian process manifold data generation process is not clear at first glance, we describe it here. We first lay out \(N\) vectors of intrinsic coordinates in a matrix \(C \in \R^{d \times N}\); in our experiments, these are uniformly generated. Then for each \(i \in \{1, \dots, D\}\), we set up a correlation matrix \(\Sigma_{i} \in \R^{N \times N}\) whose coordinates \((\Sigma_{i})_{pq} = \frac{L_{i}}{D}e^{-\rho_{pq}/2}\), where \(L_{i} > 0\) is a hyperparameter (set to be \(1\) in our quantitative experiments) and \(\rho_{pq} = \|c_{p} - c_{q}\|_{2}^{2}\), where \(c_{p}\) and \(c_{q}\) are the columns of \(C\). Then the \(i^{\mathrm{th}}\) row of \(X\) (i.e. the \(i^{\mathrm{th}}\) coordinate of all datapoints, a vector in \(\R^{N}\)) is sampled (independently of all other rows) from the Gaussian \(\mathcal{N}(0_{N}, \Sigma_{i})\). Repeating this for all \(i \in \{1, \ldots, D\}\) obtains the full data matrix \(X = \mat{x_{1} & \cdots & x_{N}} \in \R^{D \times N}\).

We measure the performance of our algorithm in a few ways.
\begin{itemize}
    \item Given data \(X \in \R^{D \times N}\), we can plot the features \(Z \doteq \fl_{\CC}(X) \in \R^{D \times N}\) and reconstructions \(\hat{X} \doteq \re_{\CC}(Z) \in \R^{D \times N}\). Then the features \(Z\) should form an affine subspace and the reconstructions \(\hat{X}\) should be close to \(X\) sample-wise. In order to understand how our method generalizes, we may linearly interpolate between pairs of features in \(Z\) to get a large new matrix \(Z_{\mathrm{test}} \in \R^{D \times M}\) for \(M \gg N\), then reconstruct that as \(\hat{X}_{\mathrm{test}} \coloneqq \re_{\CC}(Z_{\mathrm{test}}) \in \R^{D \times M}\). Ultimately, we plot \(X\), \(Z\), and \(\hat{X}_{\mathrm{test}}\).
    \item If we know an explicit formula for the ground truth curve that generates the data, we may compute the tangent space at each point. We may thus evaluate how close our tangent space estimator used in the algorithm is to the ground truth tangent space. We do this by plotting both the estimated tangent space and the true tangent space; they should greatly overlap.
\end{itemize}

In \Cref{fig:sine_example}, we demonstrate the performance of our algorithm FlatNet on data generated from the graph of a (noisy) sine wave. More precisely, set \(N = 50\), \(D = 2\), \(d = 1\), and \(M = 5000\). In the left figure we test linearization, reconstruction, and generalization, while in the right figure we demonstrate tangent space estimation. 

\begin{figure}%
    \centering
    \subfloat[\centering Converged result of our algorithm on noisy sine wave data]{{\includegraphics[width=0.4\textwidth]{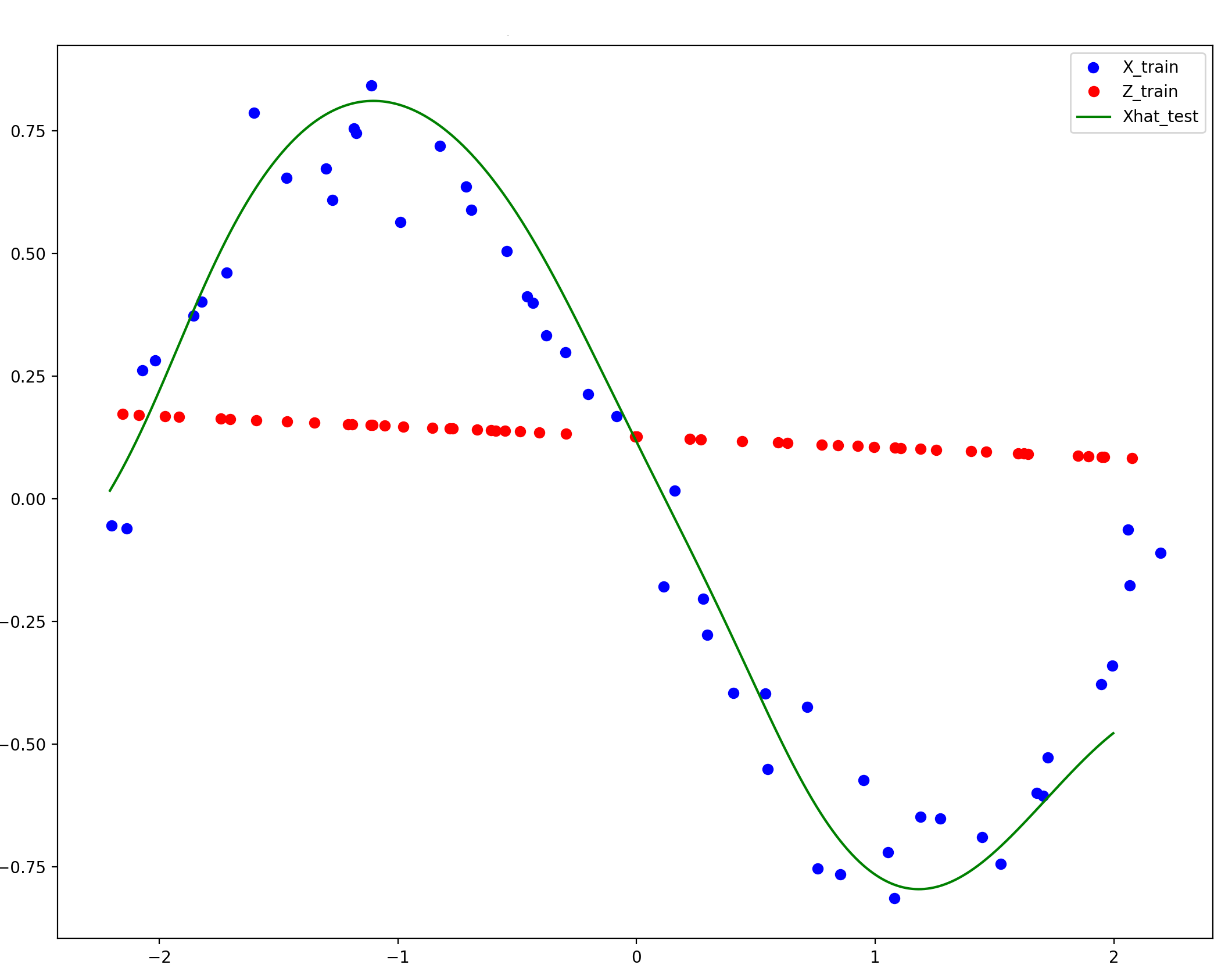} \label{fig:interpolate}}}%
    \qquad
    \subfloat[\centering Tangent space estimation performance compared to local PCA on sine wave data]{{\includegraphics[width=0.4\textwidth]{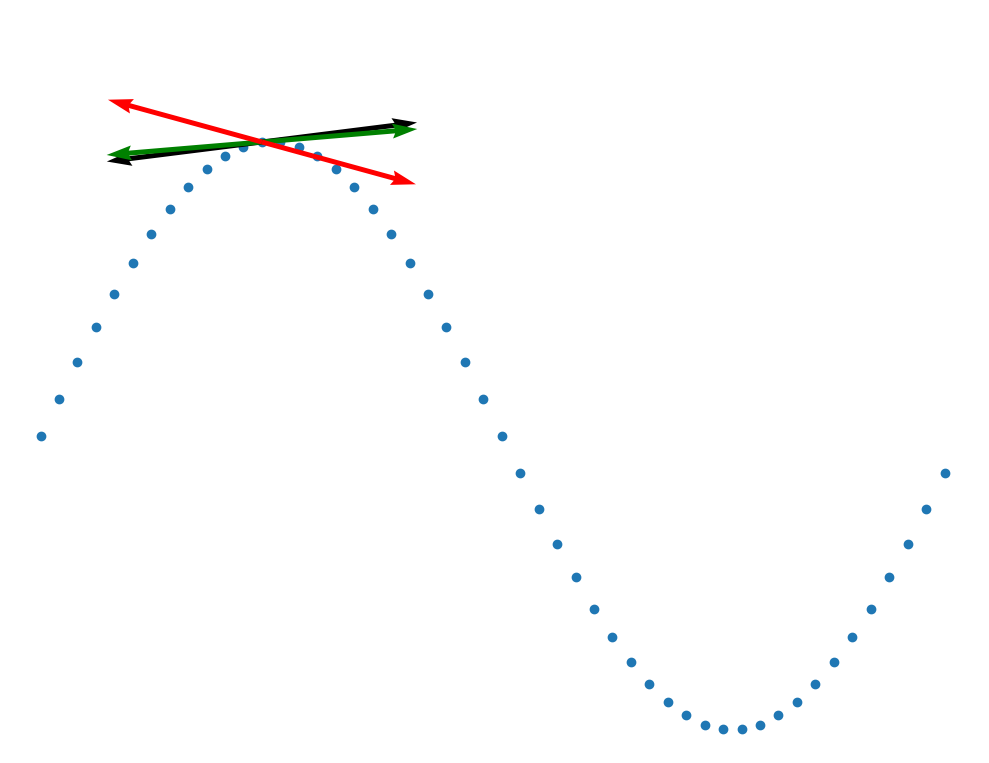} }}%
    \caption{Results of FlatNet on data sampled uniformly from the graph of a sine wave embedded in \(\R^{2}\). In the left figure: the blue points are the training data \(X\), the red points are the flattened data \(Z\), and the green line is the reconstruction of interpolated points \(\hat{X}_{\mathrm{test}}\). On the right, each double sided arrow is a tangent space: black is ground truth, green is our estimation method, and red is local PCA.}%
    \label{fig:sine_example}%
\end{figure}

In \Cref{fig:example_gp_D2_d1}, we demonstrate the performance of our algorithm FlatNet on data generated on a manifold constructed via Gaussian processes \citep{lahiri2016random,lawrence2005probabilistic}. Note that for such manifolds, we do not know ground truth tangent spaces, so we constrain our experiments to demonstrating reconstruction and generalization performance. We set \(N = 50\), \(D = 2\), \(d = 1\), and \(M = 5000\) like before. We compare the reconstruction and generalization performance of FlatNet against three types of variational autoencoders (VAEs): (1) the vanilla VAE \citep{kingma2013auto}, (2) \(\beta\)-VAE (\(\beta = 4\)) \citep{higgins2017beta}, and (3) FactorVAE (\(\gamma  = 30\)) \citep{kim2018disentangling}. Each VAE encoder and decoder is a 5 layer multi-layer perceptron, with dimension \(100\) at each layer, and latent dimension \(d = 1\). Each VAE is trained for \(100\) epochs with Adam optimizer and \(10^{-3}\) learning rate. The experiment demonstrates that FlatNet is convincingly better at learning to reconstruct low-dimensional structures, and has better generalization performance, than the VAEs we compare against.

\begin{figure}%
    \centering
    \subfloat[\centering FlatNet.]{{\includegraphics[width=0.45\textwidth]{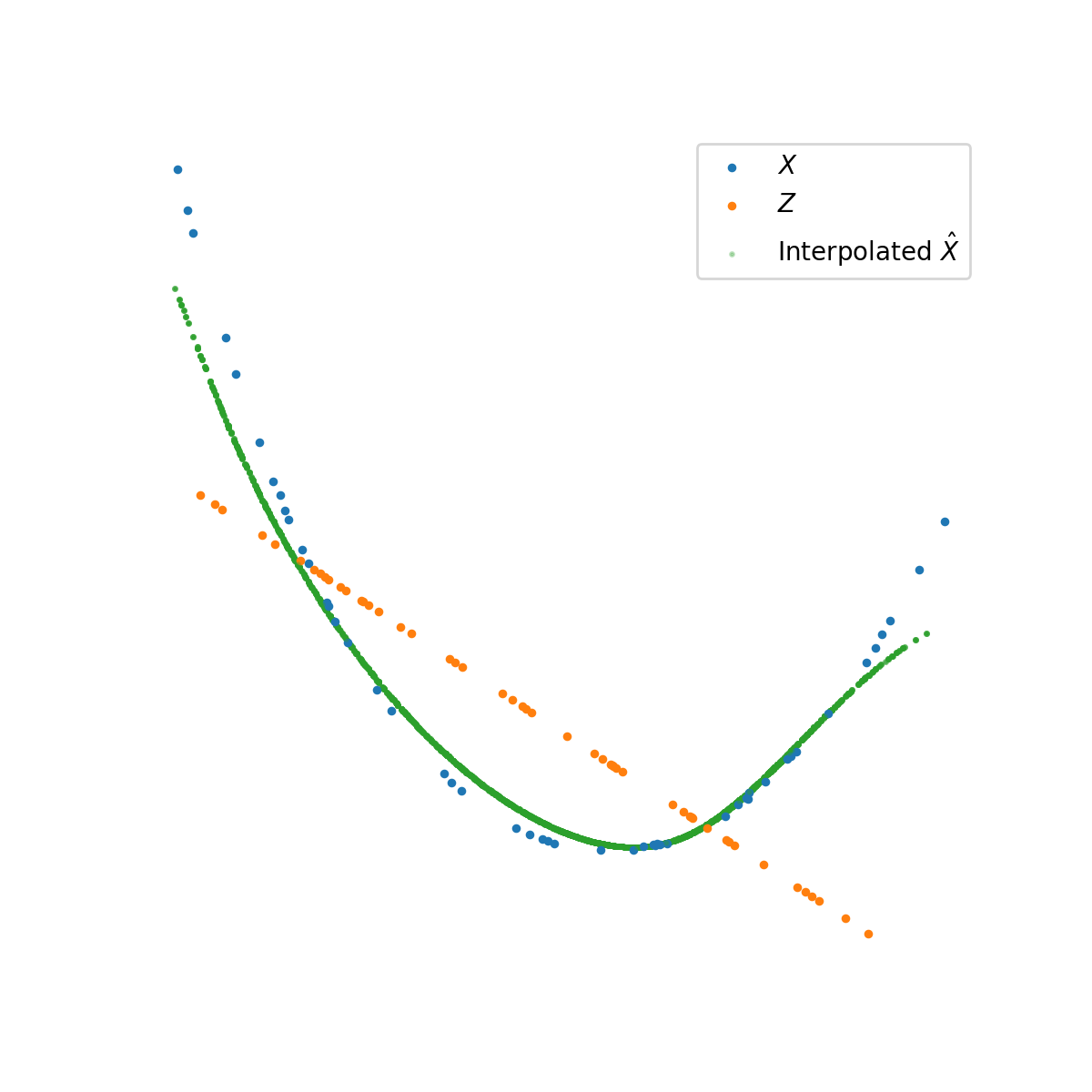} \label{fig:example_gp_D2_d1_cc}}}%
    \qquad
    \subfloat[\centering VAE.]{{\includegraphics[width=0.45\textwidth]{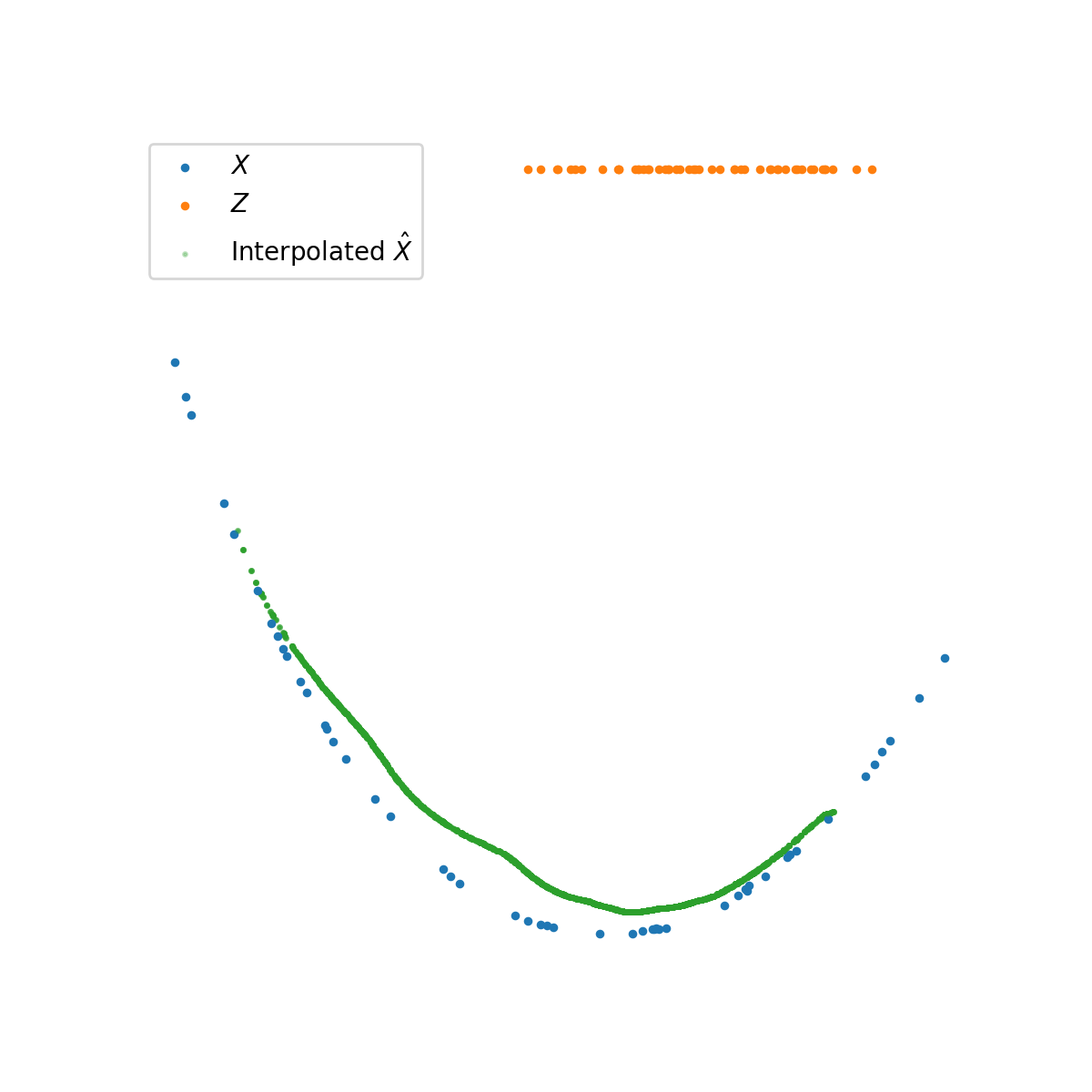} \label{fig:example_gp_D2_d1_vae} }}%
    \qquad
    \subfloat[\centering \(\beta\)-VAE.]{{\includegraphics[width=0.45\textwidth]{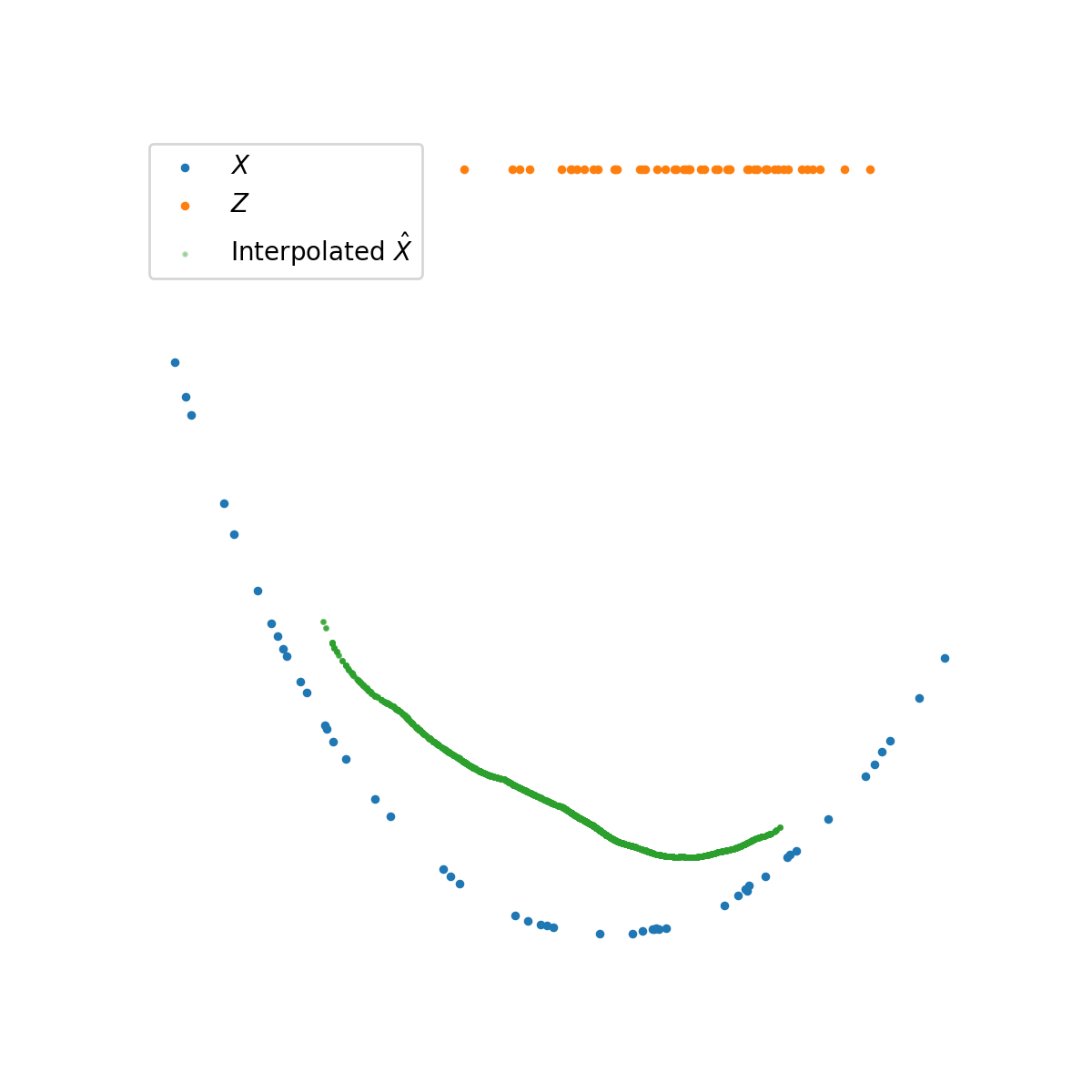} \label{fig:example_gp_D2_d1_bvae} }}%
    \qquad
    \subfloat[\centering FactorVAE.]{{\includegraphics[width=0.45\textwidth]{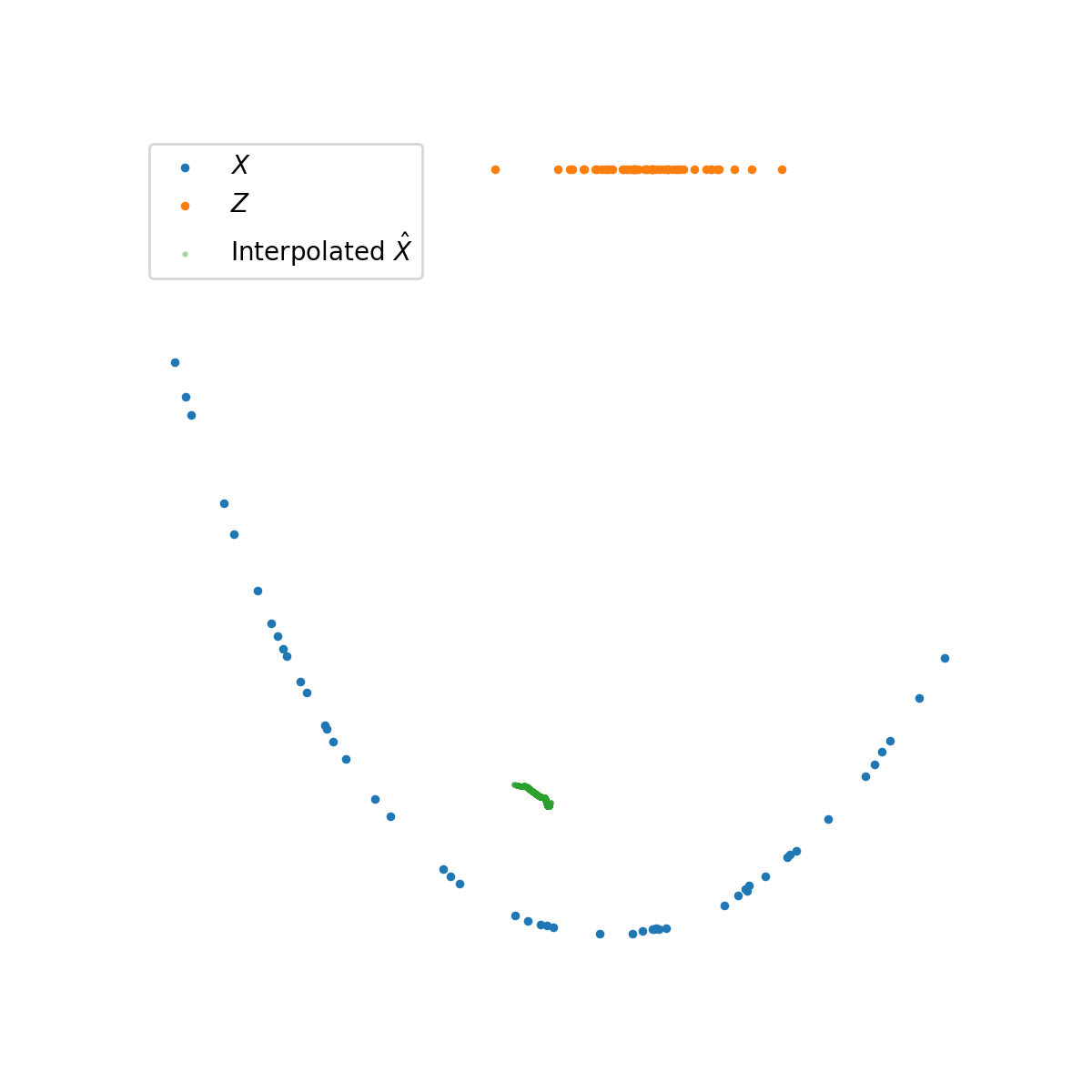} \label{fig:example_gp_D2_d1_bvae} \label{fig:example_gp_D2_d1_fvae} }}%
    \caption{Results of FlatNet contrasted with VAE, \(\beta\)-VAE, and FactorVAE. Data is \(N = 50\) points sampled from a random \(d = 1\) dimensional Gaussian process manifold in \(\R^{D} = \R^{2}\). Note that FactorVAE completely degenerates on the low-dimensional data structure, while the other two VAEs clearly do not perform as well as FlatNet.}%
    \label{fig:example_gp_D2_d1}%
\end{figure}

In \Cref{fig:example_gp_D3_d2}, we demonstrate the performance of FlatNet on data generated on a manifold constructed via Gaussian processes as before. This time, we set \(N = 50\), \(D = 3\), \(d = 2\), and \(M = 5000\). We do the same comparisons to VAEs as before; this time the VAEs have latent dimension \(d = 2\). The experiment again demonstrates that FlatNet is better at learning to reconstruct low-dimensional structures and has better generalization performance.

\begin{figure}%
    \centering
    \subfloat[\centering FlatNet.]{{\includegraphics[width=\textwidth]{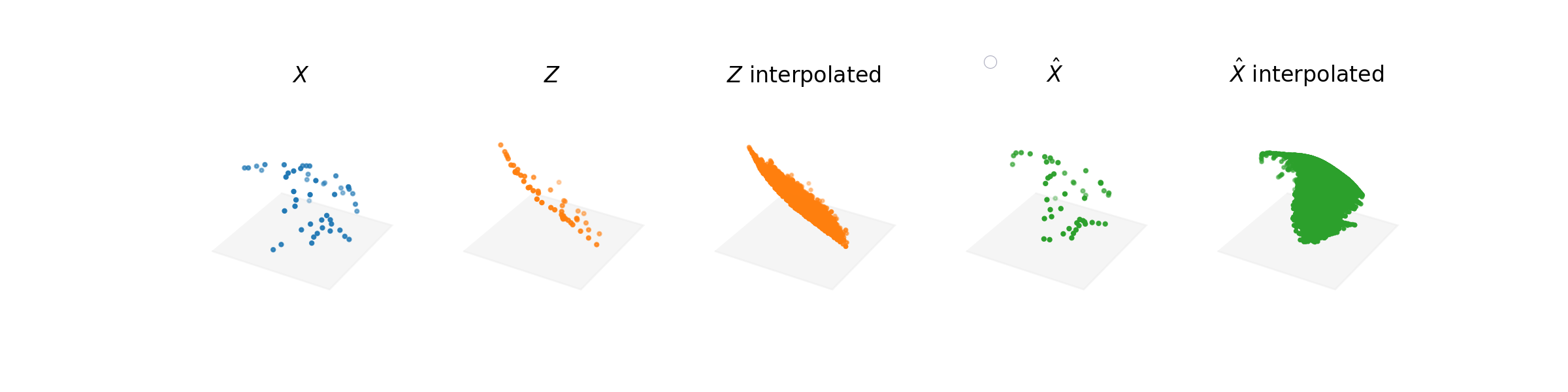} \label{fig:example_gp_D3_d2_cc}}}%
    \qquad
    \subfloat[\centering VAE.]{{\includegraphics[width=\textwidth]{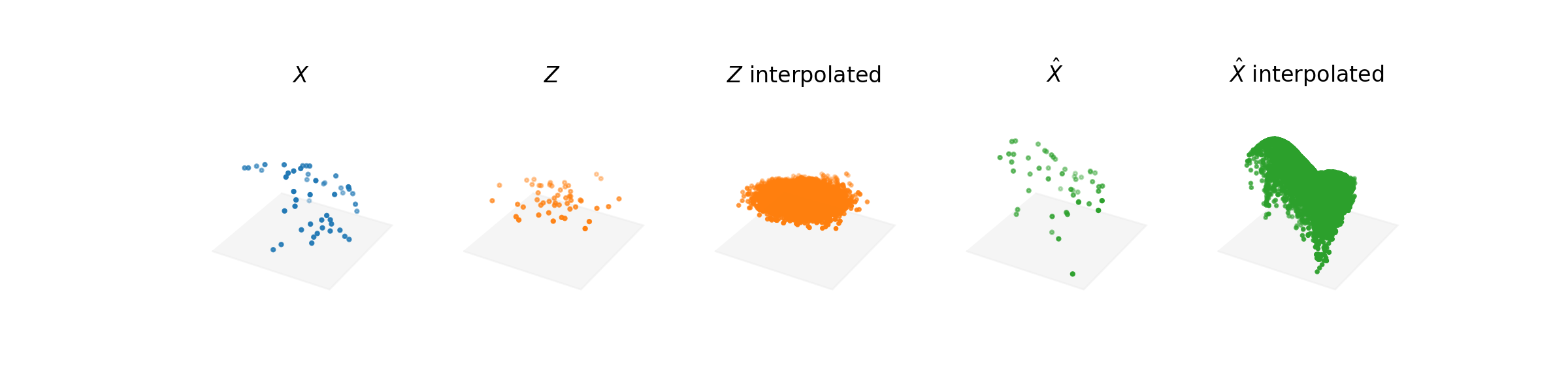} \label{fig:example_gp_D3_d2_vae}}}%
    \qquad
    \subfloat[\centering \(\beta\)-VAE.]{{\includegraphics[width=\textwidth]{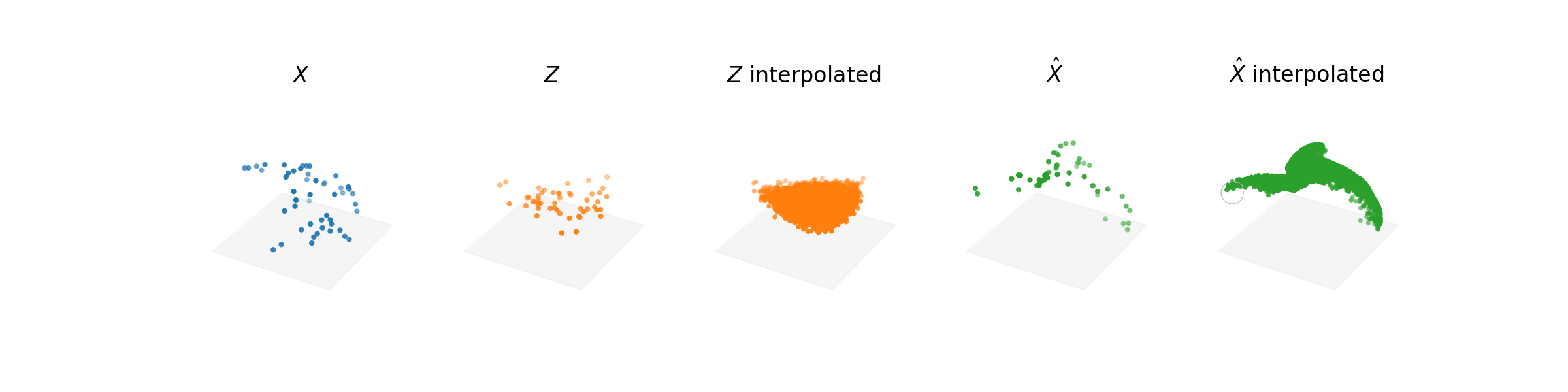} \label{fig:example_gp_D3_d2_bvae}}}%
    \qquad
    \subfloat[\centering FactorVAE.]{{\includegraphics[width=\textwidth]{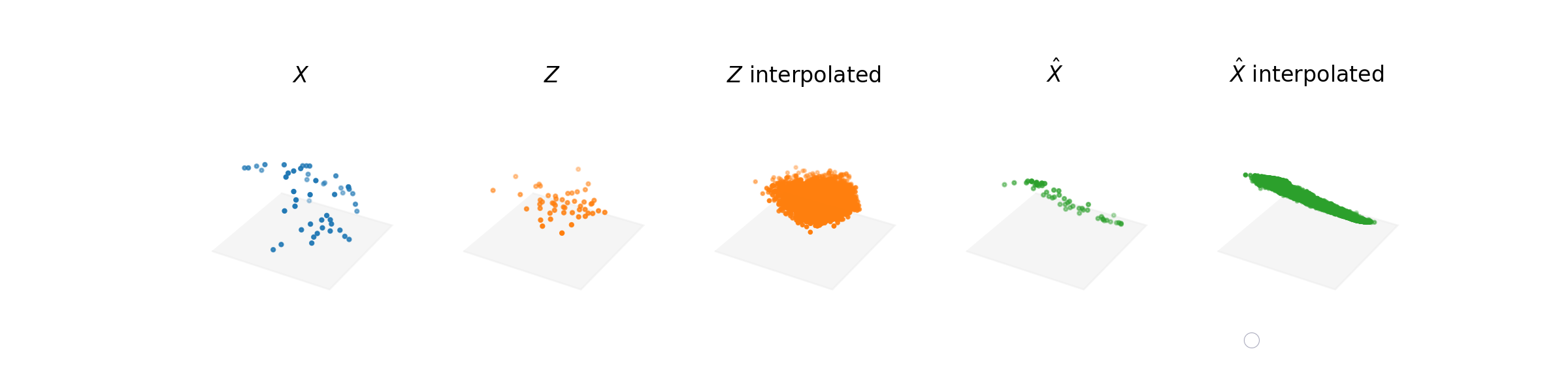} \label{fig:example_gp_D3_d2_fvae}}}%
    \caption{Results of FlatNet contrasted with VAE, \(\beta\)-VAE, and FactorVAE. Data is \(N = 50\) points sampled from a random \(d = 2\)-dimensional Gaussian process manifold in \(\R^{D} = \R^{3}\). Again, FlatNet learns to reconstruct and generalize much better than the VAEs we test againt. Also, the feature space for FlatNet is really a \(2\)-dimensional affine subspace, as desired; this would be clear after rotating the figure. Tools to perform interactive visualization are present in the attached code.}%
    \label{fig:example_gp_D3_d2}%
\end{figure}

\subsection{High-dimensional manifold data}

In this section we test our algorithm on random Gaussian process manifolds \citep{lahiri2016random,lawrence2005probabilistic} of varying intrinsic dimension \(d\) in extrinsic dimension \(D = 100\) Euclidean space. We typically use \(N = 1000\).

There are two quantitative ways we measure the performance of FlatNet in high dimensions: reconstruction quality and ability to estimate the dimension of the manifold. Note that the latter is essentially trivial in the low dimensional regime, yet in the high-dimensional regime it can be thought of as a rough certificate for the accuracy of each local flattening iteration.

In \Cref{fig:gp_manifold_reconstruction}, we empirically evaluate the reconstruction error on our finite samples \(\{x_{1}, \dots, x_{N}\}\) as 
\begin{equation}
    \text{reconstruction error} \approx \frac{1}{N}\sum_{i = 1}^{N}\|x_{i} - \hat{x}_{i}\|_{2}^{2}, \qquad \text{where} \qquad \hat{x}_{i} \coloneqq \re_{\CC}(\fl_{\CC}(x_{i})).
\end{equation}
We estimate the error across \(3\) trials of the experiment, and plot the mean curve and standard deviation, for each \(d \in \{5, 10, 15, 20\}\). We compare against FactorVAE, which has the same configuration as previously discussed except for the latent dimension which is set to \(d\).\footnote{We also compared against the vanilla VAE and \(\beta\)-VAE, but those had nearly identical performance to FactorVAE and so are omitted from the chart for visual clarity.} Overall, the experiment demonstrates again that FlatNet is convincingly better at learning to reconstruct low-dimensional structures than the VAEs we compare against.

\begin{figure}
    \centering
    \includegraphics[width=0.6\textwidth]{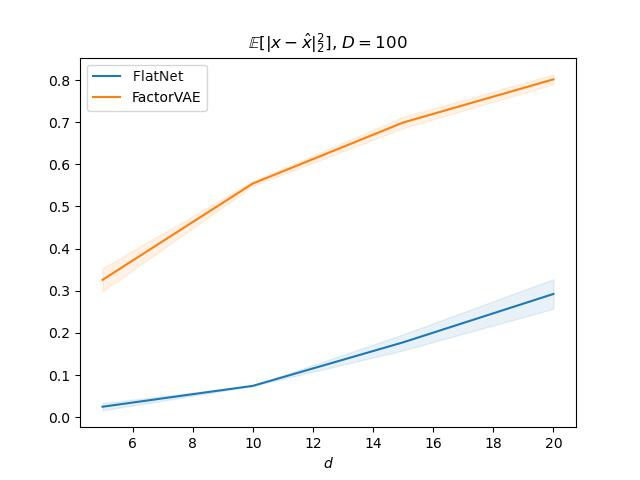}
    \caption{Reconstruction error of FlatNet and FactorVAE. Data is \(N = 1000\) points sampled from a random \(d \in \{5, 10, 15, 20\}\) dimensional Gaussian process manifold in \(\R^{D} = \R^{100}\). Error margins are \(\pm 1\) standard deviation across \(3\) trials.}
    \label{fig:gp_manifold_reconstruction}
\end{figure}

In Figure~\ref{fig:dim_estimation}, we empirically evaluate the dimension estimation capability. Given a trained FlatNet encoder and decoder, we estimate the global intrinsic dimension of \(\M\) using the following procedure. Each iteration \(\ell\) of the training process estimated an intrinsic local dimension \(\hd_{k}\) for some neighborhood of the manifold; to estimate the global dimension we simply compute \(\hd \doteq \operatorname{mode}(\hd_{\ell} \colon \ell \in [L])\), i.e., the most commonly estimated local dimension. We estimate the intrinsic dimension across \(3\) trials of the experiment, and plot the mean curve and standard deviation, for each \(d \in \{5, 10, 15, 20\}\). We compare with the popular intrinsic dimension estimation algorithms MLE \citep{levina2004maximum} and TwoNN \citep{facco2017estimating}. Overall, the experiment demonstrates that FlatNet is competitive with the state of the art at dimension estimation, even when not explicitly designed for this task.

\begin{figure}
    \centering
    \includegraphics[width=0.7\textwidth]{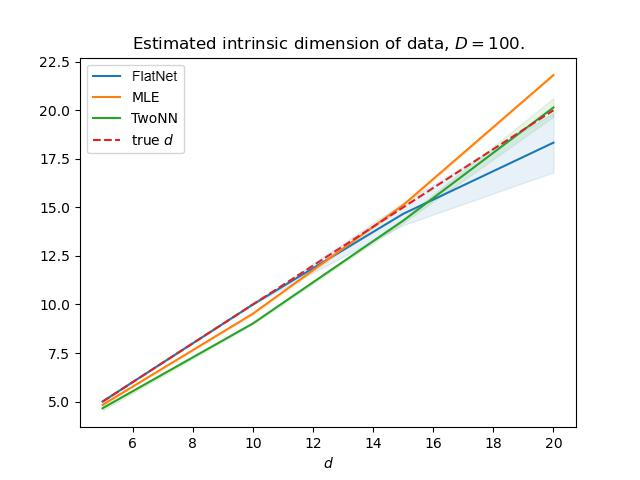}
    \caption{Global intrinsic dimension estimation of FlatNet, MLE, and TwoNN. Data is \(N = 1000\) points sampled from a random \(d \in \{5, 10, 15, 20\}\) dimensional Gaussian process manifold in \(\R^{D} = \R^{100}\). Error margins are \(\pm 1\) standard deviation across \(3\) trials.}
    \label{fig:dim_estimation}
\end{figure}

\subsection{Real-world imagery data}

Real-world imagery data such as MNIST tend to lie on high-curvature or non-differentiable pathological manifolds \citep{wakin2005multiscale}. In order to make such data tractable for the use of FlatNet, we use the (vectorized) Fourier transform of the image, say \(\F\{x_{i}\}\), instead of the image \(x_{i}\) itself, as the input to FlatNet, where it becomes easier to smooth out domain transformations (e.g. convolution with a Gaussian kernel).

We begin with a constructed manifold from MNIST to visualize how well FlatNet recovers intrinsic features from 2D vision data. To this end, we take a single image from MNIST, and both rotate it by various angles and translate it along the $y$-axis by various amounts. The result is a 2-dimensional manifold embedded in pixel space (here, \(D = 32\times 32 = 1024\)). In \Cref{fig:mnist_synthetic}, we present the results of running FlatNet on the constructed manifold, with both the features it recovers and reconstruction accuracy.

\begin{figure}[]%
    \centering
    \subfloat[\centering Original rotated/translated images (top) compared to reconstruction \(\re_{\CC}(\fl_{\CC}(x))\) from FlatNet (bottom)]{{\includegraphics[width=0.5\textwidth]{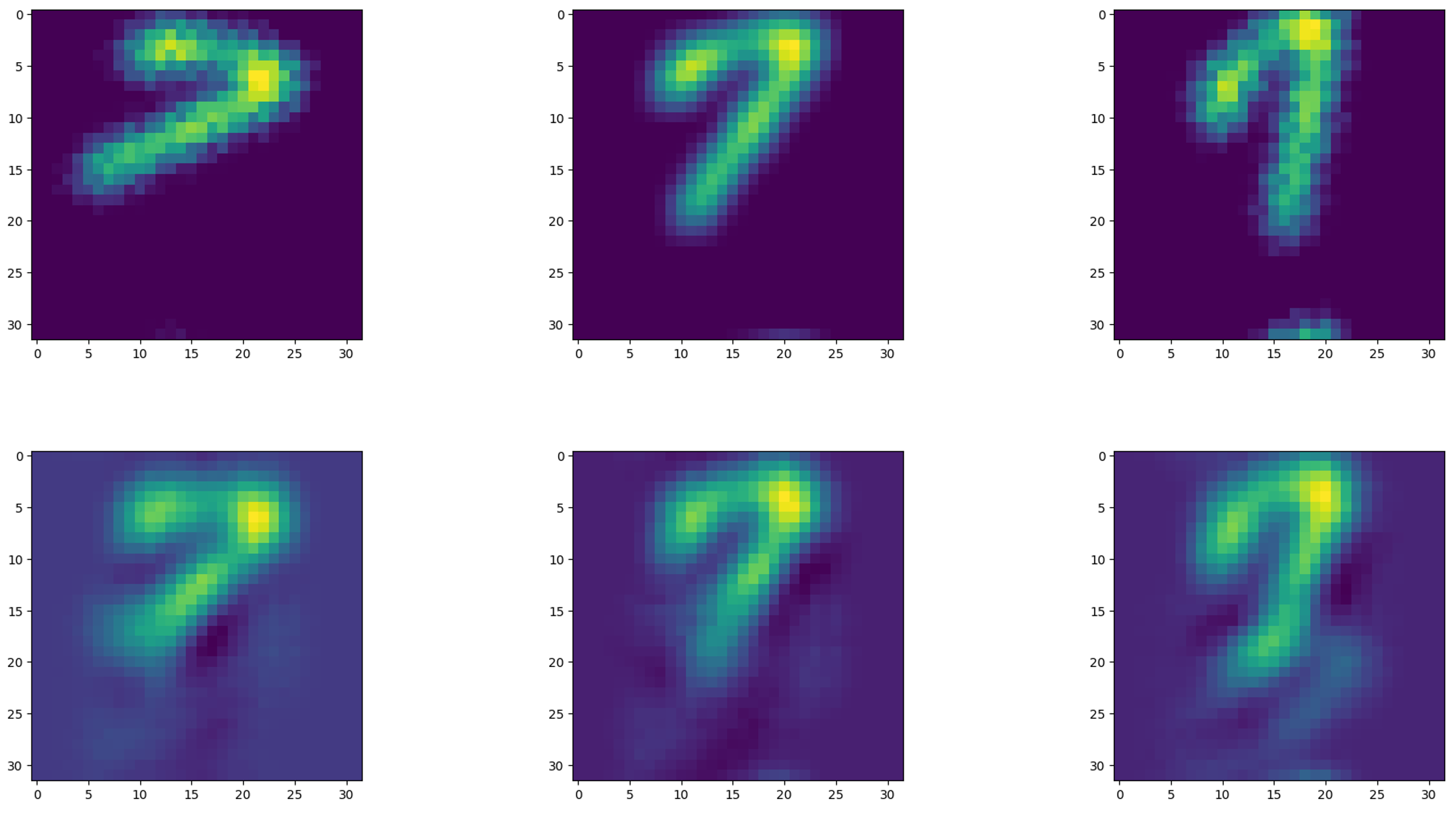} }}%
    \qquad
    \subfloat[\centering Intrinsic coordinates ($y$, $\theta$)]{{\includegraphics[width=0.3\textwidth]{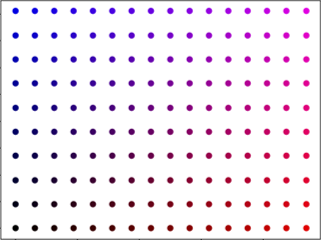} }}%
    \subfloat[\centering Learned representation from FlatNet]{{\includegraphics[width=0.3\textwidth]{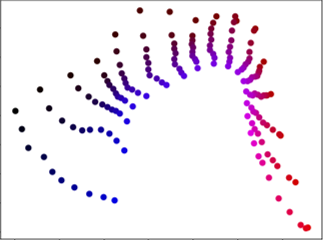} }}%
    \subfloat[\centering Representation from PCA]{{\includegraphics[width=0.3\textwidth]{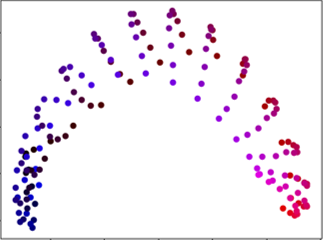} }}%
    \caption{Results of FlatNet on the following manifold: a base point is selected from MNIST (here a ``7''), and various $y$-transations between [-7, 7] pixels and rotations between [-30, 30] degrees. The representation from FlatNet was run until convergence: it determined the dataset was of dimension 2 automatically. Note that while the features look distorted compared to the intrinsic features, FlatNet learns a representation that is relatively smoothly deformable into the canonical coordinates. However, the PCA representation obtained from taking the top two principal components (which requires knowledge of $d=2$ a-priori), yields a ``tangled'' representation: mapping the PCA representation back into the intrinsic coordinates would require a more complex function.}
    \label{fig:mnist_synthetic}%
\end{figure}

We now move to the original MNIST dataset. In Figure~\ref{fig:mnist_recon}, we demonstrate the empirical reconstruction performance of FlatNet on MNIST data. Notice that the reconstruction is not pixel-wise perfectly accurate; the reconstructed samples appear closer to an archetype of the digit than the original image. This example hints that the encoder function \(\fl_{\CC}\) naturally compresses semantically non-meaningful aspects of the input image, which is a useful property for encoders.

\begin{figure}
    \centering
    \includegraphics[width=0.9\textwidth]{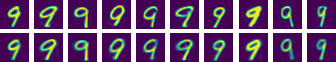} 
    \caption{Original \(x_{i}\) (top) vs. reconstruction \(\hat{x}_{i} \doteq \F^{-1}\{\re_{\CC}(\fl_{\CC}(\F\{x_{i}\}))\}\) (bottom), using FlatNet flattening \(\fl_{\CC}\) and reconstruction \(\re_{\CC}\) on the MNIST dataset.}
    \label{fig:mnist_recon}
\end{figure}

In \Cref{fig:mnist_interpolation}, we demonstrate that the feature space of FlatNet is linear in that it corresponds to an affine subspace. In particular, we demonstrate that linear interpolation in feature space corresponds to semantic interpolation in image space.

\begin{figure}
    \centering
    \includegraphics[width=0.9\textwidth]{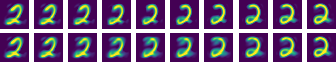}
    \caption{Comparison of interpolation: FlatNet (top), i.e., \(\hat{x} \doteq \F^{-1}\{\re_{\CC}(\theta \fl_{\CC}(\F\{x_{1}\}) + (1 - \theta)\fl_{\CC}(\F\{x_{2}\}))\}\), versus linear interpolation (bottom), i.e., \(\hat{x} \doteq \theta x_{1} + (1 - \theta)x_{2}\).}
    \label{fig:mnist_interpolation}
\end{figure}

In \Cref{fig:mnist_generation}, we demonstrate that we can sample from the linear (structured) feature space to generate semantically meaningful data in image space.

\begin{figure}[H]
    \centering
    \includegraphics[width=0.9\textwidth]{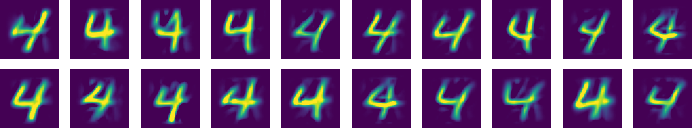}
    \caption{Sampling using FlatNet: \(\hat{x} = \F^{-1}\{\re_{\CC}(z)\}\) where \(z\) is a Gaussian random variable supported on the affine subspace \(\fl_{\CC}(\F\{\M\})\).}
    \label{fig:mnist_generation}
\end{figure}

\subsection{Low-dimensional counterexamples}

We present two low-dimensional counterexamples to show potential limitations of the method, as well as what the algorithm looks like when it hits a failure mode. We present two illustrative results.

\noindent \textbf{Closed circle}. The first is data sampled from a closed curve, a circle, with noise. This represents a manifold that is not even theoretically flattenable. While this represents a theoretical boundary for FlatNet, note that any continuous autoencoding pair, regardless of the model, likewise cannot flatten these manifolds. See \Cref{fig:circle} for a visualization on how FlatNet behaves on such manifolds: it will start to form flat patches until it cusps and is not flattenable or modelable anymore.

\noindent \textbf{Unaugmented Swiss roll}. The second is data sampled from the ``Swiss roll'' dataset. While this manifold is theoretically flattenable, it is quite ``crammed'' in its ambient space: not only is it a hypersurface (\(d = D - 1\)), its self-inward curling makes compression-based flattening methods like FlatNet fold the manifold into itself before. If FlatNet is run on the pure Swiss roll, it runs into the same problems as the closed circle: it quickly cannot make any invertible flattening progress, and thus halts without changing the manifold much at all.

However, if a simple nonlinear feature is appended (\((x_{1}, x_{2}, x_{3}) \mapsto (x_{1}, x_{2}, x_{3}, x_{1}^{2} + x_{3}^{2})\)), then FlatNet is able to both efficiently flatten the manifold and give an approximate reconstruction. While the approximate reconstruction has noticeable differences from the original manifold\footnote{As we are also forcing the reconstruction to fit the artificially constructed feature $x_4 = x+1^2 + x_3^2$, we can expect a decrease in the reconstruction quality as a cost for easier flattening.}, the learned features correspond closely to the intrinsic coordinates that generated the Swiss roll dataset. See \Cref{fig:swiss_roll} for plots of the results.

\begin{figure}[]%
    \centering
    \subfloat[\centering \(X\)]{{\includegraphics[width=0.2\textwidth]{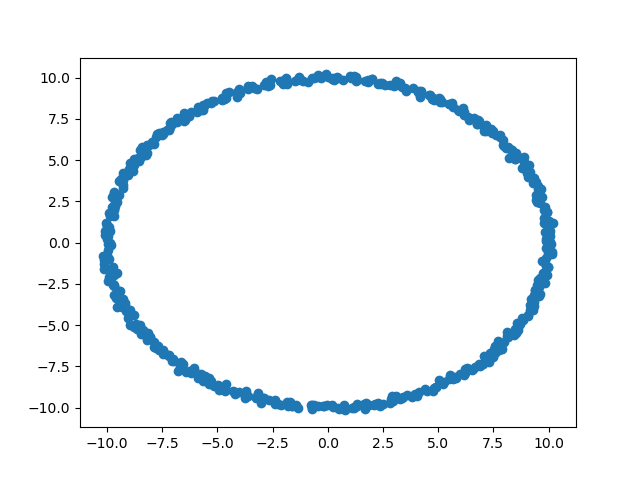} }}%
    \subfloat[\centering \(\fl_{1:20}(X)\)]{{\includegraphics[width=0.2\textwidth]{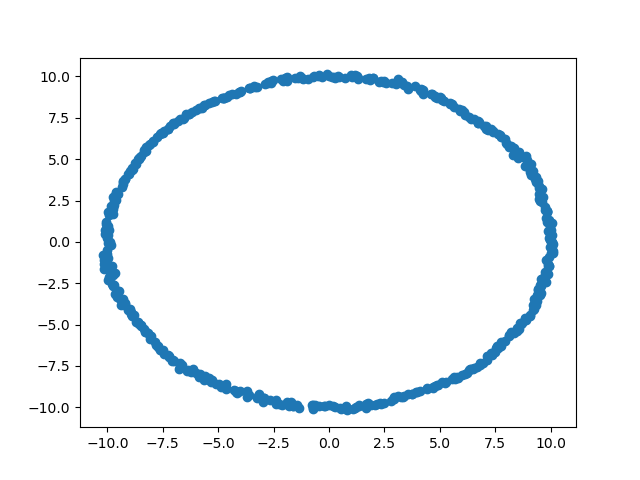} }}%
    \subfloat[\centering \(\fl_{1:40}(X)\)]{{\includegraphics[width=0.2\textwidth]{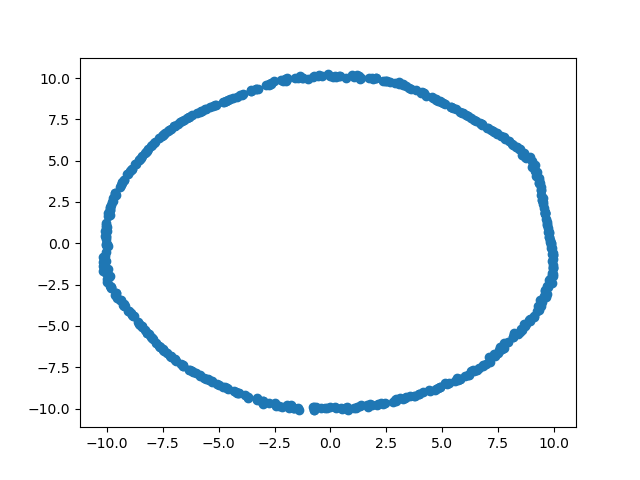} }}%
    \qquad
    \subfloat[\centering \(\fl_{1:60}(X)\)]{{\includegraphics[width=0.2\textwidth]{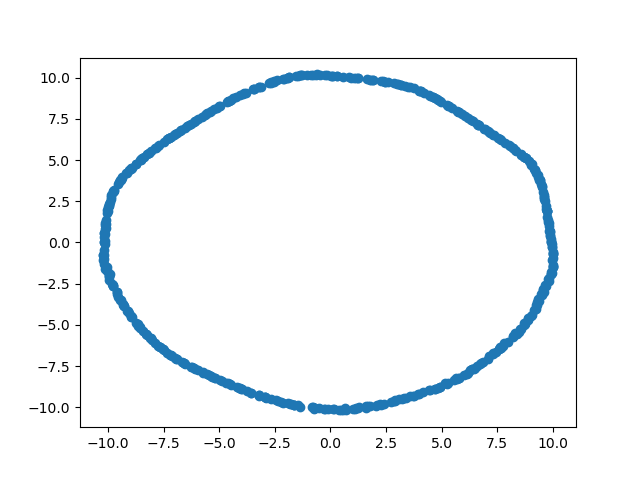} }}%
    \subfloat[\centering \(\fl_{1:80}(X)\)]{{\includegraphics[width=0.2\textwidth]{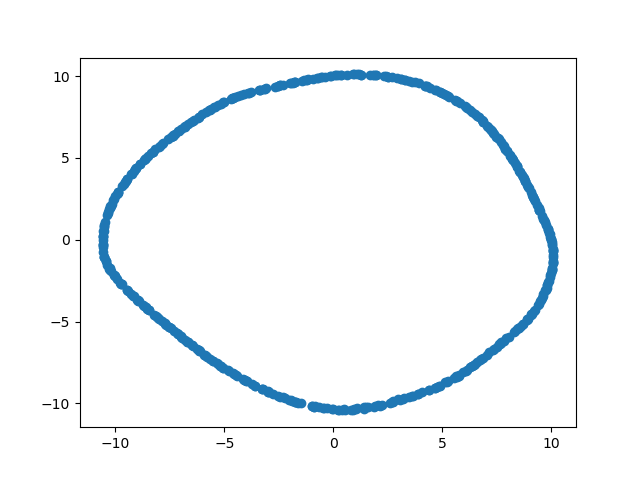} }}%
    \caption{Plots of resulting features \(\fl_{1:\ell}(X)\) through the first \(\ell\) layers of the flattening map \(\fl_{\CC}\), for data sampled from a noisy circle. We observe denoising of the dataset onto a closed curve which oscillates but ultimately does not converge to a linear representation.}
    \label{fig:circle}%
\end{figure}

\begin{figure}[]%
    \centering
    \subfloat[\centering \(X\)]{{\includegraphics[width=0.3\textwidth]{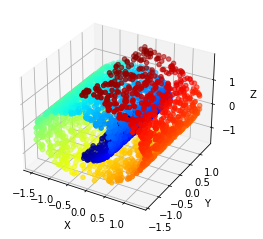} }}%
    \subfloat[\centering \(\fl_{\CC}(X)\) (features)]{{\includegraphics[width=0.3\textwidth]{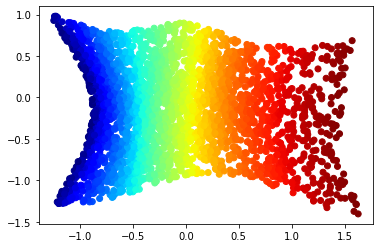} }}%
    \subfloat[\centering \(\re_{\CC}(\fl_{\CC}(X))\) (reconstruction)]{{\includegraphics[width=0.3\textwidth]{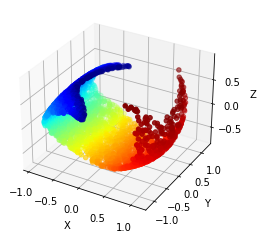} }}%
    \caption{Plots of (a) the Swiss roll dataset (\(N=3000\)), (b) the learned features after a converged FlatNet flattening, and (c) the learned reconstruction from the FlatNet. Note the reconstruction suffers from the same compounding errors problem near the boundary of the manifold, akin to the sine wave example depicted in Figure \ref{fig:sine_example}. However, the representations learned by FlatNet are noticeably close to the original generating coordinates.}
    \label{fig:swiss_roll}%
\end{figure}

\section{Conclusion and future work}

In this work, we propose a computationally tractable algorithm for flattening data manifolds, and in particular generating an autoencoding pair in a forward fashion. Of primary benefit for this methodology is the automation of network design: the network's width and height for example are both chosen automatically to be as minimal as possible. For researchers training a new autoencoder model from scratch, this can pose a large practical benefit.

There is still a lot of room for future work. Our work focuses on the local autoencoding problem for simplicity, but we see the problems of this approach experimentally through the accumulation of the small, local errors through the global map. There are many potential ways one could modify our approach to learn and correct the autoencoder's layers based on the global error. Further, the geometric model for real-world data can be improved and modernized. Outside of the flattenability assumptions (which excludes any closed loops in the manifold, like for example rotation groups for computer vision), it is still a strong assumption that the entire dataset is a single, connected manifold of a single dimension. One common example is a ``multiple manifolds hypothesis'' \citep{brown2022union,vidal2005generalized,yu2020learning}, but there is still more work to do for a geometric model that captures realistic, hierarchical structures. Nonetheless, we feel the presented work builds grounds for geometric-based learning methods in more practical environments which scale with modern data requirements.

\subsection*{Acknowledgements}

We thank Sam Buchanan of TTIC for helpful discussions around existing approaches to manifold flattening and learning.

\newpage
\appendix

\section{Additional experiments}\label{sec:ext_results}

We provide below additional figures to help the reader visualize the performance of FlatNet.

\begin{figure}[H]%
    \centering
    \subfloat[\centering FactorVAE]{{\includegraphics[width=0.35\textwidth]{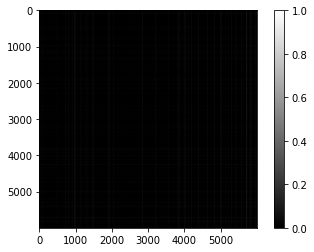} }}%
    \qquad
    \subfloat[\centering FlatNet \textbf{(ours)}]{{\includegraphics[width=0.35\textwidth]{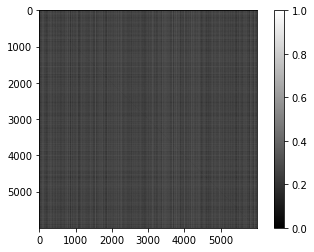} }}%
    \caption{Intrinsic distortion of FactorVAE vs. FlatNet (brighter picture is better), on a random manifold of embedding dimension \(D = 100\), intrinsic dimension \(d = 10\), and training set size \(N = 6000\). We measure the distortion rate by the elementwise ratio of matrices \(\mathrm{EDM}(\fl_{\CC}(X)) / \mathrm{EDM}(C)\), where \(\mathrm{EDM}(\cdot) \in \R^{N \times N}\) is the Euclidean distance matrix of a matrix of samples, \(X \in \R^{D\times N}\) is the training data, and \(C\in \R^{d\times N}\) are the intrinsic coordinates that generated \(X\). Depicted results are scaled by the maximum of the EDM ratio.}
    \label{fig:EDM_pics}
\end{figure}

We also test the intrinsic distortion of FactorVAE of a random manifold over various stopping times (\Cref{fig:EDM_fvae}).
\begin{figure}[]%
    \centering
    \subfloat[\centering 40 Epochs]{{\includegraphics[width=0.3\textwidth]{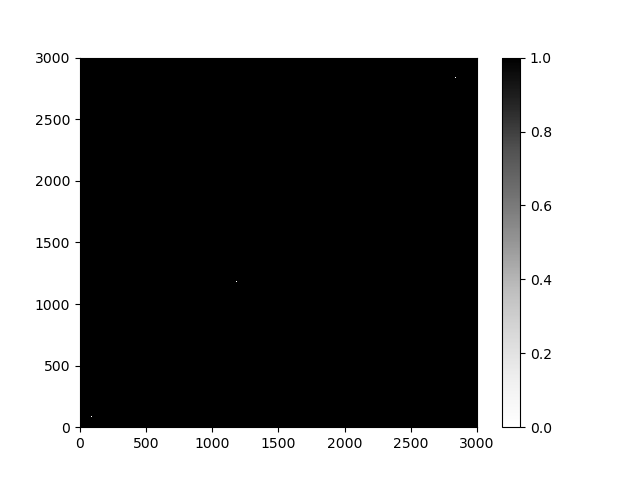} }}%
    \subfloat[\centering 60 Epochs]{{\includegraphics[width=0.3\textwidth]{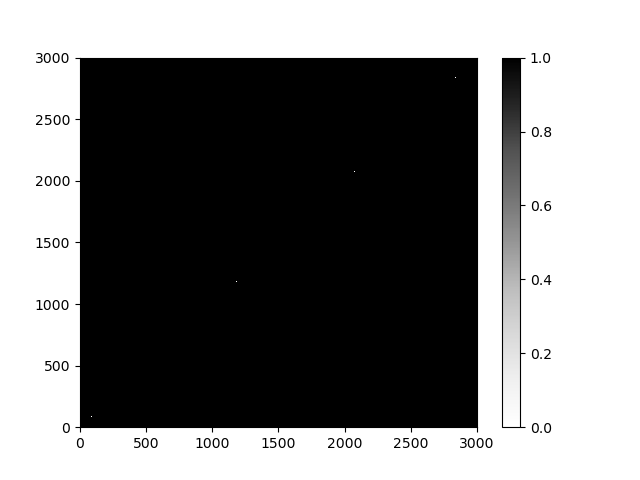} }}%
    \subfloat[\centering 80 Epochs]{{\includegraphics[width=0.3\textwidth]{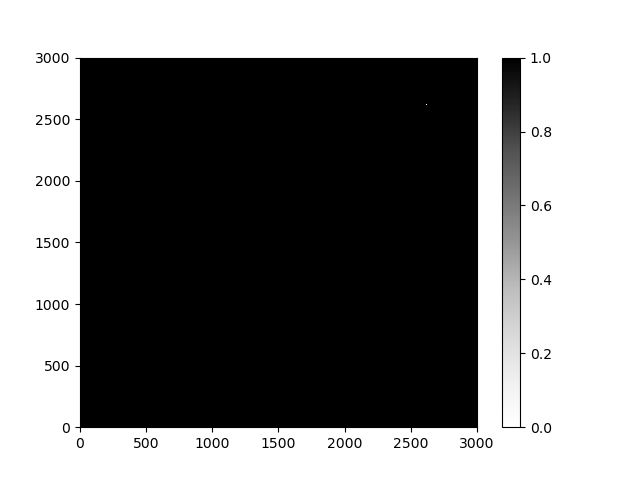} }}%
    \qquad
    \subfloat[\centering 100 Epochs]{{\includegraphics[width=0.3\textwidth]{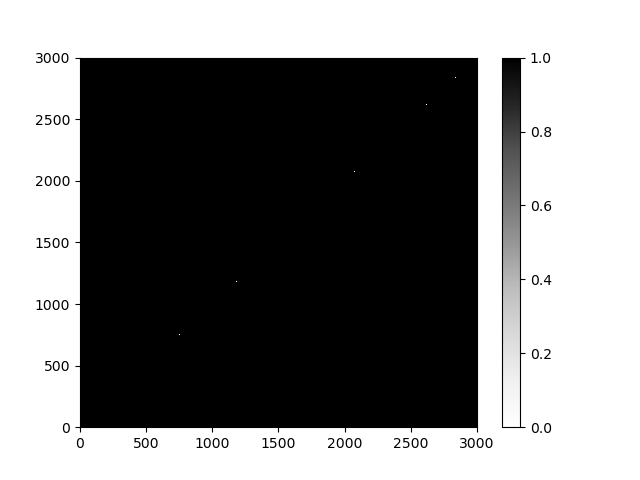} }}%
    \subfloat[\centering 120 Epochs]{{\includegraphics[width=0.3\textwidth]{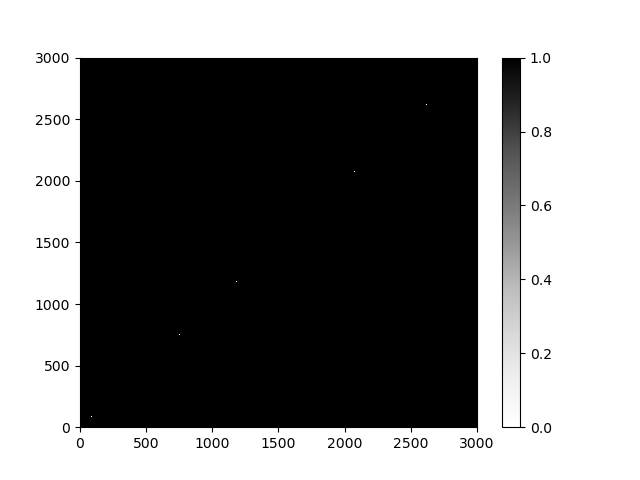} }}%
    \caption{Intrinsic distortion of FactorVAE, on a random manifold of embedding dimension \(D = 100\), intrinsic dimension \(d = 4\), and training set size \(N = 300\). We also observed similar performance with BetaVAE and VanillaVAE.}
    \label{fig:EDM_fvae}%
\end{figure}

\section{Differential geometry overview}\label{sec:diffgeo}

For the sake of being self-contained, we give a brief introduction to some fundamental definitions and constructs in differential geometry needed for this paper. These definitions are not meant to be mathematically complete, but communicate the main ideas used in this paper. Please see the appendix for a more detailed overview, and see \cite{boumal2020introduction,lee2012smoothmanifolds,lee2018introduction} for further introductions to differential geometry and embedded submanifolds.

\subsection{Embedded submanifolds}

While general differential geometry studies general manifolds, our focus for this paper are manifolds that pertain to the \textit{manifold hypothesis}; that is, lower dimensional manifolds explicitly embedded in higher-dimensional Euclidean spaces. The following is a typical definition for such manifolds:

\begin{definition}
    An \textit{embedded submanifold (or a manifold)} \(M\) is a subset of some Euclidean space \(\R^{D}\)  such that for every \(x \in M\), there is some nonempty neighborhood \(N_{\eps}(x) = M \cap B_{\eps}(x)\) that is diffeomorphic with a subset of \(\R^{d}\). The integer \(d\) is called the dimension of the manifold, denoted \(\dim(M)\).
\end{definition}

Intuitively, a an embedded submanifold is a continuous structure that is at every point locally invertible to a \(d\)-dimensional vector representation. For this paper, we will further require manifolds to be \textit{smooth} in order to discuss geodesics and curvature. Please see \cite{lee2012smoothmanifolds} for a comprehensive definition.

\subsection{Tangent space, extrinsic curvature}

Akin to how smooth functions can be locally approximated by a linear function via its derivative, smooth manifolds can be locally approximated by a linear subspace using its \textit{tangent space}; this is a construct commonly used in classical manifold learning. There are many ways to define the tangent space; we provide one here.

\begin{definition}
    Let \(M \subset \R^{D}\) be a manifold of dimension \(d\), and \(x \in M\). The \textit{tangent space} at \(x\), denoted \(\T_{x}M\) is defined as the collection of all velocity vectors of smooth curves \(\gamma\) on \(M\) from \(x\):
    \begin{equation}
        \T_{x}M \doteq \left\{\frac{\mathrm{d}}{\mathrm{d}t} \gamma(0) \middle| \gamma: [0,1] \to M, \gamma(0) = x\right\}.
    \end{equation}
\end{definition}

This is a linear space of dimension \(d\), and we call its orthogonal compliment \(\N_{x}M \doteq \{v \mid \ip{v}{w} = 0\ \mathrm{for all}\ w \in \T_{x}M\}\) the \textit{normal space}. Thus \(\T_{x}M\) encodes a linear approximation to local movement around $x$ on the manifold.

We now introduce curvature. Recall that for smooth functions, if the derivative is constant everywhere, i.e. \(\frac{\mathrm{d}}{\mathrm{d}x} f(x) = c\) for some \(c\), then \(f\) is a globally linear function. We can equivalently conclude this if \(\frac{\mathrm{d}^{2}}{\mathrm{d}x^{2}}f(x) = 0\) everywhere. If the tangent space is analogous to the derivative, then the \textit{extrinsic curvature} is analogous to the second derivative, and heuristically measures how much the tangent space is changing locally. There are again many definitions of extrinsic curvature, and one commonly used definition is through the \textit{second fundamental form}:

\begin{definition}\label{def:sff}
    Let \(M \subset \R^{D}\) be a manifold of dimension \(d\), and \(x \in M\). The second fundamental form at \(x\), denoted \(\sff_{x}\), is the bilinear map \(\sff_{x} \colon \T_{x}M \times \T_{x}M  \to \N_{x}M\) defined as the following:
    \begin{equation}
        \sff_{x}(v, w) \doteq \mathrm{D}_{v} \P_{\T_{x}M}\{w\},
    \end{equation}
    where \(\P_{\T_{x}M}\) is the orthogonal projector from \(\R^{D}\) onto \(\T_{x}M\), and \(\mathrm{D}_{v}\P_{\T_{x}M}\{w\}\) is the differential of the base point \(x\) along \(v\).
\end{definition}

Please see \cite[eq. (5.37)]{boumal2020introduction} for a more complete and rigorous version of the above definition. Intuitively, \(\sff_{x}\) measures how much the tangent space changes locally. Indeed, if \(\sff_{x}\) is identically 0 everywhere on \(M\), then the tangent space \(\T_{x}M\) doesn't change anywhere, and \(M\) is a globally linear subspace of dimension \(d\). This motivates our use of \(\sff\) to characterize the nonlinearity of a manifold \(M\), and motivates our definition of \textit{flatness} and the following immediate corollary:

\begin{definition}
    A manifold \(M \subset \R^{D}\) is said to be \textit{flat} if \(\sff_x(v, w) = 0\) for all \(x \in M\) and all \(v, w \in \T_{x}M\).
\end{definition}

\begin{corollary}
    \textit{If a manifold \(M \subset \R^D\) of dimension \(d\) is flat, then \(M\) is globally contained within an affine subspace of dimension \(d\), i.e., \(M \subset \T_{x}M + x\) for any \(x \in M\).}
\end{corollary}

Finally, we introduce geodesics. One primary challenge of nonlinear datasets is the inability to interpolate via convex interpolation; indeed, for manifolds, the linear chord \(\gamma(t) \doteq (1-t)x + ty\) for \(x, y \in M\) rarely stays in the manifold for all \(t \in [0,1]\). While there are typically many interpolations \(\gamma(t)\) between \(x\) and \(y\) that stay within \(M\), there is at most one that minimizes the path length: this is called the \textit{geodesic} between \(x\) and \(y\).

\begin{definition}
    Let \(M \subset \R^{D}\) be a manifold of dimension \(d\). For any two \(x, y \in M\), if there exists a smooth curve \(\gamma: [0,1] \to \M\) such that \(\gamma(0) = x\) and \(\gamma(1) = y\), the \textit{geodesic} between \(x\) and \(y\) is the interpolating curve of minimal arc length:
    \begin{equation}
        \gamma^{\star} \doteq \argmin_{\substack{\gamma \colon [0, 1] \to M \\ \gamma(0) = x, \gamma(1) = y}} \int_{0}^{1} \norm{\frac{\mathrm{d}}{\mathrm{d}t} \gamma(t)}_{2}\ \mathrm{d}t.
    \end{equation}
\end{definition}

For this paper, we assume for all manifolds \(M\) that all pairs of points have geodesics between them; this is akin to assuming the manifold is connected.
\section{Geometric flows for manifold flattening}

We now draw a relation between \Cref{alg:ccnet_construction} and a more traditional methods of manifold flattening: \textit{geometric flows}. A common practice in differential geometry for manifold manipulation is to define a differential equation and study the evolution of \(\M\) through the corresponding dynamics. Such equations are often called \textit{geometric evolution equations}, or \textit{geometric flows}, and are typically designed to evolve complicated manifolds into simpler, more uniform ones. They are responsible for some powerful geometric theorems, such as the uniformization theorem, and are a heavily studied area in differential geometry. 

Arguably the most well-known of the geometric flows is the \textit{Ricci flow} \citep{hamilton1982three}. While commonly used in theoretical work, the Ricci flow is ill-suited for embedded submanifold flattening, as it is an open problem as to whether a realization of the Ricci flow even exists in the embedding space \citep{coll2020ricci}. There are also geometric flows that minimize extrinsic curvature (as opposed to intrinsic curvature), such as the \textit{curve-shortening flow} \citep{abresch1986normalized} and \textit{mean curvature flow} \citep{huisken1984flow}. While there is rich theory behind both of these flows, each requires restricted settings which are non-ideal for data manifolds, with the former being defined only on curves (\(d = 1\)) and the latter on hypersurfaces (\(d = D - 1\)).

In this section, we introduce a new geometric flow that is well-defined on general embedded submanifolds of a Euclidean space, thus extending to more realistic settings for data manifolds. Our flow will flatten the input manifold \(\M\). The encoding map \(\fl_{\mathrm{CC}} \colon \R^{D} \to \R^{D}\) we seek to learn will then act on a point \(\xz \in \R^{D}\) by approximating the flow starting at \(\xz\).

\subsection{Convexification flow}\label{sec:proposed_flow}

Recall from \Cref{thm:conv_flat} that flatness and convexity are intimately connected. Thus, our flow focuses on minimizing the volume of the difference between convex hull of the manifold and the manifold itself. Accordingly, we call it the \textit{convexification flow}.

We now discuss the mechanics of how the flow should behave. In order to compress \(\M\) until it becomes convex, it should point the velocity vector of all points \(\xz \in \M\) towards the boundary \(\partial \conv(\M)\). However, \(\conv(\M)\) is challenging to compute in high dimensions \cite[Chp. 3]{erickson1996lower,gale1963neighborly}, so the flow should only use local information that computes a direction from \(\xz\) towards \(\partial \conv(\M)\).

Fix a smooth integrable function \(\pou \colon \R^{D} \to [0, 1]\). Recall that the local average over \(\M\), defined as
\begin{equation}
    A_{\M}(\xz) \doteq \frac{\int_{\M}x\pou(x)\ \mathrm{d}x}{\int_{\M}\pou(x)\ \mathrm{d}x}
\end{equation}
is not necessarily contained in \(\M\), but when there is some nonzero curvature (i.e., the manifold is not locally flat) it is contained in the set difference \(\conv(\M) \setminus \M\). Thus, it seems that we can use local averages to direct the flow velocity towards the boundary of the convex hull. 

More precisely, we define the following flow, which (as previously stated) we call the \textit{convexification flow}.
\begin{definition}[Convexification Flow]
    Let \(\xz \in \M\). The trajectory under the convexification flow \(x \colon [0, \infty) \to \R^{D}\) of \(\xz\) is given by the initial value problem (IVP):
    \begin{align}\label{eq:unnormalized}
        \frac{\mathrm{d}}{\mathrm{d}t} x(t) 
        &= \bar{x}(t) - x(t),  \\
        x(0) 
        &= x_{0}. \label{eq:unnormalized_ivp}
    \end{align}
    Here \(\bar{x}(t)\) is a short-hand for the local average of \(x(t)\) at time \(t\):
    \begin{equation}
        \bar{x}(t) \doteq A_{\M(t)}(x(t)),
    \end{equation}
    where \(\M(t)\) is the transformation of the original manifold \(\M\) under the flow until time \(t\):
    \begin{equation}
        \M(t) \doteq \{x(t) \mid \text{eq. \eqref{eq:unnormalized} holds for \(x\) on \([0, t)\),}\ x(0) \in \M\}.
    \end{equation}
\end{definition}
Fix \(t \geq 0\). Suppose that \(\M(t)\) is convex. Then if \(x(t)\) is in the interior \(\M(t)^{\circ}\), we have \(x(t) = \bar{x}(t)\). This suggests that the above flow is stationary, or halts, whenever \(\M(t)\) is convex. However, even if \(\M(t)\) is convex, the equality \(\bar{x}(t) = x(t)\) does not necessarily hold on the boundary \(\partial \M(t)\). This phenomenon is important; as is the case with many unnormalized geometric flows, this ODE will evolve \(\M(t)\) into a singularity as \(t \to \infty\). Thus, we need to normalize this flow.

\subsection{Normalized convexification flow}\label{sec:normalized}

Analyzing the cause of singularity in \Cref{eq:unnormalized}, we see that fatal collapse only starts when \(\frac{\mathrm{d}}{\mathrm{d}t} x(t)\) points towards \(\M(t)\) itself. This can be alleviated by restricting the velocity vector \(\bar{x}(t) - x(t)\) to belong to the normal space \(\N_{x(t)}\M(t)\) at \(x(t)\). Thus, we define the following flow, which we call the \textit{normalized convexification flow}. 

\begin{definition}[Normalized Convexification Flow]
    Let \(\xz \in \M\). The trajectory under the normalized convexification flow \(x \colon [0, \infty) \to \R^{D}\) of \(\xz\) is given by the following IVP:
    \begin{align}\label{eq:normalized}
        \frac{\mathrm{d}}{\mathrm{d}t} x(t) 
        &= \P_{\N_{x(t)}\M(t)}\{\bar{x}(t) - x(t)\},  \\
        x(0) 
        &= \xz. \label{eq:normalized_ivp}
    \end{align}
    Here \(\P_{\N_{x(t)}\M(t)}\) is the orthogonal projection operator onto the subspace \(\N_{x(t)}\M(t)\), and \(\bar{x}(t)\) and \(\M(t)\) are defined analogously to before:
    \begin{align}
        \bar{x}(t) 
        &\coloneqq A_{\M(t)}(x(t)), \\
        \M(t) 
        &\coloneqq \{x(t) \mid \text{eq. \eqref{eq:normalized} holds for \(x\) on \([0, t)\),}\ x(0) \in \M\}.
    \end{align}
\end{definition}

This flow does not completely avoid singularities; we resolve this issue shortly when we discretize the flow, and in the next section when we make some algorithmic tweaks.

\subsection{Discretizing the normalized convexification flow}\label{sub:discretization}

We now show how to recover a simplified version of \Cref{alg:ccnet_construction} by discretizing the convexification flow. Fix \(h > 0\) to be our discretization interval, and let \(k \geq 0\) be a non-negative integer. From here on, we will write \(x[k]\) to denote \(x(kh)\), and similarly for \(\bar{x}\) and \(\M\).

Let \(x \colon [0, \infty) \to \R^{D}\) evolve according to the normalized convexification flow in \Cref{eq:normalized,eq:normalized_ivp}. Our discretization will write \(x[k + 1] = x((k + 1)h)\) (approximately) in terms of \(x[k] = x(kh)\). We use a first-order Taylor approximation around \(x[k + 1]\):
\begin{align}
    x[k + 1]
    &= x((k + 1)h) \\
    &\approx x(kh) + h\left[\frac{\mathrm{d}}{\mathrm{d}t} x(t)\right]_{t = kh} \\
    &= x(kh) + h\P_{\N_{x(kh)}\M(kh)}\{\bar{x}(kh) - x(kh)\} \\
    &= x[k] + h\P_{\N_{x[k]}\M[k]}\{\bar{x}[k] - x[k]\} \\
    &= x[k] + h(\id_{\R^{D}} - \P_{\T_{x[k]}\M[k]})\{\bar{x}[k] - x[k]\} \\
    &= x[k] + h\bar{x}[k] - hx[k] - h\P_{\T_{x[k]}\M[k]}\bar{x}[k] + h\P_{\T_{x[k]}\M[k]}x[k] \\
    &= (1 - h)x[k] + h(\bar{x}[k] + \P_{\T_{x[k]}\M[k]}\{x[k] - \bar{x}[k]\}), \\
    &= (1 - h)x[k] + h\P_{\T_{x[k]}\M[k] + \bar{x}[k]}\{x[k]\} \label{eq:discrete_flow_step}
\end{align}
with exact equality achieved as \(h \to 0\). If we extend this step in \Cref{eq:discrete_flow_step} to a neighborhood around \(x[k]\), fixing the projector \(\P_{\T_{x[k]}\M[k] + \bar{x}[k]}\) with respect to the central point \(x[k]\), then \Cref{eq:discrete_flow_step} matches the forward map defined in \Cref{eq:local_flattening,eq:part_unity}; in particular, the quantity \(h\) in \Cref{eq:discrete_flow_step} corresponds to the partition of unity \(\pou(x)\) in \Cref{eq:part_unity}.
\section{Proofs}\label{sec:proofs}

We provide below important proofs for theory included in the main body.

\vspace{6mm}

\begin{proof}\textbf{of Theorem \ref{thm:conv_flat}}

    As we are primarily studying the set $\fl(\M)$, we denote this manifold $\mathcal{Z}$ for brevity.
    \begin{enumerate}
        \item This comes as an immediate corollary to the following geometric lemma \cite[Proposition 8.12]{lee2018introduction}:

        \begin{lemma}
            An embedded submanifold $\M \subset \R^D$ is flat if and only if the geodesics of $\M$ are geodesics in the embedding space $\R^D$.
        \end{lemma}

        As the geodesics of Euclidean space are straight lines, it follows that the geodesics of $\mathcal{Z}$ are likewise straight lines. Since $h$ is smooth and $\M$ is compact and connected, it follows that $\mathcal{Z}$ is compact and connected, and each pair of points $z_1, z_2$ have a geodesic between them. Thus, all point pairs $z_1, z_2 \in \mathcal{Z}$ are connected by straight lines, making the set $\mathcal{Z}$ convex.

        \item This claim is proven in two parts: (a) $\mathcal{Z} \subset z + \T_z \mathcal{Z}$ for any $z \in \mathcal{Z}$, and (b) $z_1 + \T_{z_1} \mathcal{Z} = z_2 + \T_{z_2} \mathcal{Z}$ for all $z_1, z_2 \in \mathcal{Z}$.
        
        Fix an arbitrary $z \in \mathcal{Z}$. As established in part 1, all pairs of points in $\mathcal{Z}$ are connected by a geodesic, and all geodesics of $\mathcal{Z}$ are straight lines. Thus, all points $z' \in \mathcal{Z}$ can be represented as $z' = z + tv$ for some $t \in \R$ and $v \in \T_z \mathcal{Z}$, and it follows that $\mathcal{Z} \subset z + \T_z \mathcal{Z}$.

        For the second claim, note that $\T_{z_1} \mathcal{Z} = \T_{z_2} \mathcal{Z}$ for all pairs of points $z_1, z_2 \in \mathcal{Z}$, as the second fundamental form is identically zero everywhere. Denoting the shared tangent space $\T_z \mathcal{Z}$, it then suffices to show that $z_1 - z_2 \in \T_z \mathcal{Z}$ for all $z_1, z_2 \in \mathcal{Z}$. This is indeed the case, as the geodesics of $\mathcal{Z}$ are of the form $\gamma(t) = z_1 + t(z_2 - z_1)$, and $\gamma'(0) = z_2 - z_1$.

        \item The autoencoding property $g(f(x)) = x$ for all $x \in \M$ holds trivially from construction, so what is left is to show is that there is no pair of functions $f : \R^D \to \R^p$ and $g : \R^p \to \R^D$ where $g(f(x)) = x$ for all $x \in \M$, and $p < d$.

        Denote $\mathrm{D}f(x)[v] := \lim_{t\to 0} \frac{f(x+ tv) - f(x)}{t}$, the differential of $f$ at $x$ along $v$. This notation is used to emphasize that $\mathrm{D}f(x)$ is a linear map. The following chain rule holds \citep{boumal2020introduction}:
        \begin{equation}
            \mathrm{D} (g \circ f) (x)[v] = \mathrm{D} g(f(x))[\mathrm{D} f(x)[v]].
        \end{equation}
        Using the autoencoding equality over $\M$, and the fact that $\mathrm{D} (\mathrm{id})[v] = v$, we get the following equality for any fixed $x \in \M$:
         \begin{equation}
             \mathrm{D} g(f(x))[\mathrm{D} f(x)[v]] = v,
         \end{equation}
        for any $v \in \T_x \M$. Since $\T_x \ M$ is a linear space of dimension $d$, it follows that the composite linear map $\mathrm{D} g(f(x))[\mathrm{D} f(x)[v]]$ must have rank of at least $d$. Since $\mathrm{D} f(x)[v]$ is a linear map from $\R^D$ to $\R^p$, this implies that $p \ge d$.
    \end{enumerate}
\end{proof}

\begin{proof}\textbf{of Lemma \ref{lem:local_inv}}

    Define \(\beta_{z}(x) \coloneqq (1 - e^{-\lambda z}e^{-\lambda x})^{2}x\), $\P_{U + \xc}\{z\} \coloneqq UU^\top(z - \xc) + \xc$, and $(I - \P_{U + \xc})\{z\} \coloneqq z - \P_{U + \xc}\{z\}$. Since $\beta_z(x)$ is a strictly monotonically increasing function for $x, z \ge 0$ (product of strictly monotonically increasing functions on positive input), $\beta_z(x)$ is an invertible scalar function for all $z \ge 0$. Denote $\eta_1 \coloneqq \|\P_{U + \xc}\{z\}\|_2^2 = \|\P_{U + \xc}\{x\}\|_2^2$, and $\eta_2 \coloneqq \|(I - \P_{U + \xc})\{z\}\|_2^2 = (1-\pou(x))^2 \|(I - \P_{U + \xc})\{x\}\|_2^2$. Further notate the original norms $\nu_1 \coloneqq \|\P_{U + \xc}\{x\}\|_2^2$ and $\nu_2 \coloneqq \|(I - \P_{U + \xc})\{x\}\|_2^2$. Note that under this notation, $\pou(x) = e^{-\lambda(\nu_1 + \nu_2)}$. Since $\nu_1 = \eta_1$, we simply need to find an equation for $\nu_2$ from $\eta_1, \eta_2$.
    \begin{align}   
        \eta_2 &= (1-\phi(x))^2 \|(I-\P_{U, x_c})x\|_2^2,\\
        &= (1-e^{-\lambda(\nu_1 + \nu_2)})^2 \nu_2,\\
        &= (1-e^{-\lambda\eta_1}e^{-\lambda\nu_2})^2 \nu_2,\\
        &= \beta_{\eta_1}(\nu_2).
    \end{align}
    Finally, we get the following equation for our partition of unity as a function of the output features $z$:
    \begin{equation} 
        \pou(x) = \pouinv(z) \coloneqq \alpha e^{-\gamma\left(\|\P_{U, x_c} z\|_2^2 + \beta_{\|\P_{U, x_c} z\|_2^2}^{-1}(\|(I - \P_{U, x_c}) z\|_2^2)\right)}.
    \end{equation}
    Computing the above function amounts to inverting a scalar function, which reduces to scalar root-finding.
\end{proof}

\begin{proof}\textbf{of Proposition \ref{prop:V_normal_space}}

    \begin{enumerate}
        \item For convenience, let \(\alpha_{ij} = \ip{\tu_{j}}{\tU\tU^{\top}(x_{i} - \xz)} = \ip{\tu_{j}}{x_{i} - \xz}\). We write
        \begin{align}
            &\mathcal{L}_{\pou}(\tU, \tV) \\
            &= \frac{1}{N}\sum_{i = 1}^{N}\pou(x_{i})^{2}\norm{\tV(\tU\tU^{\top}(x_{i} - \xz), \tU\tU^{\top}(x_{i} - \xz)) - (I - \tU\tU^{\top})(x_{i} - \xz)}_{2}^{2}, \\
            &= \frac{1}{N}\sum_{i = 1}^{N}\pou(x_{i})^{2}\norm{\sum_{j = 1}^{d}\sum_{k = 1}^{d}\alpha_{ij}\alpha_{ik}\tv_{jk} - (I - \tU\tU^{\top})(x_{i} - \xz)}_{2}^{2}, \\
            &= \frac{1}{N}\sum_{i = 1}^{N}\pou(x_{i})^{2}\norm{\tU\tU^{\top}\left(\sum_{j = 1}^{d}\sum_{k = 1}^{d}\alpha_{ij}\alpha_{ik}\tv_{jk}\right) - (I - \tU\tU^{\top})\left(x_{i} - \xz - \sum_{j = 1}^{d}\sum_{k = 1}^{d}\alpha_{ij}\alpha_{ik}\tv_{jk}\right)}_{2}^{2}, \\
            &= \frac{1}{N}\sum_{i = 1}^{N}\pou(x_{i})^{2}\norm{\tU\tU^{\top}\left(\sum_{j = 1}^{d}\sum_{k = 1}^{d}\alpha_{ij}\alpha_{ik}\tv_{jk}\right)}_{2}^{2} \\
            &\quad + \frac{1}{N}\sum_{i = 1}^{N}\pou(x_{i})^{2}\norm{(I - \tU\tU^{\top})\left(x_{i} - \xz - \sum_{j = 1}^{d}\sum_{k = 1}^{d}\alpha_{ij}\alpha_{ik}\tv_{jk}\right)}_{2}^{2} \nonumber \\
            &\geq \frac{1}{N}\sum_{i = 1}^{N}\pou(x_{i})^{2}\norm{(I - \tU\tU^{\top})\left(x_{i} - \xz - \sum_{j = 1}^{d}\sum_{k = 1}^{d}\alpha_{ij}\alpha_{ik}\tv_{jk}\right)}_{2}^{2}
        \end{align}
        with equality if and only if \(\tU\tU^{\top}\left(\sum_{j = 1}^{d}\sum_{k = 1}^{d}\alpha_{ij}\alpha_{ik}\tv_{jk}\right) = 0\) for all \(i \in \{1, \ldots, N\}\), which since $\tU$ are full rank is true if and only if $\tU^{\top}\left(\sum_{j = 1}^{d}\sum_{k = 1}^{d}\alpha_{ij}\alpha_{ik}\tv_{jk}\right) = 0$.

        \item This amounts to noticing the cost function $\mathcal{L}_{\psi}(\tU, \tV)$ can be written as a least squares problem, and accounting for trivial redundancy. We rewrite the formula of $\mathcal{L}_{\psi}(\tU, \tV)$ here for convenience:
        \begin{align}
            &\mathcal{L}_{\pou}(\tU, \tV) \\
            &=\frac{1}{N}\sum_{i = 1}^{N}\pou(x_{i})^{2}\norm{\tV(\tU\tU^{\top}(x_{i} - \xz), \tU\tU^{\top}(x_{i} - \xz)) - (I - \tU\tU^{\top})(x_{i} - \xz)}_{2}^{2},\\
            &= \frac{1}{N}\sum_{i = 1}^{N}\pou(x_{i})^{2}\norm{\sum_{j=1}^d\sum_{k=1}^d \tv_{jk}\ip{\tu_j}{x_i - \xz}\ip{\tu_k}{x_i - \xz} - (I - \tU\tU^{\top})(x_{i} - \xz)}_{2}^{2}.
        \end{align}
        Define $B$ again as in the proposition statement: $B_{ijk} = \ip{\tu_j}{x_i - \xz}\ip{\tu_k}{x_i - \xz}$. We can flatten out the last two indices to get a matrix $A\in \R^{N\times d^2}$, such that $A_{i, j + d\cdot k} = B_{ijk}$. If we further stack the vectors $\tv_{jk}$ into a matrix $M_{\tv} \in \R^{d^2\times D}$, where the $(j + d\cdot k)$\textsuperscript{th} row is $\tv_{jk}$, and a matrix $C \in \R^{N\times D}$ such that the $i$\textsuperscript{th} row of $C$ is $(I - \tU\tU^{\top})(x_{i} - \xz)$. Then we can write the following:
        \begin{equation}
            \mathcal{L}_{\pou}(\tU, \tV) = \|D_{\pou(x_i)}AM_{\tv} - D_{\pou(x_i)}C\|_F^2,
        \end{equation}
        where $D_{\pou(x_i)} \in \R^{N\times N}$ is the diagonal matrix such that $D_{\pou(x_i)}{}_{ii} = \pou(x_i)$.
        
        An important problem with the above formulation is that $A$ will never be full rank, since $B_{ijk} = B_{ikj}$; this will lead to trivial duplicate entries in the rows of $A$. This is where the proposition's construction comes in: if we instead construct $A'\in \R^{N\times \frac{1}{2}(d^2 + d)}$ such that the $i$\textsuperscript{th} row is only the flattened upper diagonal component of $B_i \in \R^{d\times d}$, and further modify $M_{\tv}' \in \R^{\frac{1}{2}(d^2 + d)\times D}$ to only contain upper-diagonal entries of the list $\tv_{jk}$, scaling each off-diagonal entry by $2$, then we can write the following:
        \begin{equation}
            \mathcal{L}_{\pou}(\tU, \tV) = \|D_{\pou(x_i)}A'M'_{\tv} - D_{\pou(x_i)}C\|_F^2,
        \end{equation}
        where the matrix $A'$ can now feasibly be full column-rank. As the above is a standard least squares problem, the solution $M'_{\hv}$ (and equivalently the solution for $\hV$) is unique only when $D_{\pou(x_i)}A'$ is of full column-rank, which since $\pou(x_i) > 0$ for all $x_i$ means $\hV$ is unique only when $A'$ is full rank.

        \item Note that, using the notation of part 2., part 1. implies that $A'M'_{\hv}\tU = 0$. Since $A'$ is of full column rank, this implies that $M'_{\hv}\tU = 0$, which implies $\tU^\top \hv_{jk} = 0$ for all $1 \le j, k, \le d$. Since any output of $\hV$ is a linear combination of $\hv_{jk}$, it follows that $\tU^\top \hV(w_1, w_2) = 0$ for all $w_1, w_2 \in \R^D$.
    \end{enumerate}
\end{proof}
\section{Algorithmic analysis}\label{sec:alg_analysis}

We provide here some algorithmic analysis to give the reader an idea for the computational burden that computing FlatNet entails. While we currently do not have rigorous computational guarantees, we can give ideas for scaling by providing critical point characterization, and time complexity of computing the cost function (including all network evaluations).

\subsection{Time complexity}\label{sec:time_complexity}

To give the reader an idea of how FlatNet scales with data size and complexity, we provide some basic asymptotic time complexity analysis on the main bottleneck computation for constructing a FlatNet pair $f, g$: the optimization in eq. \eqref{eq:autoencoding_loss}. The following theorem encapsulates this time complexity analysis:

\begin{theorem}
    The loss function given in eq. \eqref{eq:autoencoding_loss} of the main body can be computed in $O(NDd^2)$ flops.
\end{theorem}
\begin{proof}
    For reference, the cost function in question is the following:

    \begin{equation}\label{eq:proof_cost}
        \mathcal{L}_{\pou}(\tU, \tV) = \frac{1}{N}\sum_{i = 1}^{N}\pou(x_{i})^{2}\norm{\sum_{j=1}^d\sum_{k=1}^d \tv_{jk}\ip{\tu_j}{x_i - \xz}\ip{\tu_k}{x_i - \xz} - (I - \tU\tU^{\top})(x_{i} - \xz)}_{2}^{2}.
    \end{equation}

    The outer sum's size of eq. \eqref{eq:proof_cost} will be at most $N$, so we can incur an $O(N)$ cost and focus our analysis on the summand for fixed $x \in X$:

    \begin{equation}\label{eq:proof_g_doublesum}
        \norm{\sum_{j=1}^d\sum_{k=1}^d \tv_{jk}\ip{\tu_j}{x_i - \xz}\ip{\tu_k}{x_i - \xz} - (I - \tU\tU^{\top})(x_{i} - \xz)}_{2}^{2}
    \end{equation}

    Since $U\in \R^{D \times d}$, $U^\top (x_{i} - \xz)$ is computable in $O(Dd)$ flops, and the subsequent multiplication of $U$ requires another $O(Dd)$ flops. Thus, $(I - \tU\tU^{\top})(x_{i} - \xz)$ is computable in $O(Dd)$ multiplications. For the remaining double sum $\sum_{j=1}^d\sum_{k=1}^d \tv_{jk}\ip{\tu_j}{x_i - \xz}\ip{\tu_k}{x_i - \xz}$, we can reduce to a single summand as before and incur a cost of $O(d^2)$, reducing time complexity analysis to the following:

    \begin{equation}
        \tv_{jk}\ip{\tu_j}{x_i - \xz}\ip{\tu_k}{x_i - \xz}
    \end{equation}

    Since $\ip{\tu_j}{x_i - \xz}$ have already been computed from $U^\top (x_i - \xz)$ and $\tv_{jk} \in \R^{D}$, the resulting expression requires $O(D)$ multiplications to compute. Thus, eq. \eqref{eq:proof_g_doublesum} requires $O(Dd^2)$ flops to compute. The last uncomputed operation is the final $L^2$ norm of the $D$-dimensional difference vector, which requires $O(D)$ flops. Finally, the overall computation takes $O(N(Dd^2 + Dd + D)) = O(NDd^2)$ flops.
\end{proof}

\bibliography{cc}

\begin{thebibliography}{74}
\providecommand{\natexlab}[1]{#1}
\providecommand{\url}[1]{\texttt{#1}}
\expandafter\ifx\csname urlstyle\endcsname\relax
  \providecommand{\doi}[1]{doi: #1}\else
  \providecommand{\doi}{doi: \begingroup \urlstyle{rm}\Url}\fi

\bibitem[Abresch and Langer(1986)]{abresch1986normalized}
U.~Abresch and J.~Langer.
\newblock The normalized curve shortening flow and homothetic solutions.
\newblock \emph{Journal of Differential Geometry}, 23\penalty0 (2):\penalty0
  175--196, 1986.

\bibitem[Arora et~al.(2018)Arora, Risteski, and Zhang]{arora2018gans}
S.~Arora, A.~Risteski, and Y.~Zhang.
\newblock Do gans learn the distribution? some theory and empirics.
\newblock In \emph{International Conference on Learning Representations}, 2018.

\bibitem[Belkin and Niyogi(2003)]{belkin2003laplacian}
M.~Belkin and P.~Niyogi.
\newblock Laplacian eigenmaps for dimensionality reduction and data
  representation.
\newblock \emph{Neural computation}, 15\penalty0 (6):\penalty0 1373--1396,
  2003.

\bibitem[Bengio et~al.(2013)Bengio, Courville, and
  Vincent]{bengio2013representation}
Y.~Bengio, A.~Courville, and P.~Vincent.
\newblock Representation learning: A review and new perspectives.
\newblock \emph{IEEE transactions on pattern analysis and machine
  intelligence}, 35\penalty0 (8):\penalty0 1798--1828, 2013.

\bibitem[Birsan and Tiba(2005)]{birsan2005one}
T.~Birsan and D.~Tiba.
\newblock One hundred years since the introduction of the set distance by
  dimitrie pompeiu.
\newblock \emph{System Modelling and Optimization}, 199:\penalty0 35--39, 2005.

\bibitem[Boumal(2020)]{boumal2020introduction}
N.~Boumal.
\newblock An introduction to optimization on smooth manifolds.
\newblock \emph{Available online, Aug}, 2020.

\bibitem[Brown et~al.(2022)Brown, Caterini, Ross, Cresswell, and
  Loaiza-Ganem]{brown2022union}
B.~C. Brown, A.~L. Caterini, B.~L. Ross, J.~C. Cresswell, and G.~Loaiza-Ganem.
\newblock The union of manifolds hypothesis and its implications for deep
  generative modelling.
\newblock \emph{arXiv preprint arXiv:2207.02862}, 2022.

\bibitem[Campadelli et~al.(2015)Campadelli, Casiraghi, Ceruti, and
  Rozza]{campadelli2015intrinsic}
P.~Campadelli, E.~Casiraghi, C.~Ceruti, and A.~Rozza.
\newblock Intrinsic dimension estimation: Relevant techniques and a benchmark
  framework.
\newblock \emph{Mathematical Problems in Engineering}, 2015:\penalty0 1--21,
  2015.

\bibitem[Carter et~al.(2009)Carter, Raich, and Hero~III]{carter2009local}
K.~M. Carter, R.~Raich, and A.~O. Hero~III.
\newblock On local intrinsic dimension estimation and its applications.
\newblock \emph{IEEE Transactions on Signal Processing}, 58\penalty0
  (2):\penalty0 650--663, 2009.

\bibitem[Chan et~al.(2022)Chan, Yu, You, Qi, Wright, and Ma]{chan2022redunet}
K.~Chan, Y.~Yu, C.~You, H.~Qi, J.~Wright, and Y.~Ma.
\newblock Redunet: A white-box deep network from the principle of maximizing
  rate reduction.
\newblock \emph{Journal of machine learning research}, 23\penalty0 (114), 2022.

\bibitem[Chen(1996)]{chen1996rapid}
C.~P. Chen.
\newblock A rapid supervised learning neural network for function interpolation
  and approximation.
\newblock \emph{IEEE Transactions on Neural Networks}, 7\penalty0 (5):\penalty0
  1220--1230, 1996.

\bibitem[Chen et~al.(2022)Chen, Chewi, Li, Li, Salim, and
  Zhang]{chen2022sampling}
S.~Chen, S.~Chewi, J.~Li, Y.~Li, A.~Salim, and A.~R. Zhang.
\newblock Sampling is as easy as learning the score: theory for diffusion
  models with minimal data assumptions.
\newblock \emph{arXiv preprint arXiv:2209.11215}, 2022.

\bibitem[Coll et~al.(2020)Coll, Dodd, and Johnson]{coll2020ricci}
V.~E. Coll, J.~Dodd, and D.~L. Johnson.
\newblock Ricci flow on surfaces of revolution: an extrinsic view.
\newblock \emph{Geometriae Dedicata}, 207\penalty0 (1):\penalty0 81--94, 2020.

\bibitem[Costa and Hero(2006)]{costa2006determining}
J.~A. Costa and A.~O. Hero.
\newblock Determining intrinsic dimension and entropy of high-dimensional shape
  spaces.
\newblock In \emph{Statistics and analysis of shapes}, pages 231--252.
  Springer, 2006.

\bibitem[Deng(2012)]{deng2012mnist}
L.~Deng.
\newblock The mnist database of handwritten digit images for machine learning
  research.
\newblock \emph{IEEE Signal Processing Magazine}, 29\penalty0 (6):\penalty0
  141--142, 2012.

\bibitem[Devlin et~al.(2018)Devlin, Chang, Lee, and Toutanova]{devlin2018bert}
J.~Devlin, M.-W. Chang, K.~Lee, and K.~Toutanova.
\newblock Bert: Pre-training of deep bidirectional transformers for language
  understanding.
\newblock \emph{arXiv preprint arXiv:1810.04805}, 2018.

\bibitem[Dieudonn{\'e}(1937)]{dieudonne1937fonctions}
J.~Dieudonn{\'e}.
\newblock Sur les fonctions continues num{\'e}rique d{\'e}finies dans une
  produit de deux espaces compacts.
\newblock \emph{Comptes Rendus Acad. Sci. Paris}, 205:\penalty0 593--595, 1937.

\bibitem[Erickson(1996)]{erickson1996lower}
J.~G. Erickson.
\newblock \emph{Lower bounds for fundamental geometric problems}.
\newblock University of California, Berkeley, 1996.

\bibitem[Facco et~al.(2017)Facco, d’Errico, Rodriguez, and
  Laio]{facco2017estimating}
E.~Facco, M.~d’Errico, A.~Rodriguez, and A.~Laio.
\newblock Estimating the intrinsic dimension of datasets by a minimal
  neighborhood information.
\newblock \emph{Scientific reports}, 7\penalty0 (1):\penalty0 1--8, 2017.

\bibitem[Fefferman et~al.(2016)Fefferman, Mitter, and
  Narayanan]{fefferman2016testing}
C.~Fefferman, S.~Mitter, and H.~Narayanan.
\newblock Testing the manifold hypothesis.
\newblock \emph{Journal of the American Mathematical Society}, 29\penalty0
  (4):\penalty0 983--1049, 2016.

\bibitem[Fraikin et~al.(2007)Fraikin, H{\"u}per, and
  Dooren]{fraikin2007optimization}
C.~Fraikin, K.~H{\"u}per, and P.~V. Dooren.
\newblock Optimization over the stiefel manifold.
\newblock In \emph{PAMM: Proceedings in Applied Mathematics and Mechanics},
  volume~7, pages 1062205--1062206. Wiley Online Library, 2007.

\bibitem[Gale(1963)]{gale1963neighborly}
D.~Gale.
\newblock Neighborly and cyclic polytopes.
\newblock In \emph{Proc. Sympos. Pure Math}, volume~7, pages 225--232, 1963.

\bibitem[Goodfellow et~al.(2014)Goodfellow, Pouget-Abadie, Mirza, Xu,
  Warde-Farley, Ozair, Courville, and Bengio]{goodfellow2014generative}
I.~Goodfellow, J.~Pouget-Abadie, M.~Mirza, B.~Xu, D.~Warde-Farley, S.~Ozair,
  A.~Courville, and Y.~Bengio.
\newblock Generative adversarial nets.
\newblock \emph{Advances in neural information processing systems}, 27, 2014.

\bibitem[Hamilton(1982)]{hamilton1982three}
R.~S. Hamilton.
\newblock Three-manifolds with positive ricci curvature.
\newblock \emph{Journal of Differential geometry}, 17\penalty0 (2):\penalty0
  255--306, 1982.

\bibitem[He et~al.(2022)He, Chen, Xie, Li, Doll{\'a}r, and
  Girshick]{he2022masked}
K.~He, X.~Chen, S.~Xie, Y.~Li, P.~Doll{\'a}r, and R.~Girshick.
\newblock Masked autoencoders are scalable vision learners.
\newblock In \emph{Proceedings of the IEEE/CVF Conference on Computer Vision
  and Pattern Recognition}, pages 16000--16009, 2022.

\bibitem[Hein and Audibert(2005)]{hein2005intrinsic}
M.~Hein and J.-Y. Audibert.
\newblock Intrinsic dimensionality estimation of submanifolds in rd.
\newblock In \emph{Proceedings of the 22nd international conference on Machine
  learning}, pages 289--296, 2005.

\bibitem[Higgins et~al.(2017)Higgins, Matthey, Pal, Burgess, Glorot, Botvinick,
  Mohamed, and Lerchner]{higgins2017beta}
I.~Higgins, L.~Matthey, A.~Pal, C.~Burgess, X.~Glorot, M.~Botvinick,
  S.~Mohamed, and A.~Lerchner.
\newblock beta-vae: Learning basic visual concepts with a constrained
  variational framework.
\newblock In \emph{International conference on learning representations}, 2017.

\bibitem[Hinton(2022)]{hinton2022forward}
G.~Hinton.
\newblock The forward-forward algorithm: Some preliminary investigations.
\newblock \emph{arXiv preprint arXiv:2212.13345}, 2022.

\bibitem[Ho et~al.(2020)Ho, Jain, and Abbeel]{ho2020denoising}
J.~Ho, A.~Jain, and P.~Abbeel.
\newblock Denoising diffusion probabilistic models.
\newblock \emph{Advances in Neural Information Processing Systems},
  33:\penalty0 6840--6851, 2020.

\bibitem[Hoppe et~al.(1993)Hoppe, DeRose, Duchamp, McDonald, and
  Stuetzle]{hoppe1993mesh}
H.~Hoppe, T.~DeRose, T.~Duchamp, J.~McDonald, and W.~Stuetzle.
\newblock Mesh optimization.
\newblock In \emph{Proceedings of the 20th annual conference on Computer
  graphics and interactive techniques}, pages 19--26, 1993.

\bibitem[Huisken(1984)]{huisken1984flow}
G.~Huisken.
\newblock Flow by mean curvature of convex surfaces into spheres.
\newblock \emph{Journal of Differential Geometry}, 20\penalty0 (1):\penalty0
  237--266, 1984.

\bibitem[Jansen et~al.(2017)Jansen, Sell, and Lyzinski]{jansen2017scalable}
A.~Jansen, G.~Sell, and V.~Lyzinski.
\newblock Scalable out-of-sample extension of graph embeddings using deep
  neural networks.
\newblock \emph{Pattern Recognition Letters}, 94:\penalty0 1--6, 2017.

\bibitem[Kim and Mnih(2018)]{kim2018disentangling}
H.~Kim and A.~Mnih.
\newblock Disentangling by factorising.
\newblock In \emph{International Conference on Machine Learning}, pages
  2649--2658. PMLR, 2018.

\bibitem[Kingma and Welling(2013)]{kingma2013auto}
D.~P. Kingma and M.~Welling.
\newblock Auto-encoding variational bayes.
\newblock \emph{arXiv preprint arXiv:1312.6114}, 2013.

\bibitem[Kobyzev et~al.(2019)Kobyzev, Prince, and
  Brubaker]{kobyzev2019normalizing}
I.~Kobyzev, S.~Prince, and M.~Brubaker.
\newblock Normalizing flows: An introduction and review of current methods,
  arxiv e-prints.
\newblock \emph{arXiv preprint arXiv:1908.09257}, 2019.

\bibitem[Kramer(1991)]{kramer1991nonlinear}
M.~A. Kramer.
\newblock Nonlinear principal component analysis using autoassociative neural
  networks.
\newblock \emph{AIChE journal}, 37\penalty0 (2):\penalty0 233--243, 1991.

\bibitem[Lahiri et~al.(2016)Lahiri, Gao, and Ganguli]{lahiri2016random}
S.~Lahiri, P.~Gao, and S.~Ganguli.
\newblock Random projections of random manifolds.
\newblock \emph{arXiv preprint arXiv:1607.04331}, 2016.

\bibitem[Lawrence and Hyv{\"a}rinen(2005)]{lawrence2005probabilistic}
N.~Lawrence and A.~Hyv{\"a}rinen.
\newblock Probabilistic non-linear principal component analysis with gaussian
  process latent variable models.
\newblock \emph{Journal of machine learning research}, 6\penalty0 (11), 2005.

\bibitem[Lee(2012)]{lee2012smoothmanifolds}
J.~Lee.
\newblock \emph{Introduction to Smooth Manifolds}, volume 218 of \emph{Graduate
  Texts in Mathematics}.
\newblock Springer-Verlag New York, 2nd edition, 2012.
\newblock \doi{10.1007/978-1-4419-9982-5}.

\bibitem[Lee(2018)]{lee2018introduction}
J.~M. Lee.
\newblock \emph{Introduction to Riemannian manifolds}, volume 176.
\newblock Springer, 2018.

\bibitem[Levina and Bickel(2004{\natexlab{a}})]{levina2004maximum}
E.~Levina and P.~Bickel.
\newblock Maximum likelihood estimation of intrinsic dimension.
\newblock \emph{Advances in neural information processing systems}, 17,
  2004{\natexlab{a}}.

\bibitem[Levina and Bickel(2004{\natexlab{b}})]{levinaMLE}
E.~Levina and P.~Bickel.
\newblock Maximum likelihood estimation of intrinsic dimension.
\newblock In L.~Saul, Y.~Weiss, and L.~Bottou, editors, \emph{Advances in
  Neural Information Processing Systems}, volume~17. MIT Press,
  2004{\natexlab{b}}.
\newblock URL
  \url{https://proceedings.neurips.cc/paper/2004/file/74934548253bcab8490ebd74afed7031-Paper.pdf}.

\bibitem[Li et~al.(2019)Li, Li, and Zhang]{li2019le}
B.~Li, Y.-R. Li, and X.-L. Zhang.
\newblock A survey on laplacian eigenmaps based manifold learning methods.
\newblock \emph{Neurocomputing}, 335:\penalty0 336--351, 2019.
\newblock ISSN 0925-2312.
\newblock \doi{https://doi.org/10.1016/j.neucom.2018.06.077}.
\newblock URL
  \url{https://www.sciencedirect.com/science/article/pii/S0925231218312645}.

\bibitem[Li et~al.(2020)Li, Fuxin, and Todorovic]{li2020efficient}
J.~Li, L.~Fuxin, and S.~Todorovic.
\newblock Efficient riemannian optimization on the stiefel manifold via the
  cayley transform.
\newblock \emph{arXiv preprint arXiv:2002.01113}, 2020.

\bibitem[Li et~al.(2008)Li, Pan, and Chu]{li2008klpp}
J.-B. Li, J.-S. Pan, and S.-C. Chu.
\newblock Kernel class-wise locality preserving projection.
\newblock \emph{Information Sciences}, 178\penalty0 (7):\penalty0 1825--1835,
  2008.
\newblock ISSN 0020-0255.
\newblock \doi{https://doi.org/10.1016/j.ins.2007.12.001}.
\newblock URL
  \url{https://www.sciencedirect.com/science/article/pii/S0020025507005658}.

\bibitem[Lucas et~al.(2019)Lucas, Tucker, Grosse, and
  Norouzi]{lucas2019understanding}
J.~Lucas, G.~Tucker, R.~Grosse, and M.~Norouzi.
\newblock Understanding posterior collapse in generative latent variable
  models.
\newblock 2019.

\bibitem[Ma et~al.(2022)Ma, Tsao, and Shum]{ma2022principles}
Y.~Ma, D.~Tsao, and H.-Y. Shum.
\newblock On the principles of parsimony and self-consistency for the emergence
  of intelligence.
\newblock \emph{Frontiers of Information Technology \& Electronic Engineering},
  23\penalty0 (9):\penalty0 1298--1323, 2022.

\bibitem[McInnes et~al.(2018)McInnes, Healy, and Melville]{mcinnes2018umap}
L.~McInnes, J.~Healy, and J.~Melville.
\newblock Umap: Uniform manifold approximation and projection for dimension
  reduction.
\newblock \emph{arXiv preprint arXiv:1802.03426}, 2018.

\bibitem[Monera et~al.(2014)Monera, Montesinos-Amilibia, and
  Sanabria-Codesal]{monera2014taylor}
M.~G. Monera, A.~Montesinos-Amilibia, and E.~Sanabria-Codesal.
\newblock The taylor expansion of the exponential map and geometric
  applications.
\newblock \emph{Revista de la Real Academia de Ciencias Exactas, Fisicas y
  Naturales. Serie A. Matematicas}, 108\penalty0 (2):\penalty0 881--906, 2014.

\bibitem[Oue(1996)]{oue1996asymptotics}
S.~Oue.
\newblock On asymptotics of local principal component analysis.
\newblock \emph{Hitotsubashi journal of commerce and management}, pages 1--11,
  1996.

\bibitem[Qu et~al.()Qu, Zhai, Li, Zhang, and Zhu]{qu2020geometric}
Q.~Qu, Y.~Zhai, X.~Li, Y.~Zhang, and Z.~Zhu.
\newblock Geometric analysis of nonconvex optimization landscapes for
  overcomplete learning.
\newblock In \emph{International Conference on Learning Representations}.

\bibitem[Rezende and Viola(2018)]{rezende2018taming}
D.~J. Rezende and F.~Viola.
\newblock Taming vaes.
\newblock \emph{arXiv preprint arXiv:1810.00597}, 2018.

\bibitem[Rombach et~al.(2022)Rombach, Blattmann, Lorenz, Esser, and
  Ommer]{rombach2022high}
R.~Rombach, A.~Blattmann, D.~Lorenz, P.~Esser, and B.~Ommer.
\newblock High-resolution image synthesis with latent diffusion models.
\newblock In \emph{Proceedings of the IEEE/CVF Conference on Computer Vision
  and Pattern Recognition}, pages 10684--10695, 2022.

\bibitem[Ronneberger et~al.(2015)Ronneberger, Fischer, and
  Brox]{ronneberger2015u}
O.~Ronneberger, P.~Fischer, and T.~Brox.
\newblock U-net: Convolutional networks for biomedical image segmentation.
\newblock In \emph{International Conference on Medical image computing and
  computer-assisted intervention}, pages 234--241. Springer, 2015.

\bibitem[Roweis and Saul(2000)]{roweis2000nonlinear}
S.~T. Roweis and L.~K. Saul.
\newblock Nonlinear dimensionality reduction by locally linear embedding.
\newblock \emph{science}, 290\penalty0 (5500):\penalty0 2323--2326, 2000.

\bibitem[Schmidt et~al.(1992)Schmidt, Kraaijveld, Duin,
  et~al.]{schmidt1992feed}
W.~F. Schmidt, M.~A. Kraaijveld, R.~P. Duin, et~al.
\newblock Feed forward neural networks with random weights.
\newblock In \emph{International conference on pattern recognition}, pages
  1--1. IEEE Computer Society Press, 1992.

\bibitem[Sohl-Dickstein et~al.(2015)Sohl-Dickstein, Weiss, Maheswaranathan, and
  Ganguli]{sohl2015deep}
J.~Sohl-Dickstein, E.~Weiss, N.~Maheswaranathan, and S.~Ganguli.
\newblock Deep unsupervised learning using nonequilibrium thermodynamics.
\newblock In \emph{International Conference on Machine Learning}, pages
  2256--2265. PMLR, 2015.

\bibitem[Song and Ermon(2019)]{song2019generative}
Y.~Song and S.~Ermon.
\newblock Generative modeling by estimating gradients of the data distribution.
\newblock \emph{Advances in neural information processing systems}, 32, 2019.

\bibitem[Tak{\'a}cs and Tikk(2012)]{takacs2012alternating}
G.~Tak{\'a}cs and D.~Tikk.
\newblock Alternating least squares for personalized ranking.
\newblock In \emph{Proceedings of the sixth ACM conference on Recommender
  systems}, pages 83--90, 2012.

\bibitem[Tenenbaum et~al.(2000)Tenenbaum, Silva, and
  Langford]{tenenbaum2000global}
J.~B. Tenenbaum, V.~d. Silva, and J.~C. Langford.
\newblock A global geometric framework for nonlinear dimensionality reduction.
\newblock \emph{science}, 290\penalty0 (5500):\penalty0 2319--2323, 2000.

\bibitem[Turian et~al.(2010)Turian, Ratinov, and Bengio]{turian2010word}
J.~Turian, L.~Ratinov, and Y.~Bengio.
\newblock Word representations: a simple and general method for semi-supervised
  learning.
\newblock In \emph{Proceedings of the 48th annual meeting of the association
  for computational linguistics}, pages 384--394, 2010.

\bibitem[Vahdat et~al.(2021)Vahdat, Kreis, and Kautz]{vahdat2021score}
A.~Vahdat, K.~Kreis, and J.~Kautz.
\newblock Score-based generative modeling in latent space.
\newblock \emph{Advances in Neural Information Processing Systems},
  34:\penalty0 11287--11302, 2021.

\bibitem[Van Den~Oord et~al.(2017)Van Den~Oord, Vinyals, et~al.]{van2017neural}
A.~Van Den~Oord, O.~Vinyals, et~al.
\newblock Neural discrete representation learning.
\newblock \emph{Advances in neural information processing systems}, 30, 2017.

\bibitem[Van~der Maaten and Hinton(2008)]{van2008visualizing}
L.~Van~der Maaten and G.~Hinton.
\newblock Visualizing data using t-sne.
\newblock \emph{Journal of machine learning research}, 9\penalty0 (11), 2008.

\bibitem[Vandereycken(2013)]{vandereycken2013low}
B.~Vandereycken.
\newblock Low-rank matrix completion by riemannian optimization.
\newblock \emph{SIAM Journal on Optimization}, 23\penalty0 (2):\penalty0
  1214--1236, 2013.

\bibitem[Vidal et~al.(2005)Vidal, Ma, and Sastry]{vidal2005generalized}
R.~Vidal, Y.~Ma, and S.~Sastry.
\newblock Generalized principal component analysis (gpca).
\newblock \emph{IEEE transactions on pattern analysis and machine
  intelligence}, 27\penalty0 (12):\penalty0 1945--1959, 2005.

\bibitem[Wakin et~al.(2005)Wakin, Donoho, Choi, and
  Baraniuk]{wakin2005multiscale}
M.~B. Wakin, D.~L. Donoho, H.~Choi, and R.~G. Baraniuk.
\newblock The multiscale structure of non-differentiable image manifolds.
\newblock In \emph{Wavelets XI}, volume 5914, page 59141B. International
  Society for Optics and Photonics, 2005.

\bibitem[Wright and Ma(2022)]{wright2022high}
J.~Wright and Y.~Ma.
\newblock \emph{High-dimensional data analysis with low-dimensional models:
  Principles, computation, and applications}.
\newblock Cambridge University Press, 2022.

\bibitem[Xu et~al.(2008)Xu, Paiva, Park, and Principe]{xu2008reproducing}
J.-W. Xu, A.~R. Paiva, I.~Park, and J.~C. Principe.
\newblock A reproducing kernel hilbert space framework for
  information-theoretic learning.
\newblock \emph{IEEE Transactions on Signal Processing}, 56\penalty0
  (12):\penalty0 5891--5902, 2008.

\bibitem[Yu et~al.(2020)Yu, Chan, You, Song, and Ma]{yu2020learning}
Y.~Yu, K.~H.~R. Chan, C.~You, C.~Song, and Y.~Ma.
\newblock Learning diverse and discriminative representations via the principle
  of maximal coding rate reduction.
\newblock \emph{Advances in Neural Information Processing Systems},
  33:\penalty0 9422--9434, 2020.

\bibitem[Zhai et~al.()Zhai, Mehta, Zhou, and Ma]{zhai2020understanding}
Y.~Zhai, H.~Mehta, Z.~Zhou, and Y.~Ma.
\newblock Understanding l4-based dictionary learning: Interpretation,
  stability, and robustness.
\newblock In \emph{International conference on learning representations}.

\bibitem[Zhai et~al.(2020)Zhai, Yang, Liao, Wright, and Ma]{zhai2020complete}
Y.~Zhai, Z.~Yang, Z.~Liao, J.~Wright, and Y.~Ma.
\newblock Complete dictionary learning via l 4-norm maximization over the
  orthogonal group.
\newblock \emph{The Journal of Machine Learning Research}, 21\penalty0
  (1):\penalty0 6622--6689, 2020.

\bibitem[Zhang et~al.(2011)Zhang, Qiao, and Zhang]{zhang2011improved}
P.~Zhang, H.~Qiao, and B.~Zhang.
\newblock An improved local tangent space alignment method for manifold
  learning.
\newblock \emph{Pattern Recognition Letters}, 32\penalty0 (2):\penalty0
  181--189, 2011.

\bibitem[Zhang and Zha(2003)]{zhang2003nonlinear}
Z.~Zhang and H.~Zha.
\newblock Nonlinear dimension reduction via local tangent space alignment.
\newblock In \emph{Intelligent Data Engineering and Automated Learning: 4th
  International Conference, IDEAL 2003, Hong Kong, China, March 21-23, 2003.
  Revised Papers 4}, pages 477--481. Springer, 2003.

\end{thebibliography}
\end{document}